\documentclass{article}
\usepackage[utf8]{inputenc}

\usepackage{geometry}
\usepackage{natbib}

\usepackage[unicode=true,pdfusetitle,
bookmarks=true,bookmarksnumbered=false,bookmarksopen=false,
 breaklinks=true,pdfborder={0 0 0},pdfborderstyle={},backref=false,colorlinks=true]{hyperref}
\hypersetup{citecolor=blue}

\usepackage{amsmath}
\usepackage{amssymb}
\usepackage{amsthm}
\usepackage{amsfonts}
\usepackage{appendix}
\usepackage{booktabs}
\usepackage{multirow}
\usepackage{array}

\makeatletter
\newcommand{\Biggg}{\bBigg@{3.5}}
\makeatother

\theoremstyle{plain}
\newtheorem{thm}{\protect\theoremname}
\theoremstyle{plain}
\newtheorem{lem}[thm]{\protect\lemmaname}
\theoremstyle{remark}
\newtheorem{remark}[thm]{\protect\remarkname}
\theoremstyle{plain}

\theoremstyle{plain}

\providecommand{\corollaryname}{Corollary}
\providecommand{\lemmaname}{Lemma}
\providecommand{\remarkname}{Remark}
\providecommand{\theoremname}{Theorem}
\providecommand{\conjecturename}{Conjecture}

\usepackage{algorithm}
\usepackage{algpseudocode}
\usepackage{float}

\usepackage{enumitem}
\usepackage{caption} 

\newcommand{\argmin}{\text{argmin}}
\newcommand{\argmax}{\text{argmax}}

\newcommand{\R}{\mathbb{R}}

\newcommand{\E}{\mathbb{E}}
\newcommand{\Var}{\mathbb{V}}

\newcommand{\Phat}{\widehat{P}}
\newcommand{\Vhat}{\widehat{V}}

\newcommand{\pistar}{\pi^\star}

\newcommand{\one}{\mathbf{1}}
\newcommand{\zero}{\mathbf{0}}

\newcommand{\ind}{\mathbb{I}}

\newcommand{\pihat}{\widehat{\pi}}
\newcommand{\Vhatstar}{\widehat{V}^\star}

\newcommand{\Vc}{\overline{V}}

\newcommand{\rhohat}{\widehat{\rho}}
\newcommand{\hhat}{\widehat{h}}

\newcommand{\A}{\mathcal{A}}
\renewcommand{\S}{\mathcal{S}}

\newcommand{\Otilde}{\widetilde{O}}

\newcommand{\tmix}{\tau_{\mathrm{unif}}}

\DeclareMathOperator*{\Clim}{\text{C-lim}}

\usepackage{scalerel,stackengine}
\newcommand\cig[1]{\scalerel*[5.5pt]{\Big#1}{%
  \ensurestackMath{\addstackgap[1.5pt]{\big#1}}}}
\newcommand\cigl[1]{\mathopen{\cig{#1}}}
\newcommand\cigr[1]{\mathclose{\cig{#1}}}

\global\long\def\infnorm#1{\left\Vert #1\right\Vert _{\infty}}%
\global\long\def\infinfnorm#1{\left\Vert #1\right\Vert _{\infty \to \infty}}%
\global\long\def\spannorm#1{\left\Vert #1\right\Vert _{\textnormal{span}}}%

\global\long\def\bigspannorm#1{\bigl\Vert #1 \bigl\Vert _{\textnormal{span}}}%

\global\long\def\tspannorm#1{\Vert #1\Vert _{\textnormal{span}}}%
\global\long\def\tinfnorm#1{\Vert #1\Vert _{\infty}}%

\global\long\def\infnorm#1{\left\Vert #1\right\Vert _{\infty}}%
\global\long\def\spannorm#1{\left\Vert #1\right\Vert _{\textnormal{span}}}%

\newcommand{\gammahat}{\widehat{\gamma}}

\newcommand{\gstar}{\gamma^\star}

\newcommand{\hzns}{\mathcal{H}}
\newcommand{\pit}{\widetilde{\pi}}
\newcommand{\Vt}{\widetilde{V}}
\newcommand{\Obj}{\widehat{L}}
\newcommand{\Ub}{\widehat{U}}
\newcommand{\Lb}{\widehat{L}}

\newcommand{\rt}{\tilde{r}}

\newcommand{\Po}{\overline{P}}
\newcommand{\Vo}{\overline{V}}

\newcommand{\BD}{B} 
\newcommand{\spb}{M} 

\newcommand{\Clip}{\textnormal{Clip}}

\newcommand{\That}{\widehat{\mathcal{T}}}
\newcommand{\T}{\mathcal{T}}
\renewcommand{\L}{\mathcal{L}}

\newcommand{\SolveDMDP}{\textsc{SolveDMDP}}

\newcommand{\Objsp}{\widehat{L}} 

\newcommand{\Phatanc}{\underline{\Phat}}

\newcommand{\hhatanc}{\underline{\hhat}}

\newcommand{\hanc}{\underline{h}}

\usepackage{leftindex}

\newcommand{\PiMD}{\Pi_{\mathrm{MD}}}

\newcommand{\Bhat}{\widehat{B}}

\usepackage{tikz}
\usetikzlibrary{shapes, arrows, positioning}

\usepackage{nicefrac}       
\usepackage{microtype}      

\date{}

\title{Span-Agnostic Optimal Sample Complexity and Oracle Inequalities for Average-Reward RL}

\usepackage{authblk}
\author{Matthew Zurek}
\author{Yudong Chen}
\affil{Department of Computer Sciences, University of Wisconsin-Madison\\\texttt{\{matthew.zurek,yudong.chen\}@wisc.edu}}

\begin{document}

\maketitle

\begin{abstract}%
  We study the sample complexity of finding an $\varepsilon$-optimal policy in average-reward Markov Decision Processes (MDPs) with a generative model. The minimax optimal span-based complexity of $\widetilde{O}(SAH/\varepsilon^2)$, where $H$ is the span of the optimal bias function, has only been achievable with prior knowledge of the value of $H$. Prior-knowledge-free algorithms have been the objective of intensive research, but several natural approaches provably fail to achieve this goal. We resolve this problem, developing the first algorithms matching the optimal span-based complexity without $H$ knowledge, both when the dataset size is fixed and when the suboptimality level $\varepsilon$ is fixed. Our main technique combines the discounted reduction approach with a method for automatically tuning the effective horizon based on empirical confidence intervals or lower bounds on performance, which we term \textit{horizon calibration}. We also develop an \textit{empirical span penalization} approach, inspired by sample variance penalization, which satisfies an \textit{oracle inequality} performance guarantee. In particular this algorithm can outperform the minimax complexity in benign settings such as when there exist near-optimal policies with span much smaller than $H$.%
\end{abstract}

\section{Introduction}

Reinforcement Learning (RL) has achieved significant empirical successes in various fields, demonstrating its potential to solve complex decision-making problems. RL is commonly modeled as to learn a policy which maximizes cumulative rewards within a Markov Decision Process (MDP), where the cumulative rewards can be measured in several different ways. We focus on the average-reward criterion, which involves the long-term average of collected rewards as the horizon goes to infinity, making it suitable for ongoing tasks without a natural endpoint. 

A fundamental question in average-reward RL is the sample complexity for learning a near-optimal policy under a generative model of the MDP. 
This question has been the subject of intensive research. Recent work has established the minimax-optimal span-based complexity $\Otilde(SA \tspannorm{h^\star}/\varepsilon^2)$ for learning an $\varepsilon$-optimal policy \citep{zurek_span-based_2025}, where $\tspannorm{h^\star}$ denotes the span of the optimal bias $h^\star$ and is known to be a more refined complexity parameter than alternatives such as diameter or mixing times. However, this algorithm as well as earlier work all require prior knowledge of $\tspannorm{h^\star}$ (or other complexity parameters), which is generally unavailable, making the algorithms impractical. A flurry of subsequent research \citep{neu_dealing_2024, tuynman_finding_2024, jin_feasible_2024, zurek_plug-approach_2024} has focused on removing the need for prior knowledge but failed to match the optimal span-based complexity. In fact, several natural approaches to knowledge-free optimal complexity, including span estimation and the average-reward plug-in method, are shown to provably fail \citep{tuynman_finding_2024, zurek_plug-approach_2024}.

In this paper we resolve this problem, providing algorithms which obtain the optimal span-based complexity without knowing $\tspannorm{h^\star}$, for both settings where we fix the dataset size $n$ and where we prescribe a target suboptimality level $\varepsilon$. Our algorithms are based upon reductions to discounted MDPs (DMDPs) combined with a novel technique of (effective-)\textit{horizon calibration}, which chooses discount factors to maximize lower bounds or minimize confidence intervals on policy performance. This technique can be seen as related, but representing a simpler alternative, to the technique of sample variance penalization (SVP) from statistical learning \citep{maurer_empirical_2009, duchi_variance-based_2019}. We further develop an algorithm based more closely on a relaxed version of SVP, which we call \textit{empirical span penalization}, which enjoys even stronger guarantees. In particular, this algorithm satisfies a complexity bound in terms of the minimum span $\tspannorm{h^\pi}$ of any gain-optimal policy $\pi$ in place of $\tspannorm{h^\star}$. Moreover, it adapts to and competes with simpler (potentially suboptimal) policies with the best tradeoff between complexity and suboptimality. This bound is reminiscent of the \emph{oracle inequalities} from the statistical learning literature \citep{deheuvels_lectures_2007, koltchinskii_oracle_2011}, but is new to average-reward RL.

\subsection{Related Work}

The problem of learning optimal policies in average-reward MDPs (AMDPs) is studied in \citet{jin_efficiently_2020, jin_towards_2021, li_stochastic_2024, wang_near_2022, zhang_sharper_2023, wang_optimal_2023}. We start with the recent work \cite{zurek_span-based_2025}, which was the first to obtain the optimal span-based sample complexity but required prior knowledge of $\tspannorm{h^\star}$ to do so. All earlier work also required knowledge of problem-dependent complexity parameters. More recent work, which we discuss below, has studied the setting without such prior knowledge; see Table~\ref{table:AMDPs} for a summary.

\cite{tuynman_finding_2024} and \cite{zurek_span-based_2025} show that it is generally impossible to obtain a multiplicative estimate of $\tspannorm{h^\star}$ with $\text{poly}(SA\tspannorm{h^\star})$ samples. See Appendix \ref{sec:gammahat_hstar} for discussion of the relationship between our algorithms and estimating $\tspannorm{h^\star}$. 
By estimating the MDP's diameter $D$, which upper bounds $\spannorm{h^\star} $ but can be arbitrarily larger \citep{bartlett_regal_2012, lattimore_bandit_2020}, the work in \cite{tuynman_finding_2024} removes the need for prior knowledge within the algorithm of \cite{zurek_span-based_2025} but obtains a complexity involving $D$ rather than $\tspannorm{h^\star}$. The Q-learning-based algorithm in \cite{jin_feasible_2024} uses increasing discount factors and does not require prior knowledge. Their complexity bound however depends on the largest mixing time of all policies, $\tmix$, which  satisfies $3\tmix \geq \tspannorm{h^\star}$ and can be infinite or arbitrarily larger than $\tspannorm{h^\star}$ \citep{wang_near_2022,zurek_plug-approach_2024}.
(See Appendix~\ref{sec:complexity_params} for definitions of $D$ and $\tmix$.) 

\cite{neu_dealing_2024} and \cite{zurek_plug-approach_2024} study, respectively, approaches based on stochastic saddle-point optimization and the average-reward plug-in method, both obtaining bounds involving the bias spans of certain policies output by the algorithm. These spans are not generally controlled by $\tspannorm{h^\star}$; in particular, \citet[Theorem 14]{zurek_plug-approach_2024} present an example where this is the case and show that the average-reward plug-in approach cannot achieve the optimal $SA \tspannorm{h^\star}/\varepsilon^2$ complexity. \cite{zurek_plug-approach_2024} also analyzes a DMDP-reduction algorithm that uses a (relatively small) effective horizon independent of $\tspannorm{h^\star}$, achieving a suboptimal complexity with $\tspannorm{h^\star}^2$ dependence.

Also related to the present work are papers studying the sample complexity of the model-based/plug-in approach for discounted MDPs  \citep{azar_sample_2012, azar_minimax_2013, agarwal_model-based_2020, li_breaking_2020, zurek_plug-approach_2024}.
We also note that \cite{boone_achieving_2024} recently developed an algorithm for the \emph{online} setting achieving a $\tspannorm{h^\star}$-based regret bound without requiring prior knowledge. This result does not imply any sample complexity bounds in our setting, because there is no general regret-to-PAC conversion for average-reward MDPs \citep{tuynman_finding_2024}, and even if this were possible, their result appears to require $\Omega(S^{40}A^{20}\tspannorm{h^\star}^{10})$ interaction steps before achieving the optimal regret, which would imply a massive ``burn-in'' cost in our setting.

\begin{table}[t]
{
\renewcommand{\arraystretch}{2} 
\centering
\begin{tabular}{p{0.22\textwidth}ccc}
\toprule
Algorithm & Sample Complexity & Reference & \parbox[c]{1.6cm}{Prior\\Knowledge} \\ 
\midrule
\multicolumn{1}{m{0.22\textwidth}}{DMDP Reduction} & $SA\frac{\spannorm{h^\star}+1}{\varepsilon^2}$ & \cite{zurek_span-based_2025} & Yes  \\ 
\midrule
\multicolumn{1}{m{0.22\textwidth}}{Diameter Estimation + DMDP Reduction} & $SA\frac{D}{\varepsilon^2} + S^2 A D^2 $ & \cite{tuynman_finding_2024} & No  \\ 
\multicolumn{1}{m{0.22\textwidth}}{Dynamic Horizon Q-Learning} & $SA\frac{\tmix^8}{\varepsilon^8} $ & \cite{jin_feasible_2024} & No  \\ 
\multicolumn{1}{m{0.22\textwidth}}{Stochastic Saddle-Point Optimization} & $S^2A^2\frac{\spannorm{h^{\pihat}}^4}{\varepsilon^2} $ & \cite{neu_dealing_2024} & No  \\
\multirow{1}{0.22\textwidth}[-.35cm]{Plug-in Approach with Anchoring and Reward Perturbation} &  $SA \frac{\min \{D, \, \tmix\}}{\varepsilon^2}$ \vspace{.1cm} & \multirow{2}{*}{\cite{zurek_plug-approach_2024}} & \multirow{2}{*}{No} \\ 
 &  $SA\frac{\spannorm{h^\star} + \min\big\{\bigspannorm{\hhatanc^\star}, \spannorm{\hanc^{\pihat}} \big\}}{\varepsilon^2} \vspace{.1cm} $  &  &  \\ 
\multicolumn{1}{m{0.21\textwidth}}{$\sqrt{n}$-Horizon DMDP Reduction}~  & $SA\frac{\spannorm{h^\star}^2 + 1 }{\varepsilon^2} $ & \cite{zurek_plug-approach_2024} & No \\ 
\midrule
\multicolumn{1}{m{0.22\textwidth}}{DMDP Reduction + Horizon Calibration} & $SA\frac{\spannorm{h^\star}+1}{\varepsilon^2}$ & Our Theorems \ref{thm:n_based_alg} and \ref{thm:eps_based_alg} & No  \\ 
\multicolumn{1}{m{0.22\textwidth}}{Span Penalization} & $SA\,\inf_{\pi : ~\rho^\pi \text{ constant}} \cig\{ \frac{\spannorm{h^\pi}}{(\rho^\pi - \rho^\star + \varepsilon)^2} \cig\}$ & Our Theorem \ref{thm:span_regularization_performance} & No  \\ 
\bottomrule
\end{tabular}

    \caption{\textbf{Algorithms and sample complexity bounds for average reward MDPs} for finding an $\varepsilon$-optimal policy under a generative model (up to $\log$ factors). See Appendix \ref{sec:complexity_params} for the definitions of the complexity parameters $D,\tmix, \tspannorm{h^{\pihat}}, \tspannorm{\hhatanc^\star}, \tspannorm{\hanc^{\pihat}}$ used in prior work. Note that the diameter $D$ and uniform mixing time $ 3 \tmix $ are both upper bounds of $\tspannorm{h^\star}$ and can be arbitrarily larger than $\tspannorm{h^\star}$. 
    The parameters $\tspannorm{h^{\pihat}}, \tspannorm{\hhatanc^\star}, \tspannorm{\hanc^{\pihat}}$ are not generally controlled by $\tspannorm{h^\star}$.
    The guarantee for our Span Penalization algorithm involves the infimum over all policies $\pi$ with constant (state-independent) gain $\rho^\pi$; see Theorem \ref{thm:span_regularization_performance} for an equivalent guarantee in terms of the dataset size.
    }
    
\label{table:AMDPs}
}
\end{table}

\section{Problem Setup}

A Markov decision process is a tuple $(\S, \A, P, r)$, where $\S$, $\A$ are the state and action spaces, respectively, with finite cardinalities $S := |\S|$ and $A := |\A|$,  $P : \S \times \A \to \Delta(\S)$ is the transition kernel with $\Delta(\S)$ denoting the probability simplex on $\S$, and $r : \S \times \A \to [0,1]$ is the reward function. 
We only consider Markovian stationary policies of the form $\pi : \S \to \Delta(\A)$. For initial state $s_0 \in \S$ and policy $\pi$, let $\E^\pi_{s_0}$ denote the expectation w.r.t.\ the distribution over trajectories $(S_0, A_0, S_1, A_1, \dots)$ with $S_0 = s_0$, $A_t \sim \pi(S_t)$, and $S_{t+1} \sim P(\cdot \mid S_t, A_t)$. 
Let $P_\pi$ denote the transition probability matrix of the Markov chain induced by $\pi$, where $\left(P_\pi\right)_{s,s'} := \sum_{a \in \A} \pi(a | s) P(s' \mid s, a)$. Similarly let $(r_\pi)_{s} := \sum_{a \in \A} \pi(a | s) r(s, a)$. We also consider $P$ as an $(S \times A)$-by-$ S$ matrix with $P_{sa, s'} = P(s' \,|\, s, a)$, and $r$ as an $S$-dimensional vector.
For a policy $\pi$, define the policy matrix $M^\pi \in \R^{S \times SA}$ by $M^\pi_{s, s a} = \pi(a | s)$ and $M^\pi_{s, s' a} = 0$ if $s \neq s'$. Note that $P_\pi = M^\pi P$ and $r_\pi = M^\pi r$. Also define the maximization operator $M : \R^{SA} \to \R^S$ by $M(x)_s = \max_{a} x_{sa}$.

We assume $P$ is unknown, but one has access to a generative model (a.k.a. simulator) \citep{kearns_finite-sample_1998}, which provides independent samples from $P(\cdot \mid s, a)$ for each $s \in \S, a \in \A$. We assume $r$ is known, which is standard \citep{agarwal_model-based_2020, li_breaking_2020} as otherwise estimating $r$ is relatively easy.
Let $\zero, \one \in \R^{\S}$ be the all-zero and all-one vectors, respectively.

\textbf{Discounted reward criterion~~} A discounted MDP is a tuple $(\S, \A, P, r, \gamma)$, where $\gamma \in (0,1)$ is the discount factor. For a policy $\pi$, the (discounted) value function $V^\pi_\gamma : \S \to [0, \infty)$ is defined as $V^\pi_\gamma(s) := \E^\pi_s \left[\sum_{t=0}^\infty \gamma^t R_t \right]$,
where $R_t = r(S_t, A_t)$ is the reward received at time $t$. There always exists an optimal policy $\pistar_\gamma$ that satisfies $V_\gamma^{\pistar_\gamma}(s) = V_\gamma^\star(s) := \sup_{\pi} V_\gamma^\pi(s), \forall s \in \S$ \citep{puterman_markov_1994}. When using transition kernel $\Phat$ we will accordingly write $\Vhat_\gamma^\pi$ for the associated value function. For reward functions $r'$ other than $r$, we include the reward function in the subscript e.g. $V_{\gamma, r'}^\pi$.

\textbf{Average-reward criterion~~}
In an MDP $(\S, \A, P, r)$, the average reward, a.k.a.\ the \emph{gain}, of a policy $\pi$ starting from state $s$ is defined as $\rho^\pi(s)  := \lim_{T \to \infty} \frac{1}{T} \E_s^\pi \big[\sum_{t=0}^{T-1} R_t \big].$
The \emph{bias function} of a stationary policy $\pi$ is
$h^\pi(s) := \Clim_{T \to \infty} \E_s^\pi \big[\sum_{t=0}^{T-1} \left(R_t - \rho^\pi(S_t)\right) \big]$,
where $\Clim$ denotes the Cesaro limit. 
When the Markov chain induced by $P_\pi$ is aperiodic, $\Clim$ can be replaced with the usual limit. 
A policy $\pistar$ is \emph{Blackwell-optimal} if there exists some discount factor $\Bar{\gamma} \in [0,1)$ such that for all $\gamma \geq \Bar{\gamma}$ we have $V^{\pistar}_\gamma \geq V^{\pi}_\gamma, \forall \pi$. When $S$ and $A$ are finite, there always exists some Blackwell-optimal policy, denoted by $\pistar$ \citep{puterman_markov_1994}. Define the optimal gain $\rho^\star\in \R^\S$ by $\rho^\star(s) = \sup_{\pi} \rho^\pi(s)$ and note that $\rho^\star := \rho^{\pistar}$. Define the optimal bias $h^\star := h^{\pistar}$ (which is unique even when $\pistar$ is not). A policy $\pi$ is \textit{gain-optimal} if $\rho^\pi = \rho^\star$ and it is \textit{bias-optimal} if in addition $h^\pi = h^\star$. For $x \in \R^{\S} $, define the span semi-norm  $\spannorm{x} := \max_{s \in \S} x(s) - \min_{s \in \S} x(s).$

An MDP is communicating if for any states $s$ and $s'$, some policy can reach $s'$ from $s$ with probability 1.
An MDP is weakly communicating if the states can be partitioned into two  subsets $\S = \S_1 \cup \S_2$ such that all states in $\S_1$ are transient under all stationary policies and $\S_2$ is communicating. In weakly communicating MDPs, $\rho^\star$ is a constant vector (all entries are equal). All results in this paper assume that $P$ is weakly communicating.
While not used in our results, the definitions of the MDP diameter $D$ and uniform mixing time $\tmix$ are given in Appendix \ref{sec:complexity_params} for completeness.

\section{Main Results}

In this section, we present our algorithms and main results. 
Our algorithms involve the function $\alpha(\delta, n) = 96 \sqrt{ \log \left( {24 SA n^5}/{\delta}\right)} \log_2\left( \log_2 (n + 4) \right)$, which is $\Otilde(1)$; see Remark \ref{rem:alpha} for its origin.

\subsection{Fixed-$n$ Setting}

First we consider the setting where  the number of samples per state-action pair, $n$, is fixed. Our objective is to learn a policy with the best possible rate of suboptimality. We refer to this as the \emph{fixed-$n$ setting}. Our Algorithm \ref{alg:n_based_alg} is based on using the dataset to form an empirical transition kernel $\Phat$ and then computing a near-optimal policy in the DMDP $(\Phat, r)$ for some discount factor $\gamma$.
The key technique is a method for automatically calibrating $\gamma$ (equivalently, the effective horizon $\frac{1}{1-\gamma}$): we try multiple values of $\gamma$, and for each we compute a near-optimal policy $\pit_\gamma$  for the DMDP $(\Phat, r, \gamma)$ and a quantity $\Obj(\gamma)$ that lower bounds its gain. We then use the discount factor $\gammahat$ that optimizes this lower bound. For computational efficiency, we only need to try $O(\log n)$ values of $\gamma$.

\begin{algorithm}[h]
\caption{Lower Bound Maximization via Horizon Calibration} \label{alg:n_based_alg}
\begin{algorithmic}[1]
\Require Sample size per state-action pair $n$
\For{each state-action pair $(s,a) \in \S \times \A$}
    \State Collect $n$ samples $S^1_{s,a}, \dots, S^n_{s,a}$ from $P(\cdot \mid s,a)$
    \State Form the empirical transition kernel $\Phat(s' \mid s, a) = \frac{1}{n}\sum_{i=1}^n \ind\{S^i_{s,a} = s'\}$, for all $s' \in \S$
\EndFor
\State Form geometric discount factor range $\hzns := \{\gamma: \text{there exists an integer $k $ such that } \sqrt{n} \leq \frac{1}{1-\gamma} = 2^k \leq n\}$
\For{each discount factor $\gamma \in \hzns$}
    \State Obtain policy $\pit_\gamma$ and value function $\Vt_\gamma$ from $\SolveDMDP(\Phat, r, \gamma, \frac{1}{n})$ \label{alg:solver_step}
    \State Compute objective value $\Obj(\gamma) := (1-\gamma)\min_{s}\Vt_{\gamma}(s) - 2\frac{1-\gamma}{n} - \alpha(\delta, n) \label{alg:obj_value} \sqrt{\frac{\tspannorm{\Vt_{\gamma}} +\frac{3}{n} +1 }{n}} $
\EndFor
\State Find $\gammahat = \argmax_{\gamma \in \hzns} \Obj(\gamma)$
\State \Return policy $\pihat := \pit_{\gammahat}$, gain lower bound $\rhohat := \max\{\Obj(\gammahat), 0\}\one$
\end{algorithmic}
\end{algorithm}

\begin{thm}
\label{thm:n_based_alg}
    Suppose $P$ is weakly communicating. For some constant $C_3$, with probability at least $1 - \delta$, the policy $\pihat$ and gain lower bound $\rhohat$ output by Algorithm \ref{alg:n_based_alg} satisfy (elementwise)
    \begin{align*}
        \rho^{\pihat} \geq \rhohat \geq \rho^\star - C_3 \alpha(\delta, n)^2 \sqrt{\frac{\spannorm{h^{\star}}+1}{n}} \one.
    \end{align*}
\end{thm}
 Theorem \ref{thm:n_based_alg} shows that Algorithm \ref{alg:n_based_alg} returns a policy $\pihat$ with the minimax optimal rate of suboptimality without using any prior knowledge. The algorithm also returns a performance \emph{certificate} $\rhohat$, which lower-bounds $\rho^{\pihat}$ (and $\rho^\star$), and this bound is tight up to an error of $\Otilde(\sqrt{(\tspannorm{h^{\star}}+1)/n})$.

We allow $\SolveDMDP$, used in line~\ref{alg:solver_step} of Algorithm~\ref{alg:n_based_alg}, to be any subroutine for approximately solving the empirical DMDP $(\Phat, r, \gamma)$, and simply require that it returns a deterministic policy $\pit_\gamma$ and an approximate value function $\Vt_\gamma$ such that $\Vhat_\gamma^{\pit_\gamma} \geq \Vhat^\star_\gamma - \frac{1}{n} \one $ and $ \tinfnorm{\Vt_\gamma - \Vhat^\star_\gamma} \leq \frac{1}{n}.$
This can be done for instance using $O\big(\frac{\log (n/(1-\gamma))}{1-\gamma} \big)$ iterations of value iteration.

Algorithm \ref{alg:n_based_alg} is reminiscent of sample variance penalization (SVP), a statistical learning algorithm which outputs a hypothesis minimizing the empirical risk plus an estimated variance term \citep{maurer_empirical_2009, duchi_variance-based_2019}. Here we clarify the connections and differences, which also provides intuitions for our algorithms.
First, Algorithm \ref{alg:n_based_alg} can be understood as controlling not only certain empirical variance (represented by the last term in the definition of $\Obj(\gamma)$ in line~\ref{alg:obj_value}) but also a certain bias/approximation error due to discounted reduction. This is a significant difference since prior to our work it was not clear that such approximation error could be estimated/controlled without knowing $\spannorm{h^\star}$. See our proof sketch in Section \ref{sec:proof_sketch} for more on this issue. Second, supposing that a lower bound like $\rho^\pi \geq (1-\gamma)\min_s \Vhat_\gamma^{\pi} (s)\one - \sqrt{\tspannorm{\Vhat_\gamma^{\pi}}+1/n}\one$ holds for all policies $\pi$ and all $\gamma$ with high probability,\footnote{For a single policy $\pi$ a similar bound follows from our techniques. Uniformity over all $\pi$'s would incur an additional $\sqrt{S}$.} then an analogue of SVP would be~\eqref{eq:SVP_analogue} below, while Algorithm \ref{alg:n_based_alg} can instead be seen as (approximately) solving~\eqref{eq:n_based_alg_abstract}. 
{%
\begin{tabular}{p{.48\linewidth}p{.48\linewidth}}
\begin{equation}
    \max_{\pi,\gamma} ~ (1-\gamma) \min_s \Vhat_\gamma^\pi(s) - \sqrt{\frac{\tspannorm{\Vhat_\gamma^\pi} + 1}{n}} \label{eq:SVP_analogue}
\end{equation}
  &
  \begin{equation}
    \max_{\gamma}~ (1-\gamma) \min_s \Vhat_\gamma^\star(s) - \sqrt{\frac{\tspannorm{\Vhat_\gamma^\star} + 1}{n}} \label{eq:n_based_alg_abstract}
  \end{equation}
\end{tabular}
}
Since $\max_{\pi} \Vhat_\gamma^\pi = \Vhat_\gamma^\star$, solving \eqref{eq:n_based_alg_abstract} can be understood as only choosing $\pi$ to optimize $\Vhat_\gamma^\pi$ and then controlling the objective via $\gamma$ tuning, whereas~\eqref{eq:SVP_analogue} optimizes all objective terms jointly. While~\eqref{eq:SVP_analogue} may appear more principled, it is not immediately clear how to solve such a problem, whereas~\eqref{eq:n_based_alg_abstract} can simply utilize any DMDP solver for optimizing $\Vhat_\gamma^\pi$ for fixed $\gamma$, then tune $\gamma$ afterwards. However, this is not the final word on~\eqref{eq:SVP_analogue}, as in Subsection \ref{sec:span_regularization} we develop our Algorithm \ref{alg:span_regularization} based on solving~\eqref{eq:SVP_analogue}.

\subsection{Fixed-$\varepsilon$ Setting}

We next present our Algorithm~\ref{alg:eps_based_alg} for the setting where one is given a target suboptimality $\varepsilon$, and the goal is to return a policy with suboptimality bounded by $\varepsilon$ using as few samples as possible. We refer to this setting as the \emph{fixed-$\varepsilon$ setting}. At a high level, we run our algorithm for the fixed-$n$ setting for a geometrically increasing sequence of dataset sizes $\{n_i\}_i$. However, the fixed-$\varepsilon$ setting is more challenging, because beyond lower-bounding the gain of some known policies, to obtain a termination condition we additionally need an (observable) upper bound on the optimal gain $\rho^\star$, meaning we need to bound the gains of all policies. On iteration $i$ with a dataset of size $n_i$, we compute both lower and upper bounds $\Lb_i(\gamma)$ and $\Ub_i(\gamma)$ for a range of $\gamma$, yielding different confidence intervals. The algorithm  terminates once one such interval is sufficiently small and certifies the desired suboptimality level $\varepsilon$. We provide more details on how to compute such an upper bound $\Ub_i(\gamma)$ without knowing $\spannorm{h^\star}$ in the proof sketches in Section~\ref{sec:proof_sketch}. 

\begin{algorithm}[t]
\caption{Confidence Interval Minimization via Horizon Calibration}
\label{alg:eps_based_alg}
\begin{algorithmic}[1]
\Require Target suboptimality $\varepsilon > 0$
\State Set iteration number $i = 0$
\Repeat
\State $i \gets i + 1$; 
set sample size per state-action pair $n_i = 2^{i}$
\For{each state-action pair $(s,a) \in \S \times \A$}
\State Collect $n_i $ samples $S^1_{s,a}, \dots, S^{n_i}_{s,a}$ from $P(\cdot \mid s,a)$
\State Form the $i$th empirical transition kernel $\Phat^{(i)}(s' \mid s, a) = \frac{1}{n}\sum_{j=1}^{n_i} \ind\{S^j_{s,a} = s'\}, \forall s' \in \S$
\EndFor
\State Form geometric discount factor range $\hzns_i := \{\gamma: \text{there exists an integer $k $ such that } \sqrt{n_i} \leq \frac{1}{1-\gamma} = 2^k \leq n_i \}$ \label{alg:hzns_def_dyadic}
\For{each discount factor $\gamma \in \hzns_i$}
\State Obtain policy $\pit_{\gamma,i}$ and value function $\Vt_{\gamma,i}$ from $\SolveDMDP(\Phat^{(i)}, r, \gamma, \frac{1}{n_i})$ \label{alg:solver_step_eps_based}
\State Compute upper bound $\Ub_i(\gamma) := (1-\gamma) \max_s \Vt_{\gamma,i}(s)  + 5\frac{1-\gamma}{n_i}+ \frac{2 \alpha(\delta, n_i)^2}{(1-\gamma) n_i} + 4 \alpha(\delta, n_i) \sqrt{\frac{\tspannorm{\Vt_{\gamma, i}} +1+ \frac{3}{n_i}}{n_i}} $
\State Compute lower bound $\Lb_i(\gamma) := (1-\gamma) \min_s \Vt_{\gamma, i} (s)   - 2\frac{1-\gamma}{n_i} - \alpha(\delta, n_i) \sqrt{\frac{ \tspannorm{\Vt_{\gamma,i}} +\frac{3}{n_i} + 1 }{n_i}} $
\EndFor
\State Find discount factor with the smallest interval $\gammahat_i := \argmin_{\gamma \in \hzns_i} \Ub_i(\gamma) - \Lb_i(\gamma)$
\Until{$\Ub_i(\gammahat_i) - \Lb_i(\gammahat_i) \leq \varepsilon$}
\State \Return policy $\pihat := \pit_{\gammahat_i, i}$, optimal gain upper bound $\Ub := \Ub_i(\gammahat_i)$ and lower bound $\Lb := \Lb_i(\gammahat_i)$
\end{algorithmic}
\end{algorithm}

\begin{thm}
    \label{thm:eps_based_alg}
    Suppose $P$ is weakly communicating. There exist constants $C_1, C_2$ such that for any $\varepsilon > 0$, with probability at least $1-\delta$, Algorithm \ref{alg:eps_based_alg} uses at most
    \begin{align*}
        N := 4 C_1  \frac{\spannorm{h^\star}+1}{\varepsilon^2} \log^3 \bigg( \frac{C_2 SA (\spannorm{h^\star}+1)}{\delta \varepsilon}\bigg)
    \end{align*}
    samples per state-action pair
    and terminates after at most $\log_2(N)$ outer iterations.
    Upon termination Algorithm \ref{alg:eps_based_alg} returns a policy $\pihat$ and estimates $ \Ub, \Lb$ such that (elementwise)
    \begin{align*}
        \Lb \one \leq \rho^{\pihat} \leq \rho^\star \leq \Ub \one \qquad \text{and} \qquad \Ub - \Lb \leq \varepsilon.
    \end{align*}
    In particular, we have 
    \begin{align*}
       \rho^{\pihat} \geq \rho^\star -\varepsilon \one.
    \end{align*}
\end{thm}
In fact, as shown in Lemma \ref{lem:CI_validity_eps_based}, the quantities $\Ub_i(\gamma)$ are valid upper bounds of $\rho^\star$ for all $i$ and $\gamma \in \hzns_i$, so the algorithm would still be correct if we instead used $\min_{i, \gamma \in \hzns_i} \Ub_i(\gamma) - \max_{i', \gamma' \in \hzns_i} \Lb_{i'}(\gamma') \leq \varepsilon$ as a termination condition; that is, we could use a different $\gamma, i$ to compute the upper bound than the lower bound. The output policy should correspond to the best lower bound, that is, the $\pit_{\gamma_i, i}$ such that $i, \gamma \in \argmax_{i', \gamma' \in \hzns_i} \Lb_{i'}(\gamma')$. Also, within each outer-level iteration $i$ of Algorithm~\ref{alg:eps_based_alg} we consider $O(\log n_i)$ values of $\gamma$, so by Theorem~\ref{thm:eps_based_alg} the algorithm terminates after $\Otilde(1)$ total iterations. 

Compared to the algorithm of \cite{zurek_span-based_2025}, which also takes a discounted reduction approach but uses prior knowledge of $\spannorm{h^\star}$ to set the discount factor, the $\gammahat_i$ chosen by our algorithm should not be seen as implicitly estimating $\spannorm{h^\star}$ and in fact cannot be used to do so (consistent with known hardness results on span estimation~\citep{tuynman_finding_2024,zurek_span-based_2025}). Rather, $\gammahat_i$ can be understood as balancing bounds on certain approximation and estimation error terms, potentially in a superior way than the $\spannorm{h^\star}$-knowledge-based choice for non-worst-case instances.
See Appendix \ref{sec:gammahat_hstar} for further discussion of the relationship between $\spannorm{h^\star}$ and the $\gammahat_i$ computed in each iteration of our algorithm.

\subsection{Span Penalization and Oracle Inequalities}
\label{sec:span_regularization}

Finally, we return to the fixed-$n$ setting and the goal of implementing the  formulation~\eqref{eq:SVP_analogue} that resembles sample variance penalization. As we see momentarily, doing so allows us to obtain a stronger ``oracle inequality'' that optimally trades off suboptimality and complexity.

By superfluously introducing a span constraint of the form $\tspannorm{\Vhat_\gamma^\pi} \leq \spb$ to~\eqref{eq:SVP_analogue} and optimizing over $\spb$, we obtain the equivalent optimization problem
\begin{align}
    \text{~\eqref{eq:SVP_analogue}} 
    \;\equiv\; \max_{\gamma, \spb} \max_{\pi : \tspannorm{\Vhat_\gamma^\pi} \leq \spb}  (1- \gamma) \min_s \Vhat_\gamma^\pi(s) - \sqrt{\frac{\spb + 1}{n}}. \nonumber
\end{align}
These manipulations are useful if, for fixed $(\gamma, \spb)$, we are able to solve the span-constrained optimization problem $\max_{\pi : \tspannorm{\Vhat_\gamma^\pi} \leq \spb}  \Vhat_\gamma^\pi$. A natural approach is to attempt to apply value iteration with a span truncation step. This is inspired by \cite{fruit_efficient_2018}, who first introduced this truncation operator and combined it with average-reward (undiscounted) value iteration to solve a certain bias-constrained gain optimization problem. Thus we define the span truncation operator $\Clip_{\spb} : \R^{\S} \to \R^{\S}$ where
    \begin{align}
        \Clip_\spb(V)(s) := \begin{cases}
            V(s) & \text{if $V(s) \leq \spb + \min_{s'} V(s')$} \\
            \spb + \min_{s'} V(s') & \text{otherwise,}
        \end{cases}
    \end{align}
or equivalently $\Clip_\spb(V) = \min\{V, (\spb+\min_{s'} V(s'))\one \}$, where the outer $\min$ is elementwise. By combining $\Clip_{\spb}$, which is $\infnorm{\cdot}$-nonexpansive, with the discounted Bellman operator $\T_\gamma(V):=M(r+\gamma PV)$, we can now define our \emph{Span-Constrained Planning} subroutine, given as Algorithm \ref{alg:span_constrained_planning}.

\begin{algorithm}[h]
\caption{Span-Constrained Planning}
\label{alg:span_constrained_planning}
\begin{algorithmic}[1]
\Require Discounted MDP $(P, r, \gamma)$, span constraint bound $\spb > 0$, target error $\varepsilon > 0$
\State Form clipped discounted Bellman operator $\L := \Clip_{\spb} \circ \T_\gamma$
\State Set initial point $V^0 = \zero \in \R^{\S}$, total iteration count $T = \left \lceil \frac{\log (\frac{3 }{(1-\gamma)^2 \varepsilon})}{1-\gamma} \right \rceil$
\For{$t = 0, \dots, T-1$}
\State $V^{t+1} = \L(V^t)$
\EndFor
\State Set $\pihat(s) \in \argmax_{a \in \A} r(s,a) + \gamma P_{sa} V^T$ for all $ s \in \S$
\State Compute minimum state value $m = \min_{s \in \S} V^T(s)$
\State Set truncated reward $\rt(s,a) = \min \{m + \spb - \gamma P_{sa} V^T, r(s,a) \}$ for all $s \in \S$, $a \in \A$
\State \Return policy $\pihat$, approximate value function $V^T$, truncated reward $\rt$
\end{algorithmic}
\end{algorithm}

Now we discuss why this subroutine returns a truncated reward $\rt$ and whether it solves the aforementioned span-constrained planning problem $\max_{\pi : \tspannorm{V_\gamma^\pi} \leq \spb}  V_\gamma^\pi$. 
(Here we discuss value functions based on a generic $P$ rather than $\Phat$ for clarity.)
While the clipped Bellman operator $\L = \Clip_{\spb} \circ \T_\gamma$ has a unique fixed point $V_{\gamma, \spb}^\star$, this fixed point generally may not satisfy any Bellman equation $V_{\gamma, \spb}^\star = M^\pi(r + \gamma P V_{\gamma, \spb}^\star)$ for any policy $\pi$. This is to be expected, because such a Bellman equation would imply that policy $\pi$ has value function $V_\gamma^\pi = V_{\gamma, \spb}^\star$ and thus $\tspannorm{V_\gamma^\pi} \leq \spb$, but it is possible that no policies have value functions with span $\leq \spb$. (For example consider an MDP with $A=1$ where the only policy has a value function with large span.) The truncated reward $\rt$ remedies both of these closely related issues: it is defined to (approximately) satisfy a Bellman equation $V_{\gamma, \spb}^\star = M^{\pihat}(\rt + \gamma P V_{\gamma, \spb}^\star)$ for some $\pihat$, which would then imply that in the DMDP $(P, \rt, \gamma)$ with the truncated reward, the value function $V_{\gamma, \rt}^{\pihat}$ of policy $\pihat$ does indeed have $\tspannorm{V_{\gamma, \rt}^{\pihat}} \leq \spb$. 
Thus, we do not actually solve $\max_{\pi : \tspannorm{V_\gamma^\pi} \leq \spb}  V_\gamma^\pi$, but instead the relaxed problem $\max_{\pi, \rt : \tspannorm{V_{\gamma,\rt}^\pi} \leq \spb, \rt \leq r}  V_{\gamma, \rt}^\pi$, which leads to a value function $V_{\gamma,\rt}^{\pihat}$ that is at least as large as $\max_{\pi : \tspannorm{V_\gamma^\pi} \leq \spb}  V_\gamma^\pi$ yet still has span bounded by $\spb$.
And, since elementwise $\rt \leq r$, the actual value function of $\pihat$, $V_{\gamma}^{\pihat}$, is lower-bounded by $V_{\gamma, \rt}^{\pihat}$, which is still compatible with the lower-bound-based approach taken in Algorithm \ref{alg:span_regularization}, which we are now prepared to develop. 

The key properties of the Span-Constrained Planning Algorithm \ref{alg:span_constrained_planning} are summarized in the following lemma.
\begin{lem}
    \label{lem:span_constrained_planning_subroutine}
    The operator $\L$ defined in Algorithm~\ref{alg:span_constrained_planning} is $\gamma$-contractive and has unique fixed point $V_{\gamma, \spb}^\star$. 
    Moreover, the $\pihat, V^T$, and $ \rt$ returned by Algorithm \ref{alg:span_constrained_planning} satisfy
    \begin{enumerate}[nosep]
        \item (proximity to exact fixed point) $\tinfnorm{V^T - V^\star_{\gamma, \spb}} \leq \varepsilon$;
        \item (near-feasibility of $\rt$ and $\pihat$) $\rt \leq r$, $\tinfnorm{V^{\pihat}_{\gamma, \rt} - V^\star_{\gamma, \spb}} \leq \varepsilon$, and $\tspannorm{V^{\pihat}_{\gamma, \rt}}\leq \spb + 2\varepsilon$;
        \item (near-optimality of $\pihat$) for any policy $\pi'$ and reward function $r' \leq r$ such that $\tspannorm{V_{\gamma, r'}^{\pi'}} \leq \spb$, we have $V^\star_{\gamma, \spb} \geq V_{\gamma, r'}^{\pi'}$ and $V_\gamma^{\pihat} \geq V^{\pihat}_{\gamma, \rt} \geq V^\star_{\gamma, \spb} - \varepsilon\one \geq V_{\gamma, r'}^{\pi'} - \varepsilon\one$.
    \end{enumerate}
\end{lem}

We can now present our final Algorithm \ref{alg:span_regularization}, titled \emph{Empirical Span Penalization}. In summary, it can be understood as solving the problem
\begin{align*}
    \max_{\gamma, \spb} \max_{\pi, \rt : \tspannorm{\Vhat_{\gamma,\rt}^\pi} \leq \spb, \rt \leq r}  \min_s \;\; \Vhat_{\gamma, \rt}^\pi(s) - \sqrt{\frac{\spb  + 1}{n}}.
\end{align*}
For reasons apparent in the proof sketches in Section \ref{sec:proof_sketch}, it is unclear whether there is any provable benefit of optimizing over all $\gamma$ as opposed to choosing $\gamma$ such that $\frac{1}{1-\gamma} \approx \sqrt{n \spb}$, so we opt to only optimize $\spb$ and simply set $\gamma$ in this way. Like for our previous algorithms, we only need to try $O(\log n)$ values of $\spb$ to approximately solve the problem.

\begin{algorithm}[h]
\caption{Empirical Span Penalization}
\label{alg:span_regularization}
\begin{algorithmic}[1]
\Require Sample size per state-action pair $n$
\For{each state-action pair $(s,a) \in \S \times \A$}
\State Collect $n$ samples $S^1_{s,a}, \dots, S^n_{s,a}$ from $P(\cdot \mid s,a)$
\State Form the empirical transition kernel $\Phat(s' \mid s, a) = \frac{1}{n}\sum_{i=1}^n \ind\{S^i_{s,a} = s'\}$, for all $s' \in \S$
\EndFor
\For{$i = 2, \dots, \lceil \log_2 n \rceil$}
\State Set span constraint bound $\spb_i = 2^i$
\State Set discount factor $\gamma_i$ such that $\frac{1}{1-\gamma_i} = \max \big\{\frac{\sqrt{n \spb_i}}{\alpha(\delta, n) 2\sqrt{2}}, 1\big\}$
\State Obtain policy $\pit_i$, value function $\Vt_i$, truncated reward $\rt_i$ from the Span-Constrained Planning Algorithm~\ref{alg:span_constrained_planning} with input DMDP $(\Phat, r, \gamma_i)$, span constraint bound $\spb_i$, and target error $\frac{1}{n}$ \label{alg:span_constrained_planning_step}
\State Compute objective value $\Objsp(i) := (1-\gamma_i)\min_{s} \Vt_i(s) -  \alpha(\delta, n) \sqrt{\frac{\spb_i + 1}{n}}$
\EndFor
\State Find $\widehat{i} = \argmax_{i} \Objsp(i)$
\State \Return policy $\pihat := \pit_{\widehat{i}}$, gain lower bound $\rhohat := \max\{\Objsp(\,\widehat{i}\,), 0\}\one$
\end{algorithmic}
\end{algorithm}

\begin{thm}
\label{thm:span_regularization_performance}
    Suppose $P$ is weakly communicating. For some constant $C_4$, with probability at least $1 - \delta$, the policy $\pihat$ and gain lower bound $\rhohat$ output by Algorithm \ref{alg:span_regularization} satisfy (elementwise)
    \begin{align*}
        \rho^{\pihat} 
        \geq \rhohat 
        \geq \sup_{\pi : \rho^\pi \textup{ constant}} \bigg\{ \rho^\pi  - C_4 \alpha(\delta, n) \sqrt{\frac{\tspannorm{h^\pi}+1}{n}} \one \bigg\} .
    \end{align*}
\end{thm}
Theorem \ref{thm:span_regularization_performance} shows that the output policy $\pihat$ satisfies an ``oracle inequality'': $\pihat$ competes with the constant-gain policy $\pi$ that has the best tradeoff between suboptimality and complexity, as measured by $\rho^\star - \rho^{\pihat}$ and $\tspannorm{h^\pi}$, respectively. This bound, when written in the equivalent form
    \begin{align*}
        \rho^\star - \rho^{\pihat} 
        \le \inf_{\pi : \rho^\pi \textup{ constant}} \bigg\{ (\rho^\star - \rho^\pi)  + C_4 \alpha(\delta, n) \sqrt{\frac{\tspannorm{h^\pi}+1}{n}} \one \bigg\} ,
    \end{align*}
resembles the oracle inequalities \citep{deheuvels_lectures_2007, koltchinskii_oracle_2011} from the statistics literature that feature a similar tradeoff.\footnote{An oracle inequality controls the error of a statistical estimator in terms of that of an oracle that selects the optimal model by trading off approximation error and estimation error.}
It is immediate that Theorem \ref{thm:span_regularization_performance} is at least as strong as Theorem \ref{thm:n_based_alg} by setting $\pi$ to be the Blackwell-optimal policy $\pistar$, since $\rho^{\pistar} = \rho^\star$ is constant and $\tspannorm{h^{\pistar}} = \tspannorm{h^\star}$ by definition. 

In fact, since there may generally be many gain-optimal $\pi$'s with $h^\pi \neq h^\star$, Theorem \ref{thm:span_regularization_performance} implies
\begin{align*}
    \rho^{\pihat} \geq \rho^\star - C_4 \alpha(\delta, n) \sqrt{\frac{\inf_{\pi: \rho^\pi = \rho^\star} \tspannorm{h^\pi}+1}{n}} \one,
\end{align*}
thus replacing $\tspannorm{h^\star}$ with the smallest bias of any gain-optimal policy. The above bound does not contradict the minimax lower bound rate of $\sqrt{\tspannorm{h^\star}/n}$ \citep{wang_near_2022}, which is based on worst-case instances, but we present an example in Appendix \ref{sec:examples} where $\inf_{\pi: \rho^\pi = \rho^\star} \tspannorm{h^\pi}$ can be arbitrarily smaller than $\tspannorm{h^\star}$. To the best of our knowledge, no other algorithm can achieve such a bound. 
Another situation where Theorem~\ref{thm:span_regularization_performance} outperforms the minimax rate is when all gain-optimal policies have large span relative to $n$ but there is some near-optimal policy with much smaller span, such that learning a near-optimal policy is still possible. We provide such an example in Appendix~\ref{sec:examples}.
In practice, RL is commonly applied to problems for which exact optimal policies are extremely complicated, and yet such problems may be solved to a reasonable degree of optimality by following relatively simple heuristics. Our result shows that it is possible to be automatically adaptive to the policy with the best tradeoff of complexity and suboptimality.

\section{Proof Sketches}
\label{sec:proof_sketch}

In this section, we outline the key ideas in the proofs of our main theorems.

\subsection{Proof sketch for Theorem \ref{thm:n_based_alg}} 
We start by outlining the ideas behind Algorithm \ref{alg:n_based_alg} and  Theorem \ref{thm:n_based_alg} for the fixed-$n$ setting. 
First we review the approach of \cite{zurek_span-based_2025} and simplifications due to \cite{zurek_plug-approach_2024}, which achieve the same rate as Theorem \ref{thm:n_based_alg} but require knowing $\tspannorm{h^{\star}}$. 
The algorithm of \cite{zurek_span-based_2025} chooses a discount factor $\gamma^\star$ in a way that depends crucially on $\tspannorm{h^{\star}}$, and then solves the $\gamma^\star$-discounted empirical MDP $(\Phat, r, \gamma^\star)$, where $\Phat$ is constructed as in Algorithm~\ref{alg:n_based_alg}. 
The resulting policy $\pihat_{\gamma^\star}$ is shown to be near-optimal for the AMDP $(P, r)$ using the reduction from \cite{wang_near_2022}. 
We now illustrate how this optimally-chosen $\gamma^\star$ trades off certain approximation and learning error terms. Letting $\pihat_\gamma$ be the policy output by the above-described procedure but with a general discount factor $\gamma$, it is shown that 
    \begin{align}
        \infnorm{\rho^{\pihat_\gamma} - \rho^\star} &\overset{(i)}{\lesssim} (1-\gamma) \left( \tspannorm{h^{\star}} + \|V_\gamma^{\pihat_\gamma} - V_\gamma^\star\|_\infty \right) \nonumber\\
        & \overset{(ii)}{\lesssim} (1-\gamma) \Bigg( \tspannorm{h^{\star}} + \frac{1}{1-\gamma} \sqrt{\frac{\tspannorm{V_\gamma^\star} + \tspannorm{\Vhat_\gamma^\star} +1}{n}} \Bigg) \nonumber\\
        & \overset{(iii)}{\lesssim} (1-\gamma)  \tspannorm{h^{\star}} +  \sqrt{\frac{\tspannorm{h^{\star}} + \tspannorm{\Vhat_\gamma^\star} +1}{n}}, \label{eq:proof_sketch_1a}
    \end{align}  
    where the notation $\lesssim$ ignores constant/log factors, (i) is the AMDP-to-DMDP reduction of \cite{wang_near_2022}, (ii) uses \citet[Theorem 9]{zurek_plug-approach_2024} to bound $\tinfnorm{V_\gamma^{\pihat_\gamma} - V_\gamma^\star}$ in terms of the variance parameters $\tspannorm{V_\gamma^\star}, \tspannorm{\Vhat_\gamma^\star}$, and (iii) uses $\tspannorm{V_\gamma^\star} \lesssim \tspannorm{h^{\star}}$ due to \citet[Lemma 2]{wei_model-free_2020}. 
    To proceed, one can invoke \citet[Lemma 12]{zurek_plug-approach_2024} (itself summarizing a key step due to \citealt{zurek_span-based_2025} similar to our Lemma \ref{lem:recursive_err_bds_abstract}), which controls the variance parameter $\tspannorm{\Vhat_\gamma^\star}$ associated with the empirically optimal policy $\pihat$ like
    \begin{align}
         \bigspannorm{\Vhat_\gamma^\star} \lesssim \bigspannorm{V_\gamma^\star} + \infnorm{\Vhat_\gamma^\star - V_\gamma^\star} \lesssim \tspannorm{h^{\star}} + \frac{1}{1-\gamma}\sqrt{\frac{\tspannorm{h^{\star}} + \tspannorm{\Vhat_\gamma^\star} +1}{n}}.  \label{eq:proof_sketch_2}
    \end{align}
    Solving this recursion yields $\tspannorm{\Vhat_\gamma^\star} \lesssim \tspannorm{h^{\star}} + \frac{1}{(1-\gamma)^2 n}$. Plugging back into~\eqref{eq:proof_sketch_1a} gives the following bound and in particular the term $T_3$: 
    \begin{align}
        \infnorm{\rho^{\pihat_\gamma} - \rho^\star} \lesssim
        \underbrace{(1-\gamma)  \tspannorm{h^{\star}}}_{T_1} +  \underbrace{\sqrt{\frac{\tspannorm{h^{\star}} +1}{n}}}_{T_2} + \underbrace{\frac{1}{(1-\gamma)n}}_{T_3} \label{eq:proof_sketch_1}
    \end{align}
    Intuitively, this argument is reminiscent of localization techniques in statistical learning (e.g., \citealt{bartlett_local_2005}) since it controls the variance parameters associated with near-optimal policies in terms of the variance of an optimal policy. For this reason we call $T_3$ the \emph{localization error}. Term $T_1$ is the approximation error due to the DMDP reduction, and term $T_2$ is independent of $\gamma$ and is exactly the desired minimax rate. From~\eqref{eq:proof_sketch_1} it is immediate that choosing $\gamma^\star$ so that $\frac{1}{1-\gamma^\star} \approx \sqrt{n (\tspannorm{h^{\star} }+1)}$ will balance all three terms and achieve the minimax rate.

    Now we derive an algorithm which does not require prior knowledge of $\tspannorm{h^{\star}}$. The simplest idea is to attempt to estimate $\tspannorm{h^{\star}}$ and plug the estimate into the formula for $\gamma^\star$. However, estimating $\tspannorm{h^{\star}}$ up to a constant factor is shown to be impossible in the worst case \citep{zurek_span-based_2025,tuynman_finding_2024}.
    Moreover, \citet[Theorem 14]{zurek_plug-approach_2024} provides an example where choosing $\gamma$ too large provably fails to obtain the minimax-optimal rate, suggesting that the localization error term $T_3$ could not be removed with a different analysis. 
    A partially successful approach is to attempt to replace the terms in~\eqref{eq:proof_sketch_1} with \emph{observable/estimable} upper bounds that depend on $\gamma$, because observability would enable us to minimize these upper bounds over $\gamma$. Doing a similar ``localization'' approach as in~\eqref{eq:proof_sketch_2}, but instead upper-bounding $\tspannorm{V_\gamma^\star}$ in terms of $\tspannorm{\Vhat_\gamma^\star}$, one obtains $\tspannorm{V_\gamma^\star} \lesssim \tspannorm{\Vhat_\gamma^\star} + \frac{1}{(1-\gamma)^2 n}$. Using this bound but otherwise following the derivation of~\eqref{eq:proof_sketch_1}, we can get
    \begin{align}
        \infnorm{\rho^{\pihat_\gamma} - \rho^\star} 
        & \lesssim \underbrace{(1-\gamma)  \tspannorm{h^{\star}}}_{T_1} +  \underbrace{\sqrt{\frac{\tspannorm{\Vhat_\gamma^\star} +1}{n}}}_{T_2'} + \underbrace{\frac{1}{(1-\gamma)n}}_{T_3} , \label{eq:proof_sketch_3}
    \end{align}
    where now both $T_3$ and $T_2'$ are observable, and based on~\eqref{eq:proof_sketch_2}, there are some choices of $\gamma$ (including $\gamma^\star$) for which $T_2'$ will not be too much larger than $T_2$. However, $T_1$ still depends on the unknown $\tspannorm{h^{\star}}$.
    This approach can actually be salvaged with one key modification and forms the basis of our result (Theorem \ref{thm:eps_based_alg}), but here we take a slightly different approach to address term $T_1$.
    Since $T_1$ is decreasing monotonically in $\gamma$, we believe \emph{Lepski's trick} \citep{lepski1997adaptive} could actually be applied to estimate $\rho^\star$ with the minimax rate, but obtaining a policy with a matching order of suboptimality appears to require more work.

    In particular, instead of optimizing a bound on the suboptimality $\tinfnorm{\rho^{\pihat_\gamma} - \rho^\star}$ of the policy $\pihat_\gamma$ that is optimal for the empirical DMDP $(\Phat, r, \gamma)$,  we optimize a lower bound on the performance of $\pihat_\gamma$ since we are considering the fixed-$n$ setting. For any $\pi$ and any $\gamma$, it is known that $\rho^\pi \geq (1-\gamma) V_\gamma^{\pi} - (1-\gamma) \tspannorm{V_\gamma^{\pi}}\one$.\footnote{We actually refine this slightly to $\rho^\pi \geq (1-\gamma) \min_s V_\gamma^{\pi}(s) \one$, which is what appears in our algorithms.} The significance of this statement applied to our problem is that if we only desire a lower bound for $\rho^{\pihat_\gamma}$, then the approximation error term $T_1$ can be replaced with $(1-\gamma) \tspannorm{V_\gamma^{\pihat_\gamma}}$, which can be estimated.
    Specifically, using the bound $\tinfnorm{V_\gamma^{\pihat_\gamma} - \Vhat_\gamma^{\star} } \lesssim \frac{1}{1-\gamma} \sqrt{\frac{\tspannorm{\Vhat_\gamma^{\star}} + 1}{n}}$, we have
    \begin{align}
        \rho^{\pihat_\gamma} &\geq (1-\gamma) V_\gamma^{\pihat_\gamma} - (1-\gamma) \tspannorm{V_\gamma^{\pihat_\gamma}}\one \geq (1-\gamma) \Vhat_\gamma^{\star} - (1-\gamma) \Big(\tspannorm{\Vhat_\gamma^{\star}} - C \tinfnorm{V_\gamma^{\pihat_\gamma} - \Vhat_\gamma^{\star} }\Big) \one \nonumber\\
        & \geq (1-\gamma) \Vhat_\gamma^{\star} - (1-\gamma) \tspannorm{\Vhat_\gamma^{\star}}\one - C' \sqrt{\frac{\tspannorm{\Vhat_\gamma^{\star}} + 1}{n}} \one , \label{eq:proof_sketch_4}
    \end{align}
    where $C, C' \leq \Otilde(1)$.
    Let $L(\gamma)$ be the minimum entry of the RHS of~\eqref{eq:proof_sketch_4}. Then $L(\gamma)$ is observable and we can show $L(\gamma^\star) \geq \rho^\star - \sqrt{\frac{\tspannorm{h^{\star}} + 1}{n}}$, so returning the policy $\pi_{\gammahat}$ where $\gammahat = \argmax_\gamma L(\gamma)$, we obtain the minimax rate. (No explicit localization error term appears in~\eqref{eq:proof_sketch_4}, but lower-bounding $L(\gamma)$ requires the ``localization bound'' $\tspannorm{\Vhat_\gamma^\star} \lesssim \tspannorm{h^{\star}} + \frac{1}{(1-\gamma)^2 n}$, so it still appears implicitly in the size of $\tspannorm{\Vhat_\gamma^\star}$.) This is essentially our Theorem \ref{thm:n_based_alg}.

    \subsection{Proof sketch for Theorem \ref{thm:eps_based_alg}} 
    Next we describe the ideas behind Algorithm \ref{alg:eps_based_alg} and Theorem \ref{thm:eps_based_alg} for the fixed-$\varepsilon$ setting. It suffices to develop a method of bounding the suboptimality $\tinfnorm{\rho^{\pihat} - \rho^\star}$ within the fixed-$n$ setting, as we can then double the dataset size until the suboptimality bound is $\leq \varepsilon$. 
    We return to~\eqref{eq:proof_sketch_3} and the problematic term $T_1$, which originates from the AMDP-to-DMDP reduction of \cite{wang_near_2022}. We show that actually $\tspannorm{h^{\star}}$ can be replaced with $\tspannorm{V_\gamma^\star}$, that is,
    \begin{align}
        \infnorm{\rho^{\pihat_\gamma} - \rho^\star} & \lesssim (1-\gamma) \left( \tspannorm{V_\gamma^\star} + \infnorm{V_\gamma^{\pihat_\gamma} - V_\gamma^\star} \right).\label{eq:proof_sketch_5}
    \end{align}
    See Lemma \ref{lem:AMDP_DMDP_relationships} for the precise statement. We could essentially recover the reduction result of \cite{wang_near_2022} (step (i) of~\eqref{eq:proof_sketch_1}) by using $\tspannorm{V_\gamma^\star} \lesssim \tspannorm{h^{\star}}$, but the subtle improvement of~\eqref{eq:proof_sketch_5} gives a new approximation error term $(1-\gamma)\tspannorm{V_\gamma^\star}$ which, with aforementioned arguments, can be bounded by fully observable quantities involving $\Vhat^\star_\gamma$. Specifically,~\eqref{eq:proof_sketch_3} can be replaced with
    \begin{align}
        \infnorm{\rho^{\pihat_\gamma} - \rho^\star} 
        & \lesssim (1-\gamma)  \tspannorm{\Vhat^\star_\gamma} +  \sqrt{\frac{\tspannorm{\Vhat_\gamma^\star} +1}{n_i}} + \frac{1}{(1-\gamma)n_i} , \label{eq:proof_sketch_6}
    \end{align}
    where we now use a variable dataset size $n_i$, since we intend to double the size until a termination condition is satisfied.
    Comparing with our Algorithm \ref{alg:eps_based_alg}, the RHS of~\eqref{eq:proof_sketch_6} is essentially the difference between the upper and lower bounds defined in the algorithm, that is the quantity $\Ub_i(\gamma) - \Lb_i(\gamma)$. Thus Algorithm \ref{alg:eps_based_alg} is slightly more useful than the procedure sketched here, since it provides valid upper and lower bounds rather than just the difference between such bounds, but this does not require much more work. 
    Similarly to the previous algorithm, when $\gamma = \gamma^\star$, the bound~\eqref{eq:proof_sketch_6} will be $\lesssim \sqrt{(\tspannorm{h^{\star}} + 1)/n}$ thanks to the ``localization bound'' $\tspannorm{\Vhat_\gamma^\star} \lesssim \tspannorm{h^{\star}} + \frac{1}{(1-\gamma)^2 n}$ arising from~\eqref{eq:proof_sketch_2}. This concludes our motivation and proof sketch of Theorem~\ref{thm:eps_based_alg}. Also see Appendix \ref{sec:gammahat_hstar} for more discussion on the relationship between $\spannorm{h^\star}$ and the $\gammahat_i$ that minimizes the right hand side of~\eqref{eq:proof_sketch_6} over $\gamma$.

    \subsection{Proof sketch of Theorem~\ref{thm:span_regularization_performance}} 
    For Theorem \ref{thm:span_regularization_performance}, our starting point is a lower bound similar to~\eqref{eq:proof_sketch_4} used within Algorithm \ref{alg:n_based_alg}, but with a few key differences. Suppose we apply the Span-Constrained Planning Algorithm \ref{alg:span_constrained_planning} to the empirical DMDP $(\Phat, r, \gamma)$ with span constraint $\spb$ and, for simplicity, an arbitrarily small target error, to get the policy $\pihat_{\spb}$, truncated reward $\rt$, and the exact fixed point $\Vhat^\star_{\gamma, \spb} = \Vhat^{\pihat_{\spb}}_{\gamma, \rt}$. We also fix a comparator policy $\pi$ such that $\spannorm{h^\pi} + 1 \leq \frac{\spb}{4}$ and $\rho^\pi$ is constant, which enable us to show that $\tspannorm{V_\gamma^\pi} \lesssim \tspannorm{h^\pi}$ for any $\gamma$ and that $\tinfnorm{\Vhat_\gamma^\pi - V_\gamma^\pi} \lesssim \sqrt{\spb / (1-\gamma)^2 n}$ using standard techniques \citep{zurek_plug-approach_2024}. 
    By setting $\gamma$ so that $\frac{1}{1-\gamma} \lesssim \sqrt{ n \spb}$, we can ensure that $\tspannorm{\Vhat_\gamma^\pi} \leq \spb $ (by bounding it in terms of $\tspannorm{V_\gamma^\pi} \lesssim \tspannorm{h^\pi} \lesssim \spb$ and $\tinfnorm{\Vhat_\gamma^\pi - V_\gamma^\pi}$), which is essential to ensure that $\pi$ is ``feasible'' for the empirical span-constrained problem, guaranteeing that $\Vhat^{\pihat_{\spb}}_{\gamma, \rt} \geq \Vhat_\gamma^\pi$ by Lemma \ref{lem:span_constrained_planning_subroutine}. These assumptions and condition on $\gamma$ also imply that $\Vhat^\pi_\gamma \geq \frac{1}{1-\gamma}\rho^\pi - O(\spb)$. Combining all these steps we can show that
    \begin{align*}
        \rho^{\pihat_{\spb}} &\geq (1-\gamma) \min_s V_{\gamma}^{\pihat_{\spb}}(s)\one \overset{\text{(I)}}{\geq} (1-\gamma) \min_s V_{\gamma, \rt}^{\pihat_{\spb}}(s)\one \\
        &\geq (1-\gamma) \min_s \Vhat_{\gamma, \rt}^{\pihat_{\spb}}(s)\one - (1-\gamma)\infnorm{\Vhat_{\gamma, \rt}^{\pihat_{\spb}} - V_{\gamma, \rt}^{\pihat_{\spb}}}\one \\
        &\overset{\text{(II)}}{\geq} (1-\gamma) \min_s \Vhat_{\gamma, \rt}^{\pihat_{\spb}}(s)\one - C \sqrt{\frac{\spb + 1}{n}} 
        \overset{\text{(III)}}{\geq} (1-\gamma) \min_s \Vhat_{\gamma}^{\pi}(s)\one - C \sqrt{\frac{\spb + 1}{n}} \\
        & \overset{\text{(IV)}}{\geq} \rho^\pi  - (1-\gamma) C'\spb \one  - C \sqrt{\frac{\spb + 1}{n}} , 
    \end{align*}
    where $C, C' \leq \Otilde(1)$. Here
    (I) follows from  $\rt \leq r$, and (II) follows from the error bound
    \begin{align}
        \infnorm{\Vhat_{\gamma, \rt}^{\pihat_{\spb}} - V_{\gamma, \rt}^{\pihat_{\spb}}} \lesssim \frac{1}{1-\gamma}\sqrt{\frac{\tspannorm{\Vhat_{\gamma, \rt}^{\pihat_{\spb}}} + 1}{n}} = \frac{1}{1-\gamma}\sqrt{\frac{\spb + 1}{n}}, \label{eq:proof_sketch_7}
    \end{align}
    and (III) and (IV) follow from the aforementioned ``feasibility'' of $\pi$ and the lower bound on $\Vhat_\gamma^\pi$.
    The key new challenge is to develop the error bound~\eqref{eq:proof_sketch_7}, which is complicated due to the statistical dependence between $\Phat$ and $\pihat_{\spb}$. Without the span constraint this is addressed by \cite{agarwal_model-based_2020, li_breaking_2020} through the use of ``absorbing MDP'' constructions, which enable concentration inequalities for $|(\Phat_{sa} - P_{sa}) \Vhat^\star_\gamma|$. Instead, we desire such bounds for $|(\Phat_{sa} - P_{sa}) \Vhat^{\star}_{\gamma, \spb}|$, that is, involving the empirical span-constrained optimal value functions. Based on the contractivity of the span-constrained Bellman operator $\L$, we develop a new absorbing MDP construction for span-constrained value functions, ultimately leading to the desired bound~\eqref{eq:proof_sketch_7}.

\section{Conclusion}

We resolve the problem of achieving the minimax optimal span-based complexity in average-reward RL without prior knowledge of the span or other unknown MDP complexity parameters. Our algorithms apply to both the fixed-sample-size and fixed-suboptimality settings, and moreover achieve optimal tradeoff between complexity and suboptimality, surpassing the minimax lower bound in benign settings. Future directions of interest include generalizing to general/multichain MDPs and extending beyond the tabular simulator setting.

\section*{Acknowledgments}
{Y.\ Chen and M.\ Zurek acknowledge support by National Science Foundation grants CCF-2233152 and DMS-2023239.}

\bibliographystyle{plainnat}
\bibliography{refs2}

\begin{thebibliography}{28}
\providecommand{\natexlab}[1]{#1}
\providecommand{\url}[1]{\texttt{#1}}
\expandafter\ifx\csname urlstyle\endcsname\relax
  \providecommand{\doi}[1]{doi: #1}\else
  \providecommand{\doi}{doi: \begingroup \urlstyle{rm}\Url}\fi

\bibitem[Agarwal et~al.(2020)Agarwal, Kakade, and Yang]{agarwal_model-based_2020}
Alekh Agarwal, Sham Kakade, and Lin~F Yang.
\newblock Model-based reinforcement learning with a generative model is minimax optimal.
\newblock In \emph{Conference on Learning Theory}, pages 67--83. PMLR, 2020.

\bibitem[Azar et~al.(2012)Azar, Munos, and Kappen]{azar_sample_2012}
Mohammad~Gheshlaghi Azar, R{\'e}mi Munos, and Hilbert~J Kappen.
\newblock On the sample complexity of reinforcement learning with a generative model.
\newblock In \emph{Proceedings of the 29th International Coference on International Conference on Machine Learning}, pages 1707--1714, 2012.

\bibitem[Azar et~al.(2013)Azar, Munos, and Kappen]{azar_minimax_2013}
Mohammad~Gheshlaghi Azar, Rémi Munos, and Hilbert~J. Kappen.
\newblock Minimax {PAC} bounds on the sample complexity of reinforcement learning with a generative model.
\newblock \emph{Machine Learning}, 91\penalty0 (3):\penalty0 325--349, June 2013.
\newblock ISSN 1573-0565.
\newblock \doi{10.1007/s10994-013-5368-1}.
\newblock URL \url{https://doi.org/10.1007/s10994-013-5368-1}.

\bibitem[Bartlett and Tewari(2009)]{bartlett_regal_2012}
Peter Bartlett and Ambuj Tewari.
\newblock {REGAL}: a regularization based algorithm for reinforcement learning in weakly communicating mdps.
\newblock In \emph{Uncertainty in Artificial Intelligence: Proceedings of the 25th Conference}, pages 35--42. AUAI Press, 2009.

\bibitem[Bartlett et~al.(2005)Bartlett, Bousquet, and Mendelson]{bartlett_local_2005}
Peter~L. Bartlett, Olivier Bousquet, and Shahar Mendelson.
\newblock Local {Rademacher} complexities.
\newblock \emph{The Annals of Statistics}, 33\penalty0 (4):\penalty0 1497--1537, August 2005.
\newblock ISSN 0090-5364, 2168-8966.
\newblock \doi{10.1214/009053605000000282}.
\newblock URL \url{https://projecteuclid.org/journals/annals-of-statistics/volume-33/issue-4/Local-Rademacher-complexities/10.1214/009053605000000282.full}.
\newblock Publisher: Institute of Mathematical Statistics.

\bibitem[Boone and Zhang(2024)]{boone_achieving_2024}
Victor Boone and Zihan Zhang.
\newblock Achieving {Tractable} {Minimax} {Optimal} {Regret} in {Average} {Reward} {MDPs}, June 2024.
\newblock URL \url{http://arxiv.org/abs/2406.01234}.
\newblock arXiv:2406.01234 [cs].

\bibitem[Deheuvels et~al.(2007)Deheuvels, Del~Barrio, and Van De~Geer]{deheuvels_lectures_2007}
Paul Deheuvels, Eustasio Del~Barrio, and Sara Van De~Geer.
\newblock \emph{Lectures on {Empirical} {Processes}: {Theory} and {Statistical} {Applications}}, volume~6 of \emph{{EMS} {Series} of {Lectures} in {Mathematics}}.
\newblock EMS Press, 1 edition, January 2007.
\newblock ISBN 978-3-03719-027-2 978-3-03719-527-7.
\newblock \doi{10.4171/027}.
\newblock URL \url{https://ems.press/doi/10.4171/027}.

\bibitem[Duchi and Namkoong(2019)]{duchi_variance-based_2019}
John Duchi and Hongseok Namkoong.
\newblock Variance-based regularization with convex objectives.
\newblock \emph{Journal of Machine Learning Research}, 20\penalty0 (1):\penalty0 2450--2504, January 2019.
\newblock ISSN 1532-4435.

\bibitem[Fruit et~al.(2018)Fruit, Pirotta, Lazaric, and Ortner]{fruit_efficient_2018}
Ronan Fruit, Matteo Pirotta, Alessandro Lazaric, and Ronald Ortner.
\newblock Efficient bias-span-constrained exploration-exploitation in reinforcement learning.
\newblock In \emph{International Conference on Machine Learning}, pages 1578--1586. PMLR, 2018.

\bibitem[Jin et~al.(2024)Jin, Gummadi, Zhou, and Blanchet]{jin_feasible_2024}
Ying Jin, Ramki Gummadi, Zhengyuan Zhou, and Jose Blanchet.
\newblock Feasible \${Q}\$-{Learning} for {Average} {Reward} {Reinforcement} {Learning}.
\newblock In \emph{Proceedings of {The} 27th {International} {Conference} on {Artificial} {Intelligence} and {Statistics}}, pages 1630--1638. PMLR, April 2024.
\newblock URL \url{https://proceedings.mlr.press/v238/jin24b.html}.
\newblock ISSN: 2640-3498.

\bibitem[Jin and Sidford(2020)]{jin_efficiently_2020}
Yujia Jin and Aaron Sidford.
\newblock Efficiently solving {MDPs} with stochastic mirror descent.
\newblock In \emph{International Conference on Machine Learning}, pages 4890--4900. PMLR, 2020.

\bibitem[Jin and Sidford(2021)]{jin_towards_2021}
Yujia Jin and Aaron Sidford.
\newblock Towards tight bounds on the sample complexity of average-reward {MDPs}.
\newblock In \emph{International Conference on Machine Learning}, pages 5055--5064. PMLR, 2021.

\bibitem[Kearns and Singh(1998)]{kearns_finite-sample_1998}
Michael Kearns and Satinder Singh.
\newblock Finite-{Sample} {Convergence} {Rates} for {Q}-{Learning} and {Indirect} {Algorithms}.
\newblock In \emph{Advances in {Neural} {Information} {Processing} {Systems}}, volume~11. MIT Press, 1998.
\newblock URL \url{https://proceedings.neurips.cc/paper/1998/hash/99adff456950dd9629a5260c4de21858-Abstract.html}.

\bibitem[Koltchinskii(2011)]{koltchinskii_oracle_2011}
Vladimir Koltchinskii.
\newblock \emph{Oracle {Inequalities} in {Empirical} {Risk} {Minimization} and {Sparse} {Recovery} {Problems}: École d’Été de {Probabilités} de {Saint}-{Flour} {XXXVIII}-2008}, volume 2033 of \emph{Lecture {Notes} in {Mathematics}}.
\newblock Springer Berlin Heidelberg, Berlin, Heidelberg, 2011.
\newblock ISBN 978-3-642-22146-0 978-3-642-22147-7.
\newblock \doi{10.1007/978-3-642-22147-7}.
\newblock URL \url{https://link.springer.com/10.1007/978-3-642-22147-7}.

\bibitem[Lattimore and Szepesvári(2020)]{lattimore_bandit_2020}
Tor Lattimore and Csaba Szepesvári.
\newblock \emph{Bandit algorithms}.
\newblock Cambridge University Press, Cambridge ; New York, NY, 2020.
\newblock ISBN 978-1-108-57140-1.

\bibitem[Lepski and Spokoiny(1997)]{lepski1997adaptive}
Oleg~V. Lepski and Vladimir~G. Spokoiny.
\newblock Optimal pointwise adaptive methods in nonparametric estimation.
\newblock \emph{The Annals of Statistics}, 25\penalty0 (6):\penalty0 2512--2546, 1997.

\bibitem[Li et~al.(2020)Li, Wei, Chi, Gu, and Chen]{li_breaking_2020}
Gen Li, Yuting Wei, Yuejie Chi, Yuantao Gu, and Yuxin Chen.
\newblock Breaking the {Sample} {Size} {Barrier} in {Model}-{Based} {Reinforcement} {Learning} with a {Generative} {Model}.
\newblock In \emph{Advances in {Neural} {Information} {Processing} {Systems}}, volume~33, pages 12861--12872. Curran Associates, Inc., 2020.
\newblock URL \url{https://proceedings.neurips.cc/paper/2020/hash/96ea64f3a1aa2fd00c72faacf0cb8ac9-Abstract.html}.

\bibitem[Li et~al.(2024)Li, Wu, and Lan]{li_stochastic_2024}
Tianjiao Li, Feiyang Wu, and Guanghui Lan.
\newblock Stochastic first-order methods for average-reward markov decision processes.
\newblock \emph{Mathematics of Operations Research}, 2024.

\bibitem[Maurer and Pontil(2009)]{maurer_empirical_2009}
Andreas Maurer and Massimiliano Pontil.
\newblock Empirical {Bernstein} {Bounds} and {Sample} {Variance} {Penalization}, July 2009.
\newblock URL \url{http://arxiv.org/abs/0907.3740}.
\newblock Proc. Computational Learning Theory Conference (COLT 2009). arXiv:0907.3740 [stat] version: 1.

\bibitem[Neu and Okolo(2024)]{neu_dealing_2024}
Gergely Neu and Nneka Okolo.
\newblock Dealing with unbounded gradients in stochastic saddle-point optimization, June 2024.
\newblock URL \url{http://arxiv.org/abs/2402.13903}.
\newblock {Forty-first International Conference on Machine Learning}. arXiv:2402.13903 [cs, math, stat] version: 2.

\bibitem[Puterman(1994)]{puterman_markov_1994}
Martin~L. Puterman.
\newblock \emph{Markov {Decision} {Processes}: {Discrete} {Stochastic} {Dynamic} {Programming}}.
\newblock Wiley {Series} in {Probability} and {Statistics}. Wiley, 1 edition, April 1994.
\newblock ISBN 978-0-471-61977-2 978-0-470-31688-7.
\newblock \doi{10.1002/9780470316887}.
\newblock URL \url{https://onlinelibrary.wiley.com/doi/book/10.1002/9780470316887}.

\bibitem[Tuynman et~al.(2024)Tuynman, Degenne, and Kaufmann]{tuynman_finding_2024}
Adrienne Tuynman, Rémy Degenne, and Emilie Kaufmann.
\newblock Finding good policies in average-reward {Markov} {Decision} {Processes} without prior knowledge, May 2024.
\newblock URL \url{http://arxiv.org/abs/2405.17108}.
\newblock arXiv:2405.17108 [cs].

\bibitem[Wang et~al.(2022)Wang, Wang, and Yang]{wang_near_2022}
Jinghan Wang, Mengdi Wang, and Lin~F. Yang.
\newblock Near {Sample}-{Optimal} {Reduction}-based {Policy} {Learning} for {Average} {Reward} {MDP}, December 2022.
\newblock URL \url{http://arxiv.org/abs/2212.00603}.
\newblock arXiv:2212.00603 [cs].

\bibitem[Wang et~al.(2023)Wang, Blanchet, and Glynn]{wang_optimal_2023}
Shengbo Wang, Jose Blanchet, and Peter Glynn.
\newblock Optimal {Sample} {Complexity} of {Reinforcement} {Learning} for {Mixing} {Discounted} {Markov} {Decision} {Processes}, September 2023.
\newblock URL \url{http://arxiv.org/abs/2302.07477}.
\newblock arXiv:2302.07477.

\bibitem[Wei et~al.(2020)Wei, Jahromi, Luo, Sharma, and Jain]{wei_model-free_2020}
Chen-Yu Wei, Mehdi~Jafarnia Jahromi, Haipeng Luo, Hiteshi Sharma, and Rahul Jain.
\newblock Model-free reinforcement learning in infinite-horizon average-reward markov decision processes.
\newblock In \emph{International conference on machine learning}, pages 10170--10180. PMLR, 2020.

\bibitem[Zhang and Xie(2023)]{zhang_sharper_2023}
Zihan Zhang and Qiaomin Xie.
\newblock Sharper {Model}-free {Reinforcement} {Learning} for {Average}-reward {Markov} {Decision} {Processes}, June 2023.
\newblock URL \url{http://arxiv.org/abs/2306.16394}.
\newblock The Thirty Sixth Annual Conference on Learning Theory. arXiv:2306.16394 [cs].

\bibitem[Zurek and Chen(2024)]{zurek_plug-approach_2024}
Matthew Zurek and Yudong Chen.
\newblock The {Plug}-in {Approach} for {Average}-{Reward} and {Discounted} {MDPs}: {Optimal} {Sample} {Complexity} {Analysis}, October 2024.
\newblock URL \url{http://arxiv.org/abs/2410.07616}.
\newblock International Conference on Algorithmic Learning Theory (ALT 2025). arXiv:2410.07616 [cs].

\bibitem[Zurek and Chen(2025)]{zurek_span-based_2025}
Matthew Zurek and Yudong Chen.
\newblock Span-{Based} {Optimal} {Sample} {Complexity} for {Weakly} {Communicating} and {General} {Average} {Reward} {MDPs}.
\newblock \emph{Advances in Neural Information Processing Systems}, 37:\penalty0 33455--33504, January 2025.
\newblock URL \url{https://proceedings.neurips.cc/paper_files/paper/2024/hash/3acbe9dc3a1e8d48a57b16e9aef91879-Abstract-Conference.html}.

\end{thebibliography}

\clearpage

\appendix

\section{Additional Notation and Guide to Appendices}

In this section, we provide additional notations and definitions, and outline the organization of the remainder of the appendices.

\subsection{Additional Notation}

in the proofs of our main theorems, we make use of some additional notations and standard facts about average-reward MDPs.  
For any policy $\pi$, its gain and bias $\rho^\pi$ and $h^\pi$ satisfy the optimality equations $\rho^\pi = P_\pi \rho^\pi$ and $\rho^\pi + h^\pi = r_\pi + P_\pi h^\pi$, where the second equation is sometimes called the (average-reward) Bellman equation or Poisson equation. 
We let $P_\pi^\infty = \Clim_{T \to \infty} (P_{\pi})^T$ denote the limiting matrix of the Markov chain induced by the policy $\pi$. When this Markov chain is aperiodic, the Cesaro limit, $\Clim$, can be replaced with the usual limit. 
Note that $P_\pi^\infty P_\pi = P_\pi P_\pi^\infty = P_\pi^\infty$ and $\rho^\pi = P_\pi^\infty r_\pi$.
For any policy $\pi$ we define the Bellman consistency operator $\T_\gamma^\pi : \R^{\S} \to \R^{\S}$ by $\T_\gamma^\pi(x) = r_\pi + \gamma P_\pi x$.

For any $x \in \R^{S}$ and any policy $\pi$, we define a policy-specific next-state transition variance vector $\Var_{P_{\pi}}\left[x \right] \in \R^{S}$ as
$\left(\Var_{P_{\pi}} \left[x \right]\right)_s := \sum_{s' \in \S} \left(P_{\pi}\right)_{s, s'} \big[x(s') - \sum_{s''}\left(P_{\pi}\right)_{s, s''}x(s'')\big]^2.$
We use $\infinfnorm{B}$ to denote the $\infnorm{\cdot}$ operator norm of matrix $B$, which is equal to the maximum $\ell^1$-norms of any row of $B$.

\subsection{Complexity Parameters}
\label{sec:complexity_params}

In this section, we provide definitions for the diameter $D$ and the uniform mixing time $\tmix$, which are standard complexity parameters (along with optimal bias span $\spannorm{h^\star}$) used in prior work on average-reward MDPs. We also give the definitions for the quantities $\tspannorm{h^{\pihat}}, \tspannorm{\hhatanc^\star}$, and $\tspannorm{\hanc^{\pihat}}$ appearing in Table \ref{table:AMDPs}, which arise in the results of \cite{neu_dealing_2024} and \cite{zurek_plug-approach_2024}.
The results in the present paper do not involve these parameters; their definitions are included here for completeness and for comparison with our results.

First we define the diameter.
For any state $s \in \S$, let $\eta_{s} = \inf \{t \geq 1: S_t = s \}$ denote the hitting time of state $s$, which is a random variable (in the probability space of trajectories in the MDP $P$).
Then the diameter $D$ is defined as
\begin{align*}
    D := \max_{s_1 \neq s_2} \inf_{\pi\in \PiMD} \E^\pi_{s_1} \left[\eta_{s_2}\right],
\end{align*}
where $\PiMD$ is the set of all Markovian deterministic policies. $D < \infty$ if and only if the MDP is communicating, so in particular it is generally infinite in weakly communicating MDPs.

Now we define the uniform mixing time $\tmix$. This parameter is only defined if we first assume that for all $\pi \in \PiMD$, the Markov chain induced by $P_\pi$ has a unique stationary distribution, which we call $\nu_\pi$ (considered as a row vector in $\R^{\S}$).
Then for any policy $\pi\in \PiMD$, we can define its mixing time
$\tau_\pi := \inf \cig\{t \geq 1 : \max_{s \in \S} \cigl\|e_s^\top \left(P_\pi \right)^t - \nu^\top_\pi \cigr\|_1 \leq \frac{1}{2} \cig\}.$ Finally, we define the uniform mixing time $\tmix := \sup_{\pi\in\PiMD} \tau_\pi$.

It always holds that $\spannorm{h^\star} \leq D$ \citep{bartlett_regal_2012} and $\spannorm{h^\star} \leq 3 \tmix$ \citep{wang_near_2022,zurek_plug-approach_2024}. In general, $D$ and $\tmix$ are not comparable \citep{wang_near_2022}, and they both can be arbitrarily larger than $\spannorm{h^\star}$.

The prior work listed in Table \ref{table:AMDPs}  also involves the quantities $\tspannorm{h^{\pihat}}, \tspannorm{\hhatanc^\star}$ and $\tspannorm{\hanc^{\pihat}}$. These are all dependent on the outputs of certain (randomized) algorithms which introduce them, and they are generally not controlled in terms of $\tspannorm{h^\star}$. In particular, the algorithm described in  \cite[Theorem 4.1]{neu_dealing_2024} returns a policy $\pihat$, and their complexity bounds depend on $\tspannorm{h^{\pihat}}$, the span of the bias of $\pihat$ under the true MDP $(P,r)$.
\cite{zurek_plug-approach_2024} define an \textit{anchored} empirical AMDP $(\Phatanc, r)$ with the transition matrix $\Phatanc := (1-\eta) \Phat + \eta \one e_{s_0}^\top$, where $\eta \in [0,1]$ is an algorithmic parameter, $s_0$ is an arbitrary state and $e_{s_0} \in \R^S$ is the vector with all zeros except for the entry $s_0$ being equal to $1$. 
Letting $\hhatanc^\star$ denote the optimal bias function in this anchored AMDP, and letting $\hanc^{\pihat}$ denote the bias function  of the policy $\pihat$ output by (the anchored and unperturbed version of) Algorithm 1 in \cite{zurek_plug-approach_2024}, their sample complexity guarantees are given in terms of the quantities $\tspannorm{\hhatanc^\star}$ and $\tspannorm{\hanc^{\pihat}}$. Note that \citet[Theorem 14]{zurek_plug-approach_2024} gives an example where both of these terms are much larger than $\tspannorm{h^\star}$.

We refer to \cite{neu_dealing_2024} and \cite{zurek_plug-approach_2024} for the complete definitions of $\tspannorm{h^{\pihat}}, \tspannorm{\hhatanc^\star}$ and $\tspannorm{\hanc^{\pihat}}$ as well as their relationship with other complexity parameters.

\subsection{Remark on $\alpha$}

We discuss the origin of the function $\alpha$ that appears in our algorithms and bounds.
\begin{remark}
    \label{rem:alpha}
    The quantity $\alpha(\delta, n) = 96 \sqrt{ \log \left( {24 SA n^5}/{\delta}\right)} \log_2\left( \log_2 (n + 4) \right)$ appearing in our bounds arises from our choice to make use of concentration bounds developed by \cite{zurek_plug-approach_2024}. This quantity is used within our algorithms, so we make it explicit but did not attempt to optimize it. We note \cite{zurek_plug-approach_2024} also did not optimize constants/log factors, and thus we conjecture that a smaller function $\alpha$ may be sufficient for their inequalities to hold, in which case the improvement could be carried over to our work in a black-box manner.
\end{remark}

\subsection{Guide to Appendices}

In Appendix \ref{sec:gammahat_hstar}, we discuss the relationship between the $\gammahat$ selected by our algorithm and the true optimal bias span $\spannorm{h^\star}$. 
Appendix \ref{sec:proofs} contains the proofs for all of our main theorems, and Appendix \ref{sec:examples} contains examples mentioned in Subsection \ref{sec:span_regularization} of situations where the guarantee of Theorem \ref{thm:span_regularization_performance} could be much better than the minimax rate. Within Appendix \ref{sec:proofs}, we first prove Theorem \ref{thm:eps_based_alg} in Appendix \ref{sec:eps_based_alg_proof}, then Theorem \ref{thm:n_based_alg} in Appendix \ref{sec:n_based_alg_proof}. In Appendix \ref{sec:span_constrained_planning_subroutine_proof} we prove Lemma \ref{lem:span_constrained_planning_subroutine} regarding the properties of the span-constrained planning subroutine Algorithm \ref{alg:span_constrained_planning}, and finally in Appendix \ref{sec:span_regularization_proof} we prove Theorem \ref{thm:span_regularization_performance}.

\section{Relationship Between $\gammahat_i$ and $\spannorm{h^\star}$}
\label{sec:gammahat_hstar}

Recall equation~\eqref{eq:proof_sketch_6} from the proof sketch of Theorem~\ref{thm:eps_based_alg}. There, we see that the discount factor $\gammahat_i$ selected by Algorithm~\ref{alg:eps_based_alg} in a given iteration $i$ is (approximately) the minimizer of a function $\Bhat(\gamma)$ that upper-bounds the suboptimality of a policy $\pihat_\gamma$ in the following manner:
\begin{align*}
        \infnorm{\rho^{\pihat_\gamma} - \rho^\star} 
        & \lesssim \Bhat(\gamma) := (1-\gamma)  \tspannorm{\Vhat^\star_\gamma} +  \sqrt{\frac{\tspannorm{\Vhat_\gamma^\star} +1}{n_i}} + \frac{1}{(1-\gamma)n_i} .
    \end{align*}
We can understand $\Bhat(\gamma)$ by relating it to the following (deterministic) quantity
\begin{align}
    B(\gamma) := (1-\gamma)  \tspannorm{V^\star_\gamma} +  \sqrt{\frac{\tspannorm{V_\gamma^\star} +1}{n_i}} + \frac{1}{(1-\gamma)n_i}. \label{eq:B_defn}
\end{align}
Using an error bound of the form $\infnorm{V_\gamma^\star - \Vhat_\gamma^\star} \lesssim \frac{1}{1-\gamma} \sqrt{\frac{\spannorm{V_\gamma^\star} + \spannorm{\Vhat^\star_\gamma} + 1}{n_i}}$ \citep[Theorem 9]{zurek_plug-approach_2024}, as well as the ``localization'' bound $\tspannorm{V_\gamma^\star} \lesssim \tspannorm{\Vhat_\gamma^\star} + \frac{1}{(1-\gamma)^2 n_i}$ which appears in the proof sketch, we can show that (with high probability) $B(\gamma) \lesssim \Bhat(\gamma)$, and likewise repeating the same steps but with the roles of $\Vhat_\gamma^\star$ and $V_\gamma^\star$ reversed we can also show that $\Bhat(\gamma) \lesssim B(\gamma)$. Therefore, since $\Bhat(\gamma)$ and $B(\gamma)$ are equivalent up to logarithmic factors with high probability, we can understand $\gammahat_i$ by considering the $\gamma$ which minimizes $B(\gamma)$.

We now turn to the ``oracle choice'' of $\gamma^\star$ made in \cite{zurek_plug-approach_2024} (also discussed in our proof sketches) such that $\frac{1}{1-\gamma^\star} \approx \sqrt{n_i (\spannorm{h^\star}+1)}$. Using the fact that $\tspannorm{V_\gamma^\star} \lesssim \tspannorm{h^\star}$ \citep{wei_model-free_2020},
we are guaranteed the final term in the definition~\eqref{eq:B_defn} of $B(\gamma^\star)$ is the largest and equal to $\sqrt{\frac{\spannorm{h^\star}+1}{n_i}}$. This implies that the $\gammahat_i$ selected by minimizing $\Bhat(\gamma)$ (and also the minimizing $\gamma$ of $B(\gamma)$) must satisfy $\frac{1}{1-\gammahat_i} \lesssim \frac{1}{1-\gamma^\star} \approx \sqrt{n_i (\spannorm{h^\star}+1)}$, since for all $\gamma \geq \gamma^\star$ the final term of $B(\gamma)$ is $\gtrsim \frac{1}{(1-\gamma^\star ) n_i} \gtrsim B(\gamma^\star)$. However, $\frac{1}{1-\gammahat_i}$ may potentially be much smaller than $\sqrt{n_i (\spannorm{h^\star}+1)}$, so in particular it is not generally possible to convert $\gammahat_i$ into some constant-factor estimate of $\spannorm{h^\star}$ (which would contradict the hardness of estimating $\spannorm{h^\star}$ as established in \citealt{tuynman_finding_2024, zurek_span-based_2025}).
Indeed, it is possible to compute the minimizer $\tilde{\gamma}^\star$ of $B(\gamma)$ for the MDP instances that were used to show the hardness of estimating $\spannorm{h^\star}$ in \citet[Proof of Theorem 3]{zurek_span-based_2025}. We find that such instances (which have $n_i$ as an input parameter) have $\spannorm{V_\gamma^\star} \lesssim 1 + \frac{1}{1-\gamma} \frac{1}{\sqrt{n_i}}$, and for relatively small $\gamma$ this bound can be much better than the bound $\spannorm{V_\gamma^\star} \lesssim \spannorm{h^\star}$; consequently, our sharper DMDP reduction~\eqref{eq:DMDP_red_opt_policy_imp_st}, which replaces $\spannorm{h^\star}$ with $\spannorm{V_\gamma^\star}$, gives a much better bound (also see~\eqref{eq:proof_sketch_5} in the proof sketches). This fact about $\spannorm{V_\gamma^\star}$ for these instances also implies that the minimizing $\tilde{\gamma}^\star$ has $\frac{1}{1-\tilde{\gamma}^\star} \approx \sqrt{n_i} \ll \sqrt{n_i(\spannorm{h^\star}+1)}$.
Since one of the instances used in \citet[Proof of Theorem 3]{zurek_span-based_2025} has an arbitrarily large bias span $\spannorm{h^\star}$, this in particular implies that $\frac{1}{1-\tilde{\gamma}^\star}$ (and thus also $\frac{1}{1-\gammahat_i}$) may be arbitrarily smaller than $\sqrt{n_i(\spannorm{h^\star} + 1)}$.
This minimizing choice $\tilde{\gamma}^\star$ gives $B(\tilde{\gamma}^\star) \lesssim \frac{1}{\sqrt{n_i}} \ll \sqrt{\frac{\spannorm{h^\star}}{n_i}}$, which implies that these instances, while being worst-case for estimating the optimal bias span, are \emph{not} worst-case for the task of finding a gain-optimal policy or estimating the optimal gain.

In summary, although $\gammahat_i$ is chosen in our algorithm to calibrate certain approximation and estimation error terms at least as well as the ``oracle choice'' $\gamma^\star$ (which depends on $\spannorm{h^\star}$), this $\gammahat_i$ should not be seen as estimating $\spannorm{h^\star}$, but rather as calibrating tighter bounds on approximation/estimation errors; in particular, as discussed in the proof sketches, these tighter bounds replace $\spannorm{h^\star}$ with $\spannorm{V_\gamma^\star}$. 
In fact, the ``oracle choice'' $\frac{1}{1-\gamma^\star} \approx \sqrt{n_i (\spannorm{h^\star}+1)}$ is not only more difficult to compute, but can generally achieve a worse tradeoff than the $\tilde{\gamma}^\star$ which minimizes $B(\gamma)$. 
Still, for worst-case instances such as those appearing in the minimax lower bounds for learning a gain-optimal policy or estimating the optimal gain, the oracle $\gamma^\star$ must indeed approximately minimize $B(\gamma)$, since we know for such instances that $B(\gamma) \gtrsim \sqrt{\spannorm{h^\star}/n_i}$ for all $\gamma$ (since they are hard instances), and also that $B(\gamma^\star) \lesssim \sqrt{\spannorm{h^\star}/n_i}$. 
This observation suggests that for instances that are hard for optimizing/estimating the gain, we may actually be able to estimate $\spannorm{h^\star}$ to a constant factor using $\gammahat_i$, which would imply such instances are \emph{not} hard for the task of span estimation.

\section{Proofs of Main Theorems}
\label{sec:proofs}

In this section, we prove our main theorems. 

\subsection{Proof of Theorem \ref{thm:eps_based_alg}}
\label{sec:eps_based_alg_proof}
The proof has two main steps. First we will show that all upper and lower bounds computed within Algorithm \ref{alg:eps_based_alg} are valid upper/lower bounds on $\rho^\star$ and the gain of the output policy. This implies that whenever the algorithm terminates, it will have output a policy which is $\varepsilon$-optimal. The second step is to show that once a sufficient number of iterations have occurred (and thus a sufficiently large sample size is chosen), the confidence interval corresponding to some discount factor $\gstar$ will be sufficiently small, thus ensuring termination of the algorithm at or before this point.

The following lemma is very similar to results of \cite{wei_model-free_2020} and \cite{wang_near_2022}. While the bound $\infnorm{\rho^\star - (1-\gamma)V_\gamma^\star} \leq (1-\gamma)\spannorm{h^\star}$ has appeared previously \citep[Lemma 2]{wei_model-free_2020}, to the best of our knowledge a bound of the form $\infnorm{\rho^\star - (1-\gamma)V_\gamma^\star} \leq (1-\gamma)\spannorm{V_\gamma^\star}$ has not previously appeared in the literature. As explained in the proof sketches in Section \ref{sec:proof_sketch}, the difference is extremely significant in our situation, since $\spannorm{V_\gamma^\star}$ can be estimated while $\spannorm{h^\star}$ generally cannot be \citep{zurek_span-based_2025, tuynman_finding_2024}. This new result thus provides a computable upper bound for the ``approximation error'' due to reducing the AMDP to a DMDP with a certain discount factor, which can be used algorithmically. We also note that~\eqref{eq:DMDP_red_fixed_policy_imp_st} below is already known and appears in \cite[Lemma 6]{wang_near_2022}. (Their proof has a minor issue because it is not the case that all policies $\pi$ have $\rho^\pi$ which is state-independent/constant in a weakly communicating MDP, but their proof does not actually need this fact and the result is still true.)

\begin{lem}
\label{lem:AMDP_DMDP_relationships}
    Suppose $P$ is weakly communicating. Then
    \begin{align}
        (1-\gamma)\left( \min_s V_\gamma^\star (s) \right) \one \leq \rho^\star \leq (1-\gamma)\left(\max_s V_\gamma^\star (s) \right) \one\label{eq:sharper_DMDP_red_opt}
    \end{align}
    which implies
    \begin{align}
        \infnorm{\rho^\star - (1-\gamma)V_\gamma^\star} \leq (1-\gamma)\spannorm{V_\gamma^\star}. \label{eq:DMDP_red_opt_policy_imp_st}
    \end{align}
    Additionally, for any fixed policy $\pi$, we have that
    \begin{align}
        (1-\gamma)\left(\min_s V_\gamma^\pi (s)\right) \one \leq \rho^\pi \leq (1-\gamma)\left(\max_s V_\gamma^\pi (s)\right) \one,\label{eq:sharper_DMDP_red_fixed}
    \end{align}
    which implies
    \begin{align}
        \infnorm{\rho^\pi - (1-\gamma)V_\gamma^\pi} \leq (1-\gamma)\spannorm{V_\gamma^\pi}. \label{eq:DMDP_red_fixed_policy_imp_st}
    \end{align}
    Consequently, for any policy $\pi$, we have that
    \begin{align*}
        \rho^\pi \geq \rho^\star - (1-\gamma)\left( \infnorm{V_\gamma^\pi - V_\gamma^\star} + \spannorm{V_\gamma^\star}\right).
    \end{align*}
\end{lem}
\begin{proof}
    First, for the statement with a fixed policy $\pi$, as shown in \cite{wang_near_2022} we have that
    \begin{align*}
        \rho^\pi = P_{\pi}^\infty r_\pi = (1-\gamma) \sum_{t=0}^\infty \gamma^t P_{\pi}^\infty r_\pi = (1-\gamma) P_{\pi}^\infty \sum_{t=0}^\infty \gamma^t P_{\pi}^t r_\pi = (1-\gamma) P_{\pi}^\infty V_\gamma^\pi
    \end{align*}
    which implies that, for each $s \in\S$, $\rho^\pi(s) = e_s^\top P_{\pi}^\infty \left((1-\gamma)  V_\gamma^\pi\right)$, and thus $\rho^\pi(s)$ is a convex combination of entries of $(1-\gamma)  V_\gamma^\pi$, implying that $(1-\gamma) \min_{s'} V_\gamma^\pi(s') \leq \rho^\pi(s) \leq (1-\gamma) \max_{s'} V_\gamma^\pi(s')$.
    By subtracting $(1-\gamma) V_\gamma^\pi(s)$ we obtain that
    \begin{align*}
        (1-\gamma) \left(\min_{s'} V_\gamma^\pi(s') - V_\gamma^\pi(s) \right) \leq \rho^\pi(s) - (1-\gamma) V_\gamma^\pi(s) \leq (1-\gamma) \left(\max_{s'} V_\gamma^\pi(s') - V_\gamma^\pi(s) \right)
    \end{align*}
    which can be further lower and upper bounded as
    \begin{align*}
        -(1-\gamma) \spannorm{V_\gamma^\pi}  \leq \rho^\pi(s) - (1-\gamma) V_\gamma^\pi(s) \leq (1-\gamma) \spannorm{V_\gamma^\pi}
    \end{align*}
    implying~\eqref{eq:DMDP_red_fixed_policy_imp_st}.

    Now we show the statement~\eqref{eq:sharper_DMDP_red_opt}. For any $s \in \S$, we similarly have
    \begin{align*}
        \rho^\star(s) = (1-\gamma) e_s^\top P_{\pistar}^\infty V_\gamma^{\pistar} \leq (1-\gamma) e_s^\top P_{\pistar}^\infty V_\gamma^\star \leq (1-\gamma) \max_s V_\gamma^\star(s)
    \end{align*}
    and also for any $s \in \S$
    \begin{align*}
        \rho^\star(s) \geq \rho^{\pistar_\gamma}(s) = (1-\gamma) e_s^\top P_{\pistar_\gamma}^\infty V_\gamma^\star \geq (1-\gamma) \min_s V_\gamma^\star(s).
    \end{align*}
    Statement~\eqref{eq:DMDP_red_opt_policy_imp_st} follows in an identical manner as to how~\eqref{eq:DMDP_red_fixed_policy_imp_st} followed from~\eqref{eq:sharper_DMDP_red_fixed}.

    For the final desired statement, using the previous results we have
    \begin{align*}
        \rho^\pi &\geq (1-\gamma)\min_s V_\gamma^\pi(s)\one \geq (1-\gamma)\min_s V_\gamma^\star(s)\one - (1-\gamma)\infnorm{V_\gamma^\pi - V_\gamma^\star} \\
        &\geq (1-\gamma)\max_s V_\gamma^\star(s)\one  - (1-\gamma) \spannorm{ V_\gamma^\star}  - (1-\gamma)\infnorm{V_\gamma^\pi - V_\gamma^\star} \\
        &\geq \rho^\star  - (1-\gamma) \spannorm{ V_\gamma^\star}  - (1-\gamma)\infnorm{V_\gamma^\pi - V_\gamma^\star}.
    \end{align*}
\end{proof}

Here we repeat the definitions of some quantities which are defined in Algorithm \ref{alg:eps_based_alg}, and additionally we extend their definitions to hold for all integers $i \geq 1$ (rather than just the iterations which appear in the course of the algorithm before its termination). For all $i \geq 1$, we define $\Phat^{(i)}$ to be constructed in the manner shown in Algorithm \ref{alg:eps_based_alg} via taking $n_i = 2^i$ samples from all state-action pairs. We also let $\hzns_i := \{\gamma: \text{there exists an integer $k $ such that } \sqrt{n_i} \leq \frac{1}{1-\gamma} = 2^k \leq n_i \}$.
We can then define, for all $\gamma \in \hzns_i$, the policy $\pit_{\gamma,i}$ and value function $\Vt_{\gamma,i}$ as the outputs from $\SolveDMDP(\Phat^{(i)}, r, \gamma, \frac{1}{n_i})$. With such quantities we can then define $\Ub_i(\gamma), \Lb_i(\gamma)$ following the definitions given in \ref{alg:eps_based_alg}, and finally we let $\gammahat_i := \argmin_{\gamma \in \hzns_i} \Ub_i(\gamma) - \Lb_i(\gamma)$. We also use $\Vhat_{\gamma, i}^\pi$ to denote the value function of policy $\pi$ in the DMDP $(\Phat^{(i)}, r, \gamma)$ and $\Vhat_{\gamma, i}^\star$ to denote the optimal value function of this DMDP.

\begin{lem}
\label{lem:concentration_event_eps_based}
    Define the function $\alpha(\tilde{\delta}, \tilde{n}) = 24 \sqrt{16 \log \left( \frac{24 SA \tilde{n}^5}{\tilde{\delta}}\right)} \log_2\left( \log_2 (\tilde{n} + 4) \right)$. Fix $\delta > 0$. Then with probability at least $1 - \delta$, we have for all integers $i \geq 1$ and all $\gamma \in \hzns_i$ that
\begin{align}
    \infnorm{V_{\gamma}^{\pistar_{\gamma}} - \Vhat_{\gamma, i}^{\pistar_{\gamma}}} &\leq \frac{\alpha(\delta, n_i)}{1-\gamma} \sqrt{\frac{ \spannorm{V_{\gamma}^{\pistar_{\gamma}}} + 1 }{n_i}} \label{eq:true_opt_pol_eval_bd_2}
\end{align}
    and also the subroutine on line \ref{alg:solver_step_eps_based} outputs a policy $\pit_{\gamma, i}$ such that
\begin{align}
        \Vhat_{\gamma,i}^{\pit_{\gamma,i}} &\geq \Vhat^\star_{\gamma,i} - \frac{1}{n_i} \one \label{eq:pit_near_opt_2}\\
        \infnorm{\Vt_{\gamma,i} - \Vhat^\star_{\gamma,i}} &\leq \frac{1}{n_i} \label{eq:vt_error_2}\\
        \infnorm{\Vhat_{\gamma, i}^{\pit_{\gamma,i}} - V_\gamma^{\pit_{\gamma,i}}} & \leq \frac{\alpha(\delta, n_i)}{1-\gamma} \sqrt{\frac{ \spannorm{\Vhat_{\gamma, i}^{\pit_{\gamma,i}}} + 1 }{n_i}} \label{eq:emp_opt_pol_eval_bd_2}.
\end{align}
\end{lem}
\begin{proof}
    We fix a pair $i$ and $\gamma \in \hzns_i$ and establish the desired bounds, and then we will take a union bound. From \cite[Proof of Theorem 9]{zurek_plug-approach_2024}, in particular \cite[Equations (31) and (32)]{zurek_plug-approach_2024}, we have with probability at least $1 - 2\delta$ that
    \begin{align*}
       \infnorm{V_\gamma^{\pistar_\gamma} - \Vhat_\gamma^{\pistar_\gamma}} & \leq \frac{24\log_2 \log_2 \left( \frac{1}{1-\gamma} + 4\right)}{1-\gamma} \sqrt{16 \frac{\spannorm{V^{\star}_\gamma} + 1}{n_i} \log\left( \frac{12 SAn_i}{(1-\gamma)^2 \delta}\right)}
    \end{align*}
    and
    \begin{align*}
       \infnorm{V_\gamma^{\pit_\gamma} - \Vhat_\gamma^{\pit_\gamma}} & \leq \frac{24\log_2 \log_2 \left( \frac{1}{1-\gamma} + 4\right)}{1-\gamma} \sqrt{16 \frac{\spannorm{\Vhat_\gamma^{\pit_\gamma}} + 1}{n_i} \log\left( \frac{12 SAn_i}{(1-\gamma)^2 \delta}\right)}.
    \end{align*}
By definition of $\hzns_i$ we have that $\frac{1}{1-\gamma} \leq n_i$, which we use to simplify the bounds slightly to
\begin{align}
       \infnorm{V_\gamma^{\pistar_\gamma} - \Vhat_\gamma^{\pistar_\gamma}} & \leq \frac{24\log_2 \log_2 \left( n_i + 4\right)}{1-\gamma} \sqrt{16 \frac{\spannorm{V^{\star}_\gamma} + 1}{n_i} \log\left( \frac{12 SAn_i^3}{ \delta}\right)} \label{eq:conc_bd1_s1}
    \end{align}
    and
    \begin{align}
       \infnorm{V_\gamma^{\pit_\gamma} - \Vhat_\gamma^{\pit_\gamma}} & \leq \frac{24\log_2 \log_2 \left( n_i + 4\right)}{1-\gamma} \sqrt{16 \frac{\spannorm{\Vhat_\gamma^{\pit_\gamma}} + 1}{n_i} \log\left( \frac{12 SAn_i^3}{ \delta}\right)}. \label{eq:conc_bd2_s1}
    \end{align}
    Now keeping $i$ fixed and taking a union bound over all $\gamma \in \hzns_i$, we have that inequalities~\eqref{eq:conc_bd1_s1} and~\eqref{eq:conc_bd2_s1} hold for all $\gamma \in \hzns_i$ with probability at least $1 - 2\delta |\hzns_i|$. Also
    \begin{align*}
        |\hzns_i| \leq 1 + \log_2 \frac{n_i}{\sqrt{n_i}} \leq 1 + \frac{1}{2} \log_2 n_i \leq 1 + \frac{1}{2} n_i \leq n_i
    \end{align*}
    using the facts that $\log_2 x \leq x$ and that $n_i \geq 2$ (so $1 \leq \frac{1}{2}n_i$). Now adjusting the failure probability, we have (for any $\delta_i>0$) that with probability at least $1 - \delta_i$, for all $\gamma \in \hzns_i$, both
    \begin{align}
       \infnorm{V_\gamma^{\pistar_\gamma} - \Vhat_\gamma^{\pistar_\gamma}} & \leq \frac{24\log_2 \log_2 \left( n_i + 4\right)}{1-\gamma} \sqrt{16 \frac{\spannorm{V^{\star}_\gamma} + 1}{n_i} \log\left( \frac{24 SAn_i^4}{ \delta_i}\right)} \label{eq:conc_bd1_s2}
    \end{align}
    and
    \begin{align}
       \infnorm{V_\gamma^{\pit_\gamma} - \Vhat_\gamma^{\pit_\gamma}} & \leq \frac{24\log_2 \log_2 \left( n_i + 4\right)}{1-\gamma} \sqrt{16 \frac{\spannorm{\Vhat_\gamma^{\pit_\gamma}} + 1}{n_i} \log\left( \frac{24 SAn_i^4}{ \delta_i}\right)} \label{eq:conc_bd2_s2}
    \end{align}
    are true. Now we union bound over all natural numbers $i \geq 1$. We set $\delta_i = \frac{\delta}{2^i} = \frac{\delta}{n_i}$ and thus obtain that with probability at least $1 - \sum_{i \geq 1}\delta_i = 1-\delta$, for all $i \geq 1$ and all $\gamma \in \hzns_i$ we have both~\eqref{eq:conc_bd1_s2} and~\eqref{eq:conc_bd2_s2}. By our choice of $\delta_i$'s these can be written as
    \begin{align*}
       \infnorm{V_\gamma^{\pistar_\gamma} - \Vhat_\gamma^{\pistar_\gamma}} & \leq \frac{24\log_2 \log_2 \left( n_i + 4\right)}{1-\gamma} \sqrt{16 \frac{\spannorm{V^{\star}_\gamma} + 1}{n_i} \log\left( \frac{24 SAn_i^5}{ \delta}\right)} \\
       &= \frac{\alpha(\delta, n_i)}{1-\gamma}\sqrt{\frac{\spannorm{V^{\star}_\gamma} + 1}{n_i}}
    \end{align*}
    and
    \begin{align*}
       \infnorm{V_\gamma^{\pit_\gamma} - \Vhat_\gamma^{\pit_\gamma}} & \leq \frac{24\log_2 \log_2 \left( n_i + 4\right)}{1-\gamma} \sqrt{16 \frac{\spannorm{\Vhat_\gamma^{\pit_\gamma}} + 1}{n_i} \log\left( \frac{24 SAn_i^5}{ \delta}\right)}  \\
       &= \frac{\alpha(\delta, n_i)}{1-\gamma}\sqrt{\frac{\spannorm{\Vhat_\gamma^{\pit_\gamma}} + 1}{n_i}}
    \end{align*}
    as desired.
\end{proof}

\begin{lem}
\label{lem:recursive_err_bds_abstract}
    Suppose that there exists some $m , \beta > 1$, $\gamma < 1$, policy $\pi$, MDPs $(P,r)$, $(\Po, r)$, such that 
    \begin{align}
        \infnorm{V_{\gamma}^{\pistar_{\gamma}} - \Vo_{\gamma}^{\pistar_{\gamma}}} &\leq \frac{\beta}{1-\gamma} \sqrt{\frac{ \spannorm{V_{\gamma}^{\pistar_{\gamma}}} + 1 }{m}} \label{eq:true_opt_pol_eval_bd_abs} \\
        \Vo_{\gamma}^{\pi} &\geq \Vo^\star_{\gamma} - \frac{1}{m} \one \label{eq:pi_near_opt_abs}\\
        \infnorm{\Vo_{\gamma}^{\pi} - V_\gamma^{\pi}} & \leq \frac{\beta}{1-\gamma} \sqrt{\frac{ \spannorm{\Vo_{\gamma}^{\pi}} + 1 }{m}} \label{eq:emp_opt_pol_eval_bd_abs}
\end{align}
    where for any $\pi'$ we use $\Vo_\gamma^{\pi'}$ to denote the value of policy $\pi'$ in the DMDP $(\Po, r, \gamma)$, where $\Vo^\star_{\gamma}$ denotes the optimal value function in the DMDP $(\Po, r, \gamma)$, and where $\pistar_{\gamma}$ is the optimal policy for the DMDP $(P, r, \gamma)$.
    Then
    \begin{align}
        \infnorm{\Vo_{\gamma}^{\star} - V_{\gamma}^{\star}} & \leq \frac{2 \beta^2}{(1-\gamma)^2 m} + \frac{4 \beta}{1-\gamma} \sqrt{\frac{\spannorm{V^\star_{\gamma}} +1+ \frac{1}{m}}{m}} + \frac{4}{m} \label{eq:rec_bd_1}
    \end{align}
    and
    \begin{align}
        \infnorm{\Vo_{\gamma}^{\star} - V_{\gamma}^{\star}} & \leq \frac{2 \beta^2}{(1-\gamma)^2 m} + \frac{4 \beta}{1-\gamma} \sqrt{\frac{\spannorm{\Vo_{\gamma}^{\star}} +1+ \frac{1}{m}}{m}} + \frac{4}{m}. \label{eq:rec_bd_2}
    \end{align}
\end{lem}
\begin{proof}
    Using the triangle inequality several times and~\eqref{eq:pi_near_opt_abs}, we have
    \begin{align*}
        V_\gamma^\star & \geq 
        V_\gamma^{\pi} \\
        & \geq \Vo_{\gamma}^{\pi} - \infnorm{\Vo_{\gamma}^{\pi} - V_\gamma^{\pi}}\one \\
        & \geq \Vo_{\gamma}^{\star} - \frac{1}{m} \one - \infnorm{\Vo_{\gamma}^{\pi} - V_\gamma^{\pi}}\one \\
        & \geq \Vo_{\gamma}^{\pistar_\gamma} - \frac{1}{m} \one - \infnorm{\Vo_{\gamma}^{\pi} - V_\gamma^{\pi}}\one \\
        & \geq V_{\gamma}^{\pistar_\gamma} - \infnorm{V_{\gamma}^{\pistar_{\gamma}} - \Vo_{\gamma}^{\pistar_{\gamma}}} \one- \frac{1}{m} \one - \infnorm{\Vo_{\gamma}^{\pi} - V_\gamma^{\pi}}\one.
    \end{align*}
    Since by definition $V_{\gamma}^{\pistar_\gamma} = V_\gamma^\star$, subtracting this term from all expressions we obtain
    \begin{align*}
        \zero & \geq \Vo_{\gamma}^{\star} - V_\gamma^\star - \frac{1}{m} \one - \infnorm{\Vo_{\gamma}^{\pi} - V_\gamma^{\pi}}\one \geq - \infnorm{V_{\gamma}^{\pistar_{\gamma}} - \Vo_{\gamma}^{\pistar_{\gamma}}} \one- \frac{1}{m} \one - \infnorm{\Vo_{\gamma}^{\pi} - V_\gamma^{\pi}}\one
    \end{align*}
    which after rearranging implies that
\begin{align*}
        \frac{1}{m} \one + \infnorm{\Vo_{\gamma}^{\pi} - V_\gamma^{\pi}}\one & \geq \Vo_{\gamma}^{\star} - V_\gamma^\star  \geq - \infnorm{V_{\gamma}^{\pistar_{\gamma}} - \Vo_{\gamma}^{\pistar_{\gamma}}} \one
    \end{align*}
    which further implies
\begin{align}
    \infnorm{\Vo_{\gamma}^{\star} - V_{\gamma}^{\star}} & \leq \infnorm{\Vo_{\gamma}^{\pi}- V^{\pi}_{\gamma}} + \infnorm{V_{\gamma}^{\pistar_{\gamma}} - \Vo_{\gamma}^{\pistar_{\gamma}}} + \frac{1}{m}. \label{eq:opt_val_diff_1}
\end{align}

Now we will focus on establishing the first inequality~\eqref{eq:rec_bd_1} in the lemma statement.
For the first term on the RHS of~\eqref{eq:opt_val_diff_1}, we can use condition~\eqref{eq:emp_opt_pol_eval_bd_abs} and then triangle inequality to bound
\begin{align}
    \infnorm{\Vo_{\gamma}^{\pi}- V^{\pi}_{\gamma}} & \leq \frac{\beta}{1-\gamma} \sqrt{\frac{ \spannorm{\Vo^\pi_{\gamma}} +1}{m}} \nonumber\\
    &\leq \frac{\beta}{1-\gamma} \sqrt{\frac{ \spannorm{V^\star_{\gamma}} + \spannorm{\Vo^\pi_{\gamma} - V^\star_{\gamma}} +1}{m}} \nonumber \\
    & \leq \frac{\beta}{1-\gamma} \sqrt{\frac{ \spannorm{V^\star_{\gamma}} + 2\infnorm{\Vo^\star_{\gamma} - V^\star_{\gamma}}+\frac{1}{m} +1}{m}} \label{eq:opt_val_diff_2}
\end{align}
using that
\begin{align*}
    \spannorm{\Vo^\pi_{\gamma} - V^\star_{\gamma}} \leq \spannorm{\Vo^\pi_{\gamma} - \Vo^\star_{\gamma}} + \spannorm{\Vo^\star_{\gamma} - V^\star_{\gamma}} \leq \frac{1}{m} + 2\infnorm{\Vo^\star_{\gamma} - V^\star_{\gamma}}
\end{align*}
since $\spannorm{\cdot} \leq 2 \infnorm{\cdot}$ and $\zero \leq  \Vo^\star_{\gamma} - \Vo^\pi_{\gamma} \leq \frac{1}{m}\one$.
For the second term on the RHS of~\eqref{eq:opt_val_diff_1}, we can use the condition~\eqref{eq:true_opt_pol_eval_bd_abs}. Thus combining~\eqref{eq:opt_val_diff_1},~\eqref{eq:opt_val_diff_2}, and~\eqref{eq:true_opt_pol_eval_bd_abs}, we obtain
\begin{align}
    \infnorm{\Vo_{\gamma}^{\star} - V_{\gamma}^{\star}} & \leq \frac{\beta}{1-\gamma} \sqrt{\frac{ \spannorm{V^\star_{\gamma}} + 2\infnorm{\Vo^\star_{\gamma} - V^\star_{\gamma}} +1 + \frac{1}{m}}{m}} + \frac{\beta}{1-\gamma} \sqrt{\frac{ \spannorm{V_{\gamma}^{\pistar_{\gamma}}} + 1 }{m}} +\frac{1}{m}\nonumber \\
    & \leq \frac{2\beta}{1-\gamma} \sqrt{\frac{\spannorm{V_{\gamma}^{\star}} + 1 + \frac{1}{m}}{m}} + \frac{\beta}{1-\gamma} \sqrt{\frac{2 \infnorm{\Vo^\star_{\gamma} - V^\star_{\gamma}} }{m}} + \frac{1}{m}\label{eq:opt_val_diff_3}
\end{align}
where we used that $\sqrt{a+b} \leq \sqrt{a} + \sqrt{b}$ (and $ V_{\gamma}^{\star} = V_{\gamma}^{\pistar_{\gamma}}$).
Rearranging~\eqref{eq:opt_val_diff_3} as
\begin{align*}
    \infnorm{\Vo_{\gamma}^{\star} - V_{\gamma}^{\star}} - \frac{\sqrt{2}\beta}{(1-\gamma)\sqrt{m}} \sqrt{\infnorm{\Vo_{\gamma}^{\star} - V_{\gamma}^{\star}}} - \frac{2\beta}{1-\gamma} \sqrt{\frac{\spannorm{V_{\gamma}^{\star}} + 1+ \frac{1}{m}}{m}} + \frac{1}{m}\leq 0
\end{align*}
and then using the quadratic formula to find the largest possible root of this quadratic polynomial in $\sqrt{\infnorm{\Vo_{\gamma}^{\star} - V_{\gamma}^{\star}}}$, we obtain that
\begin{align}
    \sqrt{\infnorm{\Vo_{\gamma}^{\star} - V_{\gamma}^{\star}}} & \leq \frac{1}{2} \frac{\sqrt{2}\beta}{(1-\gamma)\sqrt{m}} + \frac{1}{2} \sqrt{\left( \frac{\sqrt{2}\beta}{(1-\gamma)\sqrt{m}}\right)^2 + 4 \frac{2\beta}{1-\gamma} \sqrt{\frac{\spannorm{V_{\gamma}^{\star}} + 1+ \frac{1}{m}}{m}} + \frac{4}{m}}.
\end{align}
Now squaring both sides and using that $(a+b)^2 \leq 2a^2 + 2b^2$ which implies $(\frac{a}{2}+\frac{b}{2})^2 \leq \frac{a^2}{2} + \frac{b^2}{2}$, we obtain
\begin{align}
    \infnorm{\Vo_{\gamma}^{\star} - V_{\gamma}^{\star}} & \leq \frac{1}{2} \frac{2 \beta^2}{(1-\gamma)^2 m} + \frac{1}{2} \left( \left( \frac{\sqrt{2}\beta}{(1-\gamma)\sqrt{m}}\right)^2 + 4 \frac{2\beta}{1-\gamma} \sqrt{\frac{\spannorm{V_{\gamma}^{\star}} + 1+ \frac{1}{m}}{m}} + \frac{4}{m}\right) \nonumber\\
    &= \frac{2 \beta^2}{(1-\gamma)^2 m} + \frac{4 \beta}{1-\gamma} \sqrt{\frac{\spannorm{V^\star_{\gamma}} +1+ \frac{1}{m}}{m}} + \frac{4}{m}
\end{align}
as desired. 

We next focus on establishing the second inequality~\eqref{eq:rec_bd_2} from the lemma statement.
We now will bound the first term on the RHS of~\eqref{eq:opt_val_diff_1} in terms of $\spannorm{\Vo^\star_\gamma}$ with an analogous argument but instead using only that $\spannorm{\Vo^\pi_\gamma} \leq \spannorm{\Vo^\star_{\gamma}} + \spannorm{\Vo^\pi_{\gamma} - \Vo^\star_{\gamma}} \leq \spannorm{\Vo^\star_{\gamma}} + \frac{1}{m}$, to obtain
\begin{align}
    \infnorm{\Vo_{\gamma}^{\pi}- V^{\pi}_{\gamma}}
    & \leq \frac{\beta}{1-\gamma} \sqrt{\frac{ \spannorm{\Vo^\star_{\gamma}} +\frac{1}{m} +1}{m}} .\label{eq:opt_val_diff_4}
\end{align}
We bound the second term on the RHS of~\eqref{eq:opt_val_diff_1} in terms of $\spannorm{\Vo^\star_\gamma}$ by using~\eqref{eq:true_opt_pol_eval_bd_abs} and then that $\spannorm{V_{\gamma}^{\pistar_{\gamma}}} \leq \spannorm{\Vo_{\gamma}^{\star}} + \spannorm{\Vo_{\gamma}^{\star} - V_{\gamma}^{\star}} \leq \spannorm{\Vo_{\gamma}^{\star}} + 2\infnorm{\Vo_{\gamma}^{\star} - V_{\gamma}^{\star}}$ to obtain
\begin{align}
    \infnorm{V_{\gamma}^{\pistar_{\gamma}} - \Vo_{\gamma}^{\pistar_{\gamma}}} &\leq \frac{\beta}{1-\gamma} \sqrt{\frac{ \spannorm{V_{\gamma}^{\pistar_{\gamma}}} + 1 }{m}} \leq \frac{\beta}{1-\gamma} \sqrt{\frac{ \spannorm{\Vo_{\gamma}^{\star}} + 2\infnorm{\Vo_{\gamma}^{\star} - V_{\gamma}^{\star}} + 1 }{m}}. \label{eq:opt_val_diff_5}
\end{align}
Combining~\eqref{eq:opt_val_diff_1},~\eqref{eq:opt_val_diff_4}, and~\eqref{eq:opt_val_diff_5}, we obtain
\begin{align}
    \infnorm{\Vo_{\gamma}^{\star} - V_{\gamma}^{\star}} & \leq \frac{\beta}{1-\gamma} \sqrt{\frac{ \spannorm{\Vo^\star_{\gamma}} +\frac{1}{m} +1}{m}} + \frac{\beta}{1-\gamma} \sqrt{\frac{ \spannorm{\Vo_{\gamma}^{\star}} + 2\infnorm{\Vo_{\gamma}^{\star} - V_{\gamma}^{\star}} + 1 }{m}} + \frac{1}{m} \nonumber\\
    & \leq \frac{2\beta}{1-\gamma} \sqrt{\frac{ \spannorm{\Vo^\star_{\gamma}} +\frac{1}{m} +1}{m}} + \frac{\beta}{1-\gamma} \sqrt{\frac{  2\infnorm{\Vo_{\gamma}^{\star} - V_{\gamma}^{\star}} }{m}} + \frac{1}{m} \label{eq:opt_val_diff_6}
\end{align}
using $\sqrt{a+b} \leq \sqrt{a} + \sqrt{b}$ for the second inequality. Since~\eqref{eq:opt_val_diff_6} is identical to~\eqref{eq:opt_val_diff_3} except for the presence of $\spannorm{\Vo^\star_{\gamma}}$ instead of $\spannorm{V^\star_{\gamma}}$, we can take completely analogous steps to obtain~\eqref{eq:rec_bd_2}.
\end{proof}

The following lemma is not necessary to show the validity of the lower bounds, but it is necessary for the upper bounds.
\begin{lem}
\label{lem:eps_based_recursive_err_bds}
    Under the event in Lemma \ref{lem:concentration_event_eps_based}, for all integers $i \geq 1$ and all $\gamma \in \hzns_i$, we have
    \begin{align}
        \infnorm{\Vhat_{\gamma, i}^{\star} - V_{\gamma}^{\star}} & \leq \frac{2 \alpha(\delta, n_i)^2}{(1-\gamma)^2 n_i} + \frac{4 \alpha(\delta, n_i)}{1-\gamma} \sqrt{\frac{\spannorm{V_{\gamma}^{\star}} +1+ \frac{1}{n_i}}{n_i}} + \frac{4}{n_i} \label{eq:eps_based_recursive_err_bd_1}
    \end{align}
    and
    \begin{align}
        \infnorm{\Vhat_{\gamma, i}^{\star} - V_{\gamma}^{\star}} & \leq \frac{2 \alpha(\delta, n_i)^2}{(1-\gamma)^2 n_i} + \frac{4 \alpha(\delta, n_i)}{1-\gamma} \sqrt{\frac{\spannorm{\Vhat_{\gamma, i}^{\star}} +1+ \frac{1}{n_i}}{n_i}} + \frac{4}{n_i} \label{eq:eps_based_recursive_err_bd_2}
    \end{align}
\end{lem}
\begin{proof}
    This follows immediately from Lemma \ref{lem:recursive_err_bds_abstract} since the event described in Lemma \ref{lem:concentration_event_eps_based} exactly meets the conditions of Lemma \ref{lem:recursive_err_bds_abstract} (by setting $m = n_i$ and $\beta = \alpha(\delta, n_i)$, for all $i$ and $\gamma \in \hzns_i$).
\end{proof}

Now we show that under the event in Lemma \ref{lem:concentration_event_eps_based}, all lower and upper bounds constructed within Algorithm \ref{alg:eps_based_alg} are valid lower/upper bounds of $\rho^\star$.
\begin{lem}
    \label{lem:CI_validity_eps_based}
    Under the event in Lemma \ref{lem:concentration_event_eps_based}, for all integers $i \geq 1$ and all $\gamma \in \hzns_i$, we have
    \begin{align}
        \Lb_i(\gamma) \one \leq \rho^{\pit_{\gamma, i}} \leq \rho^\star \leq \Ub_i(\gamma) \one.
    \end{align}
\end{lem}
\begin{proof}
    Fix an arbitrary integer $i \geq 1$ and $\gamma \in \hzns_i$. By~\eqref{eq:sharper_DMDP_red_fixed}, optimality of $\rho^\star$, and~\eqref{eq:sharper_DMDP_red_opt} we have
    \begin{align}
        (1-\gamma) \left(\min_s V_\gamma^{\pit_{\gamma, i}}(s) \right) \one &\leq \rho^{\pit_{\gamma, i}} \nonumber\\
        &\leq \rho^\star \nonumber\\
        & \leq (1-\gamma) \left(\max_s V_\gamma^{\star}(s) \right) \one. \label{eq:eps_based_CI_1}
    \end{align}
    We start by lower-bounding the LHS of~\eqref{eq:eps_based_CI_1}. Using~\eqref{eq:pit_near_opt_2} and~\eqref{eq:emp_opt_pol_eval_bd_2},
    \begin{align}
        (1-\gamma) \left(\min_s V_\gamma^{\pit_{\gamma, i}}(s) \right) \one & \geq (1-\gamma) \left(\min_s \Vhat_{\gamma, i}^{\pit_{\gamma, i}} (s) \right) \one - (1-\gamma) \infnorm{\Vhat_{\gamma, i}^{\pit_{\gamma, i}} - V_\gamma^{\pit_{\gamma,i}}} \one \nonumber \\
        & \geq (1-\gamma) \left(\min_s \Vhat_{\gamma, i}^\star (s) \right) \one - \frac{1-\gamma}{n_i}\one - (1-\gamma) \infnorm{\Vhat_{\gamma, i}^{\pit_{\gamma, i}} - V_\gamma^{\pit_{\gamma,i}}}  \one \nonumber \\
        & \geq (1-\gamma) \left(\min_s \Vhat_{\gamma, i}^\star (s) \right) \one - \frac{1-\gamma}{n_i}\one - \alpha(\delta, n_i) \sqrt{\frac{ \spannorm{\Vhat_{\gamma, i}^{\pit_{\gamma, i}}} + 1 }{n_i}} \one. \label{eq:eps_based_CI_2}
    \end{align}
    Now we replace the quantities in~\eqref{eq:eps_based_CI_2} with observable quantities in terms of $\Vt_{\gamma,i}$. Using the requirement~\eqref{eq:pit_near_opt_2} to relate $\Vhat_{\gamma,i}^{\pit_{\gamma,i}} $ and $ \Vhat^\star_{\gamma,i}$, and then~\eqref{eq:vt_error_2} which bounds $\infnorm{\Vt_{\gamma,i} - \Vhat^\star_{\gamma,i}}$, we have
    \begin{align}
         \spannorm{\Vhat_{\gamma, i}^{\pit_{\gamma, i}}} &\leq \spannorm{\Vhat_{\gamma, i}^{\star}} + \spannorm{\Vhat_{\gamma, i}^{\star} - \Vhat_{\gamma, i}^{\pit_{\gamma, i}}} \nonumber\\
         &\leq \spannorm{\Vt_{\gamma,i}} +\spannorm{\Vt_{\gamma,i} - \Vhat_{\gamma, i}^{\star}} + \spannorm{\Vhat_{\gamma, i}^{\star} - \Vhat_{\gamma, i}^{\pit_{\gamma, i}}} \nonumber\\
         & \leq \spannorm{\Vt_{\gamma,i}} +2\infnorm{\Vt_{\gamma,i} - \Vhat_{\gamma, i}^{\star}} + \frac{1}{n_i} \nonumber \\
        & \leq \spannorm{\Vt_{\gamma,i}} +\frac{2}{n_i} + \frac{1}{n_i} \label{eq:eps_based_CI_sp_bd}
    \end{align}
    (where $\spannorm{\Vhat_{\gamma, i}^{\star} - \Vhat_{\gamma, i}^{\pit_{\gamma, i}}} \leq \frac{1}{n_i}$ since $\zero \leq \Vhat_{\gamma, i}^{\star} - \Vhat_{\gamma, i}^{\pit_{\gamma, i}} \leq \frac{1}{n_i}\one$).
    Using this bound~\eqref{eq:eps_based_CI_sp_bd}, as well as~\eqref{eq:vt_error_2} again, we
    can further bound~\eqref{eq:eps_based_CI_2} as
    \begin{align}
        &(1-\gamma) \left(\min_s \Vhat_{\gamma, i}^\star (s) \right) \one - \frac{1-\gamma}{n_i}\one - \alpha(\delta, n_i) \sqrt{\frac{ \spannorm{\Vhat_{\gamma, i}^{\pit_{\gamma, i}}} + 1 }{n_i}} \one \nonumber\\
        & \geq (1-\gamma) \left(\min_s \Vt_{\gamma, i} (s) \right) \one - (1-\gamma)\infnorm{\Vt_{\gamma, i} - \Vhat_{\gamma, i}^\star}\one - \frac{1-\gamma}{n_i}\one - \alpha(\delta, n_i) \sqrt{\frac{ \spannorm{\Vt_{\gamma,i}} +\frac{3}{n_i} + 1 }{n_i}} \one  \nonumber \\
        & \geq (1-\gamma) \left(\min_s \Vt_{\gamma, i} (s) \right) \one  - 2\frac{1-\gamma}{n_i}\one - \alpha(\delta, n_i) \sqrt{\frac{ \spannorm{\Vt_{\gamma,i}} +\frac{3}{n_i} + 1 }{n_i}} \one  \nonumber \\
        &= \Lb_i(\gamma)\one \label{eq:eps_based_CI_3}
    \end{align}
    Now we upper-bound the RHS of~\eqref{eq:eps_based_CI_1}. The inequality~\eqref{eq:eps_based_recursive_err_bd_2} plays a key role. Using this bound, as well as~\eqref{eq:vt_error_2} to bound $\infnorm{\Vt_{\gamma,i} - \Vhat_{\gamma,i}^{\star}}$, and the fact that $\spannorm{\Vhat^\star_{\gamma, i}} \leq \spannorm{\Vt_{\gamma,i}} + \spannorm{\Vhat^\star_{\gamma, i} - \Vt_{\gamma,i}} \leq \spannorm{\Vt_{\gamma,i}} + \frac{2}{n_i}$ (since $\spannorm{\cdot} \leq 2 \infnorm{\cdot}$ and using~\eqref{eq:vt_error_2}), we obtain
    \begin{align}
        &(1-\gamma) \left(\max_s V_\gamma^{\star}(s) \right) \one \nonumber \\
        & \leq (1-\gamma) \left(\max_s \Vhat_{\gamma,i}^{\star}(s) \right) \one + (1-\gamma) \infnorm{\Vhat_{\gamma,i}^{\star} -  V_\gamma^{\star}} \one \nonumber \\
        & \leq (1-\gamma) \left(\max_s \Vt_{\gamma,i}(s) \right) \one + (1-\gamma) \infnorm{\Vt_{\gamma,i} - \Vhat_{\gamma,i}^{\star}} \one+ (1-\gamma) \infnorm{\Vhat_{\gamma,i}^{\star} -  V_\gamma^{\star}} \one \nonumber \\
        & \leq (1-\gamma) \left(\max_s \Vt_{\gamma,i}(s) \right) \one + \frac{1-\gamma}{n_i}\one + \frac{2 \alpha(\delta, n_i)^2}{(1-\gamma) n_i} \one + 4 \alpha(\delta, n_i) \sqrt{\frac{\spannorm{\Vhat_{\gamma, i}^{\star}} +1+ \frac{1}{n_i}}{n_i}}\one + (1-\gamma)\frac{4}{n_i}\one\nonumber \\
        & \leq (1-\gamma) \left(\max_s \Vt_{\gamma,i}(s) \right) \one + 5\frac{1-\gamma}{n_i}\one+ \frac{2 \alpha(\delta, n_i)^2}{(1-\gamma) n_i}\one + 4 \alpha(\delta, n_i) \sqrt{\frac{\spannorm{\Vt_{\gamma, i}} +1+ \frac{3}{n_i}}{n_i}} \one\nonumber \\
        &= \Ub_i(\gamma)\one. \label{eq:eps_based_CI_4}
    \end{align}
    Combining~\eqref{eq:eps_based_CI_1},~\eqref{eq:eps_based_CI_3}, and~\eqref{eq:eps_based_CI_4}, and unfixing $i$ and $\gamma$, we obtain the desired conclusion.
\end{proof}

Lemma \ref{lem:CI_validity_eps_based} implies that whenever the algorithm terminates, the resulting policy will be $\varepsilon$-optimal and the resulting confidence interval will be valid and of size $\leq \varepsilon$.
Next we show that the algorithm will terminate by a certain iteration, by showing that on this iteration the confidence interval corresponding to some discount factor will be small.

First we define, for all integers $i \geq 1$, the discount factor $\gstar_i$ as
\begin{align}
    \gstar_i = \inf \left\{ \gamma \in \hzns_i : \frac{1}{1-\gamma}\geq \sqrt{n_i \left(\spannorm{h^\star} + 1 \right)}\right\}. \label{eq:gstar_i_defn}
\end{align}
(If $\spannorm{h^\star}$ is large relative to $n_i$ then, since the largest value of $\frac{1}{1-\gamma}$ for $\gamma \in \hzns_i$ is $n_i$, the above set might be empty, in which case by usual convention we would have $\gstar_i = \inf \emptyset = \infty$.)
\begin{lem}
    \label{lem:gstar_i_bounds}
    For all integers $i \geq 1$, if
    \begin{align}
        \spannorm{h^\star} + 1 \leq  n_i \label{eq:gstar_i_finite_cond}
    \end{align}
    then $\gstar_i$ is finite, and furthermore
    \begin{align}
        \sqrt{n_i \left(\spannorm{h^\star} + 1 \right)} \leq \frac{1}{1-\gstar_i} \leq 2\sqrt{n_i \left(\spannorm{h^\star} + 1 \right)}. \label{eq:gstar_i_bounds}
    \end{align}
\end{lem}
\begin{proof}
    First, by condition~\eqref{eq:gstar_i_finite_cond} we have that $\sqrt{n_i \left(\spannorm{h^\star} + 1 \right)} \leq n_i$, and since the largest element of $\hzns_i$ is $n_i$, the set in the definition~\eqref{eq:gstar_i_defn} of $\gstar_i$ will be nonempty and thus $\gstar_i$ will be finite.

    Furthermore since $\spannorm{h^\star} + 1 \geq 1$, we have $ \sqrt{n_i \left(\spannorm{h^\star} + 1 \right)} \geq \sqrt{ n_i}$. $\sqrt{n_i}$ may not be a member of $\hzns_i$ but the smallest power of $2$ which is $\geq \sqrt{n_i}$ will be a member of $\hzns_i$. Therefore $\min \hzns_i \leq \frac{1}{1-\gstar_i} \leq \max \hzns_i$, so by the construction of $\hzns_i$ in line~\eqref{alg:hzns_def_dyadic} of Algorithm~\ref{alg:eps_based_alg}, condition~\eqref{eq:gstar_i_bounds} must hold as $\hzns_i$ contains all powers of $2$ within $[\min \hzns_i,\max \hzns_i]$.
\end{proof}

\begin{lem}
    \label{lem:small_confidence_interval}
    Under the event in Lemma \ref{lem:concentration_event_eps_based}, for all integers $i \geq 1$ such that $\spannorm{h^\star} + 1 \leq  n_i$, it holds that
    \begin{align*}
        \Ub_{i}(\gstar_i) - \Lb_{i}(\gstar_i) \leq 90 \alpha(\delta, n_i)^2 \sqrt{\frac{\spannorm{h^\star}+1}{n_i}} .
    \end{align*}
    In particular there exist some absolute constants $C_1, C_2$ such that letting 
    \[
    \BD = \left \lceil \log_2  \left(C_1\frac{\spannorm{h^\star}+1}{\varepsilon^2}  \log^3\left(\frac{C_2 SA (\spannorm{h^\star}+1)}{\delta \varepsilon} \right) \right) \right \rceil
    \]
    we have
    \begin{align*}
        \Ub_{\BD}(\gstar_\BD) - \Lb_{\BD}(\gstar_\BD) \leq \varepsilon.
    \end{align*}
\end{lem}
\begin{proof}
    For any integer $i\geq 1$ and any $\gamma \in \hzns_i$ we have
    \begin{align}
        \Ub_{i}(\gamma) - \Lb_{i}(\gamma) &= (1-\gamma) \left(\max_s \Vt_{\gamma,i}(s) - \min_s \Vt_{\gamma, i} (s) \right) + 7 \frac{1-\gamma}{n_i} + \frac{2 \alpha(\delta, n_i)^2}{(1-\gamma) n_i} \nonumber\\
        &\qquad + 5 \alpha(\delta, n_i) \sqrt{\frac{\spannorm{\Vt_{\gamma, i}} +1+ \frac{3}{n_i}}{n_i}} \nonumber\\
        &= (1-\gamma) \spannorm{\Vt_{\gamma, i}} + 7 \frac{1-\gamma}{n_i} + \frac{2 \alpha(\delta, n_i)^2}{(1-\gamma) n_i} + 5 \alpha(\delta, n_i) \sqrt{\frac{\spannorm{\Vt_{\gamma, i}} +1+ \frac{3}{n_i}}{n_i}}.\label{eq:conf_intvl_size_1}
    \end{align}
    Now we will set $\gamma = \gstar_i$ and relate all terms in~\eqref{eq:conf_intvl_size_1} to $\spannorm{h^\star}$, but first we bound $\spannorm{\Vt_{\gamma, i}}$. Under the event of Lemma \ref{lem:concentration_event_eps_based} we have from~\eqref{eq:vt_error_2}, and from~\eqref{eq:eps_based_recursive_err_bd_1} in Lemma \ref{lem:eps_based_recursive_err_bds}, that
    \begin{align}
        \spannorm{\Vt_{\gamma, i}} & \leq \spannorm{V^\star_\gamma} + \spannorm{\Vt_{\gamma, i} - V^\star_\gamma}  \nonumber\\
        & \leq \spannorm{V^\star_\gamma} + 2\infnorm{\Vt_{\gamma, i} - V^\star_\gamma}  \nonumber \\
        & \leq \spannorm{V^\star_\gamma} + 2\infnorm{\Vt_{\gamma,i} - \Vhat^\star_{\gamma,i}} + 2\infnorm{\Vhat^\star_{\gamma,i} - V^\star_\gamma} \nonumber \\
        & \leq \spannorm{V^\star_\gamma}+ 2\frac{1}{n_i} + 2\left(\frac{2 \alpha(\delta, n_i)^2}{(1-\gamma)^2 n_i} + \frac{4 \alpha(\delta, n_i)}{1-\gamma} \sqrt{\frac{\spannorm{V_{\gamma}^{\star}} +1+ \frac{1}{n_i}}{n_i}} + \frac{4}{n_i} \right). \label{eq:conf_intvl_size_2}
    \end{align}
    Bounding~\eqref{eq:conf_intvl_size_2} by substituting $\gamma = \gstar_i$ and using Lemma \ref{lem:gstar_i_bounds} to bound $\frac{1}{1-\gstar_i}$, and using that $\spannorm{V^\star_\gamma} \leq 2\spannorm{h^\star}$ \citep[Lemma 2]{wei_model-free_2020}, we have
    \begin{align}
        \spannorm{\Vt_{\gstar_i, i}} &  \leq 2\spannorm{h^\star}+ \frac{2}{n_i} + \frac{4 \alpha(\delta, n_i)^2}{(1-\gstar_i)^2 n_i} + \frac{8 \alpha(\delta, n_i)}{1-\gstar_i} \sqrt{\frac{2\spannorm{h^{\star}} +1+ \frac{1}{n_i}}{n_i}} + \frac{8}{n_i} \nonumber \\ 
        & \leq 2\spannorm{h^\star}+ \frac{10}{n_i} + \frac{4 \alpha(\delta, n_i)^2}{ n_i}4n_i \left(\spannorm{h^{\star}} +1 \right) \nonumber\\
        &\qquad + 8 \alpha(\delta, n_i) 2\sqrt{n_i \left(\spannorm{h^{\star}} +1 \right)} \sqrt{\frac{2\spannorm{h^{\star}} +1+ \frac{1}{n_i}}{n_i}}  \nonumber \\
        & \leq 2\spannorm{h^\star}+ 5 + 16 \alpha(\delta, n_i)^2 \left(\spannorm{h^{\star}} +1 \right) \nonumber\\
        &\qquad + 16 \alpha(\delta, n_i) \sqrt{n_i \left(\spannorm{h^{\star}} +1 \right)} \sqrt{\frac{2\spannorm{h^{\star}} +2}{n_i}} \nonumber \\
        &= 2\spannorm{h^\star}+ 5 + 16 \alpha(\delta, n_i)^2 \left(\spannorm{h^{\star}} +1 \right) + 16\sqrt{2} \alpha(\delta, n_i) \left(\spannorm{h^{\star}} +1 \right) \nonumber \\
        & \leq 44 \alpha(\delta, n_i)^2 \left(\spannorm{h^{\star}} +1 \right) \label{eq:conf_intvl_size_3}
    \end{align}
    where in the third inequality we used that $n_i \geq 2$ to bound $\frac{1}{n_i} \leq \frac{1}{2} \leq 1$, and in the final inequality we used that $\alpha(\delta, n_i) \geq 1$ (which is immediate from the form of $\alpha$ and the facts that $\delta \leq 1, n_i \geq 1$) and that $5 + 16 + 16 \sqrt{2} \leq 44$. Substituting the inequality~\eqref{eq:conf_intvl_size_3} into~\eqref{eq:conf_intvl_size_1} and simplifying in the same ways (including setting $\gamma = \gstar_i$), we obtain that
    \begin{align*}
        & \Ub_{i}(\gstar_i) - \Lb_{i}(\gstar_i) \\
        &\leq  (1-\gstar_i) \spannorm{\Vt_{\gstar_i, i}} + 7 \frac{1-\gstar_i}{n_i} + \frac{2 \alpha(\delta, n_i)^2}{(1-\gstar_i) n_i} + 5 \alpha(\delta, n_i) \sqrt{\frac{\spannorm{\Vt_{\gstar_i, i}} +1+ \frac{3}{n_i}}{n_i}} \\
        &\leq  (1-\gstar_i) 44 \alpha(\delta, n_i)^2 \left(\spannorm{h^{\star}} +1 \right)  + 7 \frac{1-\gstar_i}{n_i} + \frac{2 \alpha(\delta, n_i)^2}{(1-\gstar_i) n_i} \\
        & \quad \quad + 5 \alpha(\delta, n_i) \sqrt{\frac{44 \alpha(\delta, n_i)^2 \left(\spannorm{h^{\star}} +1 \right) +1+ \frac{3}{n_i}}{n_i}} \\
        &\leq  \frac{1}{\sqrt{n_i  \left(\spannorm{h^{\star}} +1 \right) }} 44 \alpha(\delta, n_i)^2 \left(\spannorm{h^{\star}} +1 \right)  + 7 \frac{1}{n_i} + \frac{2 \alpha(\delta, n_i)^2}{ n_i} 2\sqrt{n_i  \left(\spannorm{h^{\star}} +1 \right)} \\
        & \quad \quad + 5 \alpha(\delta, n_i) \sqrt{\frac{47 \alpha(\delta, n_i)^2 \left(\spannorm{h^{\star}} +1 \right) }{n_i}} \\
        &= \left(48 + 5 \sqrt{47} \right) \alpha(\delta, n_i)^2 \sqrt{\frac{\spannorm{h^\star}+1}{n_i}} + \frac{7}{n_i} \\
        & \leq 83  \alpha(\delta, n_i)^2 \sqrt{\frac{\spannorm{h^\star}+1}{n_i}} + 7 \sqrt{\frac{\spannorm{h^\star}+1}{n_i}} \\
        & \leq 90 \alpha(\delta, n_i)^2 \sqrt{\frac{\spannorm{h^\star}+1}{n_i}} .
    \end{align*}
    We have thus shown the first part of the lemma statement. 

    For the second part of the lemma statement, we need to find some $i$ such that  $90 \alpha(\delta, n_i)^2 \sqrt{\frac{\spannorm{h^\star}+1}{n_i}}\leq \varepsilon$. First we compute
    \begin{align*}
        \alpha(\delta, n_i)^4 &= 24^4 16^2 \log^2\left(\frac{24 SA n_i^5}{\delta}\right) \log^4_2 (\log_2 (n_i + 4)) \\
        & \leq 24^4 16^2 117 \log^2\left(\frac{24 SA n_i^5}{\delta}\right) \log(n_i) \\
        & \leq 24^4 16^2 5^2 117 \log^2\left(\frac{24 SA n_i}{\delta}\right) \log(n_i) \\
        & \leq 24^4 16^2 5^2 117 \log^3\left(\frac{24 SA n_i}{\delta}\right)  
    \end{align*}
    so
    \begin{align}
        & 90 \alpha(\delta, n_i)^2 \sqrt{\frac{\spannorm{h^\star}+1}{n_i}}\leq \varepsilon \nonumber\\
        \iff & n_i \geq 90^2 \alpha(\delta, n_i)^4 \frac{\spannorm{h^\star}+1}{\varepsilon^2} \nonumber \\
        \overset{\text{(I)}}{\impliedby} & n_i \geq 90^2 24^4 16^2 5^2 117 \frac{\spannorm{h^\star}+1}{\varepsilon^2} \log^3\left(\frac{24 SA n_i}{\delta}\right) \nonumber\\
        \overset{\text{(II)}}{\impliedby} &n_i \geq 10 \cdot 90^2 24^4 16^2 5^2 117 \frac{\spannorm{h^\star}+1}{\varepsilon^2} \log^3\left(\frac{240 SA }{\delta} 90^2 24^4 16^2 5^2 117 \frac{\spannorm{h^\star}+1}{\varepsilon^2}\right) \label{eq:eps_based_n_sufficient_cond}
    \end{align}
    where (I) is due to Lemma \ref{lem:logs_upper_bound} and (II) is due to Lemma \ref{lem:sample_size_unraveling}. Therefore if we set 
    \[
    \BD = \left \lceil \log_2  \left(C_1\frac{\spannorm{h^\star}+1}{\varepsilon^2}  \log^3\left(\frac{C_2 SA (\spannorm{h^\star}+1)}{\delta \varepsilon} \right) \right) \right \rceil
    \]
    where $C_1 = 10 \cdot 90^2 24^4 16^2 5^2 117$ and $C_2 = 240 \cdot 90^2 24^4 16^2 5^2 117$, then~\eqref{eq:eps_based_n_sufficient_cond} will be satisfied for $i = \BD$.
\end{proof}

\begin{lem}
\label{lem:logs_upper_bound}
$\left(\log_2 \log_2 (x+4) \right)^4 \leq 117 \log x$ for all $x \geq 2$.
\end{lem}
\begin{proof}
    Since $z \mapsto 3 \cdot 2^{z/4}$ is convex, $3 \cdot 2^{z/4}$ is lower-bounded by its tangent line at $z = 4$, so for all $z$ we have
    \[
    3 \cdot 2^{z/4}\geq 3 \cdot 2^{4/4} +3 \cdot 2^{4/4} \cdot \ln(2^{1/4}) \cdot (z-4) = 6 + \frac{3 \ln(2)}{2}(z-4).
    \]
    Since $\frac{3 \ln(2)}{2} \in (1, 1.1)$, we have for all $z \geq 4$ that
    \[
    6 + \frac{3 \ln(2)}{2}(z-4) \geq 6 + (z-4) = z + 2
    \]
    and for all $0 \leq z \leq 4$ that
    \[
    6 + \frac{3 \ln(2)}{2}(z-4) \geq 6 + 1.1(z-4) > 1.1 z + 1.5 \geq z + 1.5
    \]
    so we have that $3 \cdot 2^{z/4}\geq z + 1.5$ for all $z \geq 0$. Also $\log_2 (1 + \log_2 3) < 1.5$, so we have $3 \cdot 2^{z/4}\geq z + \log_2 (1 + \log_2 3)$. Now letting $y = 2^z$ or equivalently $z = \log_2 y$ (and we must have $y \geq 1$ since $z \geq 0$), we have that
    \begin{align*}
        \log_2 (y + \log_2 3) \leq \log_2 (y + y\log_2 3)
        &= \log_2 y + \log_2 (1 + \log_2 3) \\
        &=z + \log_2 (1 + \log_2 3) \leq 3 \cdot 2^{z/4} = 3 y^{1/4}.
    \end{align*}
    This implies
    \begin{align*}
        \left(  \log_2 (y + \log_2 3) \right)^4 \leq 3^4 y.
    \end{align*}
    Now letting $x = 2^y$ or equivalently $y = \log_2(x)$ (and we must have $x \geq 2$ since $y \geq 1$), we have that
    \begin{align*}
        \left(  \log_2 (\log_2 (x + 4)) \right)^4 \leq \left(  \log_2 (\log_2 (3x)) \right)^4 = \left(  \log_2 (\log_2 x + \log_2 3) \right)^4=\left(  \log_2 (y + \log_2 3) \right)^4 \leq 3^4 y ,
    \end{align*}
    where the first inequality is because $x + 4 \leq 3x$ for $x \geq 2$. We also have
    \begin{align*}
        3^4 y = 3^4 \log_2(x) = 3^4 \log_2(e) \ln x < 117 \ln(x).
    \end{align*}
    Thus, we can conclude the desired result for any $x \geq 2$ by making the appropriate choice of $z$ (that is, $z = \log_2 \log_2 x$).
\end{proof}

\begin{lem}
\label{lem:sample_size_unraveling}
    Suppose that $x , y \geq 1$. Then if $n \geq 10 x \log^3 (10xy)$, we have that $n \geq x \log^3(y n)$.
\end{lem}
\begin{proof}
    The desired conclusion $n \geq x \log^3(y n)$ is equivalent to $\frac{n}{\log^3(y n)} \geq x$. The derivative of $\frac{n}{\log^3(y n)}$ with respect to $n$ is $\frac{\log(yn)-3}{\log^4(yn)}$ which is $\geq 0$ if $n \geq \frac{e^3}{y}$, so the function $\frac{n}{\log^3(y n)}$ is monotone non-decreasing for $n \geq \frac{e^3}{y}$. Thus it suffices to find some $m$ such that $m \geq \frac{e^3}{y}$ and $\frac{m}{\log^3(y m)} \geq x$, because then by monotonicity we have that $n \geq m$ implies $\frac{n}{\log^3(y n)} \geq \frac{m}{\log^3(y m)} \geq x$. 
    
    Now we claim that $m = 10 x \log^3 (10xy)$ satisfies these conditions. First, since $x, y \geq 1$, we have that $\log (10xy) > \log (e^2) = 2$, so $m = 10 x \log^3 (10xy) > 10 \cdot 2^3 = 80 \geq e^3 \geq \frac{e^3}{y}$, meeting the first required condition. Next, we have that
    \begin{align*}
        \log^3(y m) &= \log^3\left( y 10 x \log^3 (10xy) \right) = \left(\log(10xy) + 3 \log \log (10xy) \right)^3 \\
        &\leq \left(\log(10xy) + \frac{3}{e} \log (10xy) \right)^3 \leq 10 \log^3 (10xy)
    \end{align*}
    where we used that $\log x \leq \frac{x}{e}$ and then in the last step we used that $(1 + \frac{3}{e})^3 \leq 10$. (To obtain the inequality $\log x \leq \frac{x}{e}$, first we can show the inequality $x \leq e^{x/e}$ by noting that $e^{x/e}$ is convex, using the first-order convexity condition, and taking the tangent line to $e^{x/e}$ at $x = e$. Then we can take $\log$ of both sides.) Thus
    \begin{align*}
        \frac{m}{\log^3(y m)} \geq \frac{10 x \log^3 (10xy)}{10 \log^3 (10xy)} \geq x
    \end{align*}
    as desired.
\end{proof}

Combining the consequences of Lemmas \ref{lem:CI_validity_eps_based} and \ref{lem:small_confidence_interval}, we can prove Theorem \ref{thm:eps_based_alg}.
\begin{proof}[Proof of Theorem \ref{thm:eps_based_alg}]
    For this proof we assume the event in Lemma \ref{lem:concentration_event_eps_based} holds, which occurs with probability at least $1-\delta$. 
    First we argue that the algorithm terminates and uses at most the claimed number of samples. By Lemma \ref{lem:small_confidence_interval}, we have that $\Ub_{\BD}(\gstar_\BD) - \Lb_{\BD}(\gstar_\BD) \leq \varepsilon.$ Since this would trigger the termination condition of Algorithm \ref{alg:eps_based_alg}, this implies that the algorithm must terminate on or before iteration $\BD$. Therefore the total number of samples used per state-action pair is at most
    \begin{align*}
        \sum_{i = 1}^{\BD} n_i &= \sum_{i = 1}^{\BD} 2^i \leq 2^{\BD + 1} \\
        &= 2 \cdot 2^{\left \lceil \log_2  \left(C_1\frac{\spannorm{h^\star}+1}{\varepsilon^2}  \log^3\left(\frac{C_2 SA (\spannorm{h^\star}+1)}{\delta \varepsilon} \right) \right) \right \rceil} \\
        & \leq 4 C_1\frac{\spannorm{h^\star}+1}{\varepsilon^2}  \log^3\left(\frac{C_2 SA (\spannorm{h^\star}+1)}{\delta \varepsilon} \right) =: N.
    \end{align*}
    Also we can further bound $\BD$, which is an upper bound on the number of iterations, by
    \begin{align*}
        \BD &= \left \lceil \log_2  \left(C_1\frac{\spannorm{h^\star}+1}{\varepsilon^2}  \log^3\left(\frac{C_2 SA (\spannorm{h^\star}+1)}{\delta \varepsilon} \right) \right) \right \rceil \\
        &\leq \log_2  \left(C_1\frac{\spannorm{h^\star}+1}{\varepsilon^2}  \log^3\left(\frac{C_2 SA (\spannorm{h^\star}+1)}{\delta \varepsilon} \right) \right) + 1 \\
        & = \log_2  \left(2C_1\frac{\spannorm{h^\star}+1}{\varepsilon^2}  \log^3\left(\frac{C_2 SA (\spannorm{h^\star}+1)}{\delta \varepsilon} \right) \right) \\
        & \leq \log_2(N).
    \end{align*}
    Since the algorithm terminates and by definition of $\Lb$ and $\Ub$, it is immediate that we have $\Ub - \Lb \leq \varepsilon$.
    By Lemma \ref{lem:CI_validity_eps_based}, we have for all $i \geq 1$ and all $\gamma \in \hzns_i$ that
    \begin{align*}
        \Lb_i(\gamma) \one \leq \rho^{\pit_{\gamma, i}} \leq \rho^\star \leq \Ub_i(\gamma) \one
    \end{align*}
    so in particular we have, for the (random) final iteration $I$, that
    \begin{align*}
        \Lb \one = \Lb_I(\gammahat_I) \one \leq \rho^{\pit_{\gamma, I}} = \rho^{\pihat} \leq \rho^\star \leq \Ub_I(\gammahat_I) \one = \Ub \one.
    \end{align*}
    Finally, combining this with the fact that $\Ub - \Lb \leq \varepsilon$ we see that
    \begin{align*}
        \rho^{\pihat} \geq \Lb \one \geq \Ub \one - \varepsilon \one \geq \rho^\star - \varepsilon \one
    \end{align*}
    as desired.
\end{proof}

\subsection{Proof of Theorem \ref{thm:n_based_alg}}
\label{sec:n_based_alg_proof}
We start by recalling the definitions of some objects which appear in Algorithm \ref{alg:n_based_alg} for convenience. These definitions will be in effect for the entirety of this subsection. We define the empirical transition kernel $\Phat(s' \mid s, a) = \frac{1}{n}\sum_{i=1}^n \ind\{S^i_{s,a} = s'\}$, for all $s' \in \S$, using the $n$ samples drawn from all state-action pairs within Algorithm \ref{alg:n_based_alg}.
Let $\hzns = \{\gamma: \text{there exists an integer $k $ such that $\sqrt{n} \leq \frac{1}{1-\gamma} = 2^k \leq n$}\}$. Then for all $\gamma \in \hzns$ we define the policy $\pit_\gamma$ and value function $\Vt_\gamma$ as the outputs of $\SolveDMDP(\Phat, r, \gamma, \frac{1}{n})$.

Now let $\gstar$ be the smallest member of $\hzns$ that corresponds to an effective horizon at least $\sqrt{n (\spannorm{h^\star}+1)}$, that is,
\begin{align}
    \gstar = \inf \left\{ \gamma \in \hzns : \frac{1}{1-\gamma}\geq \sqrt{n \left(\spannorm{h^\star} + 1 \right)}\right\}. \label{eq:gstar_defn}
\end{align}

\begin{lem}
    \label{lem:gstar_bds}
    If $n \geq 4$ then $\hzns$ is nonempty. Furthermore if also $n \geq 16(\spannorm{h^\star}+1)$, then
\begin{align}
    \sqrt{n (\spannorm{h^\star}+1)} \leq \frac{1}{1-\gstar} \leq 2 \sqrt{n (\spannorm{h^\star}+1)}. \label{eq:gstar_bds}
\end{align}
\end{lem}
\begin{proof}
    First we note that $\hzns$ is nonempty if the interval $[\sqrt{n},n]$ contains some power of $2$, which is ensured if $n \geq 2 \sqrt{n}$, or equivalently if $n \geq 4$. It is also simple to check for $n \in \{1,2,3\}$ that $\hzns$ is also nonempty.

    Next, the smallest power of $2$ which is $\geq\sqrt{n (\spannorm{h^\star}+1)}$ is at most $2\sqrt{n (\spannorm{h^\star}+1)}$, so to guarantee it is contained in $\hzns$ we need it to be $\leq$ the largest element of $\hzns$, which is at least $n/2$. Therefore if
    \begin{align*}
        2\sqrt{n (\spannorm{h^\star}+1)} \leq n/2 \iff n \geq 16 (\spannorm{h^\star}+1)
    \end{align*}
    then~\eqref{eq:gstar_bds} will be satisfied.
\end{proof}

\begin{lem}
\label{lem:concentration_event_n_based}
    Define the function $\alpha(\tilde{\delta}, \tilde{n}) = 24 \sqrt{16 \log \left( \frac{24 SA \tilde{n}^5}{\tilde{\delta}}\right)} \log_2\left( \log_2 (\tilde{n} + 4) \right)$. Fix $\delta > 0$. Then with probability at least $1 - \delta$, we have for all $\gamma \in \hzns$ that
\begin{align}
    \infnorm{V_{\gamma}^{\pistar_{\gamma}} - \Vhat_{\gamma}^{\pistar_{\gamma}}} &\leq \frac{\alpha(\delta, n)}{1-\gamma} \sqrt{\frac{ \spannorm{V_{\gamma}^{\pistar_{\gamma}}} + 1 }{n}} \label{eq:true_opt_pol_eval_bd_1}
\end{align}
    and also the subroutine on line \ref{alg:solver_step} outputs a policy $\pit_{\gamma}$ such that
\begin{align}
        \Vhat_{\gamma}^{\pit_{\gamma}} &\geq \Vhat^\star_{\gamma} - \frac{1}{n} \one \label{eq:pit_near_opt_1}\\
        \infnorm{\Vt_{\gamma} - \Vhat^\star_{\gamma}} &\leq \frac{1}{n} \label{eq:vt_error_1}\\
        \infnorm{\Vhat_{\gamma}^{\pit_{\gamma}} - V_\gamma^{\pit_{\gamma}}} & \leq \frac{\alpha(\delta, n)}{1-\gamma} \sqrt{\frac{ \spannorm{\Vhat_{\gamma}^{\pit_{\gamma}}} + 1 }{n}} \label{eq:emp_opt_pol_eval_bd_1}.
\end{align}
\end{lem}
\begin{proof}
    Following identical steps as to the proof of Lemma \ref{lem:concentration_event_eps_based}, but with $n$ in place of $n_i$, up until equations~\eqref{eq:conc_bd1_s2} and~\eqref{eq:conc_bd2_s2}, we obtain that with probability at least $1-\delta$ the two inequalities
    \begin{align*}
       \infnorm{V_\gamma^{\pistar_\gamma} - \Vhat_\gamma^{\pistar_\gamma}} & \leq \frac{24\log_2 \log_2 \left( n + 4\right)}{1-\gamma} \sqrt{16 \frac{\spannorm{V^{\star}_\gamma} + 1}{n} \log\left( \frac{24 SAn^4}{ \delta}\right)}
    \end{align*}
    and
    \begin{align*}
       \infnorm{V_\gamma^{\pit_\gamma} - \Vhat_\gamma^{\pit_\gamma}} & \leq \frac{24\log_2 \log_2 \left( n + 4\right)}{1-\gamma} \sqrt{16 \frac{\spannorm{\Vhat_\gamma^{\pit_\gamma}} + 1}{n} \log\left( \frac{24 SAn^4}{ \delta}\right)}
    \end{align*}
    both hold. We obtain the desired conclusion by noting that
    \[
    \alpha(\delta, n) \geq 24\log_2 \log_2 \left( n + 4\right) \sqrt{16 \log\left( \frac{24 SAn^4}{ \delta}\right)} .
    \]
\end{proof}

\begin{lem}
\label{lem:gain_lower_bound_validity}
Under the event described in Lemma \ref{lem:concentration_event_n_based}, we have
    \begin{align}
        \rho^{\pit_\gamma} & \geq \Obj(\gamma)\one = (1-\gamma)\min_{s}\Vt_{\gamma}(s) \one - 2\frac{1-\gamma}{n}\one - \alpha(\delta, n) \sqrt{\frac{\spannorm{\Vt_{\gamma}} +\frac{3}{n} +1 }{n}} \one
    \end{align}
    for all $\gamma \in \hzns$.    
\end{lem}
\begin{proof}
    The proof is very similar to the first part of the proof of Lemma \ref{lem:CI_validity_eps_based}.
    Fix $\gamma \in \hzns$. Using inequality~\eqref{eq:sharper_DMDP_red_fixed} from Lemma \ref{lem:AMDP_DMDP_relationships}, then using the triangle inequality,
    then~\eqref{eq:pit_near_opt_1}, then~\eqref{eq:emp_opt_pol_eval_bd_1}, then the triangle inequality again, then~\eqref{eq:vt_error_1}, we have
    \begin{align}
        \rho^{\pit_\gamma} &\geq (1-\gamma)\min_{s} V_\gamma^{\pit_\gamma}(s) \one \nonumber\\
        & \geq (1-\gamma)\min_{s}\Vhat_\gamma^{\pit_\gamma}(s) \one - (1-\gamma)\infnorm{\Vhat_\gamma^{\pit_\gamma} - V_\gamma^{\pit_\gamma}} \one \nonumber \\
        & \geq (1-\gamma)\min_{s}\Vhat_\gamma^{\star}(s) \one - \frac{1-\gamma}{n}\one - (1-\gamma)\infnorm{\Vhat_\gamma^{\pit_\gamma} - V_\gamma^{\pit_\gamma}} \one \nonumber \\
        & \geq (1-\gamma)\min_{s}\Vhat_\gamma^{\star}(s) \one - \frac{1-\gamma}{n}\one - \alpha(\delta, n) \sqrt{\frac{\spannorm{\Vhat_\gamma^{\pit_\gamma}} +1 }{n}} \one \nonumber \\
        & \geq (1-\gamma)\min_{s}\Vt_\gamma(s) \one - (1-\gamma)\infnorm{\Vhat_\gamma^{\star}- \Vt_\gamma}\one - \frac{1-\gamma}{n}\one - \alpha(\delta, n) \sqrt{\frac{\spannorm{\Vhat_\gamma^{\pit_\gamma}} +1 }{n}} \one \nonumber \\
        & \geq (1-\gamma)\min_{s}\Vt_{\gamma}(s) \one - 2\frac{1-\gamma}{n}\one - \alpha(\delta, n) \sqrt{\frac{\spannorm{\Vhat_\gamma^{\pit_\gamma}} +1 }{n}} \one .\label{eq:gain_LB_1}
    \end{align}
    Now, nearly identically to the bound~\eqref{eq:eps_based_CI_sp_bd}, we can use the requirements~\eqref{eq:pit_near_opt_1} and~\eqref{eq:vt_error_1} to bound
    \begin{align}
        \spannorm{\Vhat_{\gamma}^{\pit_{\gamma}}} &\leq \spannorm{\Vhat_{\gamma}^{\star}} + \spannorm{\Vhat_{\gamma}^{\star} - \Vhat_{\gamma}^{\pit_{\gamma}}} \nonumber\\
         &\leq \spannorm{\Vt_{\gamma}} +\spannorm{\Vt_{\gamma} - \Vhat_{\gamma}^{\star}} + \spannorm{\Vhat_{\gamma}^{\star} - \Vhat_{\gamma}^{\pit_{\gamma}}} \nonumber\\
         & \leq \spannorm{\Vt_{\gamma}} +2\infnorm{\Vt_{\gamma} - \Vhat_{\gamma}^{\star}} + \frac{1}{n} \nonumber \\
        & \leq \spannorm{\Vt_{\gamma}} +\frac{3}{n} . \label{eq:n_based_sp_bd}
    \end{align}
    Finally combining~\eqref{eq:n_based_sp_bd} with~\eqref{eq:gain_LB_1}, we obtain that
    \begin{align*}
        \rho^{\pit_\gamma} & \geq (1-\gamma)\min_{s}\Vt_{\gamma}(s) \one - 2\frac{1-\gamma}{n}\one - \alpha(\delta, n) \sqrt{\frac{\spannorm{\Vt_{\gamma}} +\frac{3}{n} +1 }{n}} \one = \Obj(\gamma)\one
    \end{align*}
    as desired.
\end{proof}

\begin{lem}
\label{lem:n_based_recursive_err_bds}
    Under the event in Lemma \ref{lem:concentration_event_n_based}, for all $\gamma \in \hzns$, we have
    \begin{align}
        \infnorm{\Vhat_{\gamma}^{\star} - V_{\gamma}^{\star}} & \leq \frac{2 \alpha(\delta, n)^2}{(1-\gamma)^2 n} + \frac{4 \alpha(\delta, n)}{1-\gamma} \sqrt{\frac{\spannorm{V_{\gamma}^{\star}} +1+ \frac{1}{n}}{n}} + \frac{4}{n}. \label{eq:n_based_recursive_err_bd_1}
    \end{align}
\end{lem}
\begin{proof}
    Similarly to Lemma \ref{lem:eps_based_recursive_err_bds}, this follows immediately from Lemma \ref{lem:recursive_err_bds_abstract} as its conditions are satisfied under the event in Lemma \ref{lem:concentration_event_n_based} (by setting $m = n$ and $\beta = \alpha(\delta, n)$ for all $\gamma \in \hzns$).
\end{proof}

\begin{lem}
\label{lem:n_based_opt_obj_lb}
Suppose $n \geq 16(\spannorm{h^\star}+1)$.
Under the event described in Lemma \ref{lem:concentration_event_n_based}, we have
    \begin{align*}
        \Obj(\gstar)\one \geq \rho^\star - 30\alpha(\delta, n)^2 \sqrt{\frac{\spannorm{h^\star}+1   }{n}} \one   . 
    \end{align*}
\end{lem}
\begin{proof}
    First we bound $\infnorm{\Vhat_{\gamma}^{\star} - V_{\gamma}^{\star}}$ in terms of $\spannorm{h^\star}$. Using the bound~\eqref{eq:n_based_recursive_err_bd_1} from Lemma \ref{lem:n_based_recursive_err_bds} and substituting $\gamma = \gstar$, as well as using Lemma \ref{lem:gstar_bds} to bound $\frac{1}{1-\gstar}$, we have
    \begin{align}
        \infnorm{\Vhat_{\gstar}^{\star} - V_{\gstar}^{\star}} 
        & \leq \frac{2 \alpha(\delta, n)^2}{(1-\gstar)^2 n} + \frac{4 \alpha(\delta, n)}{1-\gstar} \sqrt{\frac{\spannorm{V_{\gstar}^{\star}} +1+ \frac{1}{n}}{n}} + \frac{4}{n} \nonumber \\
        & \leq \frac{2 \alpha(\delta, n)^2}{ n}4 n \left(\spannorm{h^\star}+1 \right) + 8 \alpha(\delta, n) \sqrt{n \left(\spannorm{h^\star}+1 \right)} \sqrt{\frac{\spannorm{V_{\gstar}^{\star}} +1+ \frac{1}{n}}{n}} + \frac{4}{n} \nonumber \\
        & \leq 8 \alpha(\delta, n)^2 \left(\spannorm{h^\star}+1 \right) + 8 \alpha(\delta, n) \sqrt{n \left(\spannorm{h^\star}+1 \right)} \sqrt{\frac{2 \spannorm{h^\star} +2}{n}} + \frac{1}{4} \nonumber \\
        & \leq 8 \alpha(\delta, n)^2 \left(\spannorm{h^\star}+1 \right) + (8 \sqrt{2}+1/4) \alpha(\delta, n) \left(\spannorm{h^\star}+1 \right) \nonumber \\
        & \leq 20 \alpha(\delta, n)^2 \left(\spannorm{h^\star}+1 \right) \label{eq:n_based_opt_obj_lb_1}
    \end{align}
    also using the facts that $n \geq 16$, that $\alpha(\delta, n) \geq 1$, and that $\spannorm{V_{\gstar}^{\star}} \leq 2 \spannorm{h^\star}$ \cite[Lemma 2]{wei_model-free_2020}.

    We can use the triangle inequality, that $\spannorm{\cdot} \leq 2 \infnorm{\cdot}$, triangle inequality again, that $\spannorm{V_{\gstar}^{\star}} \leq 2 \spannorm{h^\star}$,~\eqref{eq:vt_error_1} and~\eqref{eq:n_based_opt_obj_lb_1}, and then that $\alpha(\delta, n) \geq 1$ and $n\geq 16$ to bound
    \begin{align}
        \spannorm{\Vt_{\gstar}} & \leq \spannorm{V_{\gstar}^\star} + \spannorm{\Vt_{\gstar} - V_{\gstar}^\star} \nonumber\\
        & \leq \spannorm{V_{\gstar}^\star} + 2\infnorm{\Vt_{\gstar} - V_{\gstar}^\star} \nonumber\\
        & \leq 2\spannorm{h^\star} + 2\infnorm{\Vt_{\gstar} - \Vhat_{\gstar}^\star} + 2\infnorm{\Vhat_{\gstar}^\star - V_{\gstar}^\star}\nonumber\\
        & \leq 2\spannorm{h^\star} + \frac{2}{n} + 40\alpha(\delta, n)^2 \left(\spannorm{h^\star}+1 \right) \nonumber\\
        & \leq 42\alpha(\delta, n)^2 \left(\spannorm{h^\star}+1 \right). \label{eq:n_based_opt_obj_lb_2}
    \end{align}

    Now using the inequalities~\eqref{eq:n_based_opt_obj_lb_1} and~\eqref{eq:n_based_opt_obj_lb_2} to lower-bound $\Obj(\gstar)$, we obtain
    \begin{align*}
        \Obj(\gstar)\one 
        &= (1-\gstar)\min_{s}\Vt_{\gstar}(s) \one - 2\frac{1-\gstar}{n}\one - \alpha(\delta, n) \sqrt{\frac{\spannorm{\Vt_{\gstar}} +\frac{3}{n} +1 }{n}} \one \\
        &\geq (1-\gstar)\min_{s}\Vhat_{\gstar}^{\star}(s) \one - 3\frac{1-\gstar}{n}\one - \alpha(\delta, n) \sqrt{\frac{\spannorm{\Vt_{\gstar}} +\frac{3}{n} +1 }{n}} \one \\
        & \geq (1-\gstar)\min_{s}V_{\gstar}^{\star}(s) \one - (1-\gstar)\infnorm{\Vhat_{\gstar}^{\star} - V_{\gstar}^{\star}}\one - 3\frac{1-\gstar}{n}\one \\
        &\quad - \alpha(\delta, n) \sqrt{\frac{\spannorm{\Vt_{\gstar}} +\frac{3}{n} +1 }{n}} \one \\
        & \geq (1-\gstar)\min_{s}V_{\gstar}^{\star}(s) \one - (1-\gstar)21 \alpha(\delta, n)^2 \left(\spannorm{h^\star}+1 \right)\one - 3\frac{1-\gstar}{n}\one \nonumber\\
        &\qquad - \alpha(\delta, n)^2 \sqrt{\frac{42\left(\spannorm{h^\star}+1 \right) +\frac{3}{n} +1 }{n}} \one \\
        & \geq \rho^\star -2(1-\gstar) \spannorm{h^\star} \one - (1-\gstar)21 \alpha(\delta, n)^2 \left(\spannorm{h^\star}+1 \right)\one - 3\frac{1-\gstar}{n}\one \nonumber\\
        &\qquad - \alpha(\delta, n)^2 \sqrt{\frac{44\left(\spannorm{h^\star}+1 \right)  }{n}} \one \\
        & \geq \rho^\star  - (1-\gstar)22 \alpha(\delta, n)^2 \left(\spannorm{h^\star}+1 \right)\one - \alpha(\delta, n)^2 \sqrt{\frac{44\left(\spannorm{h^\star}+1 \right)  }{n}} \one \\
        & \geq \rho^\star  - \frac{1}{\sqrt{n  \left(\spannorm{h^\star}+1 \right) }}22 \alpha(\delta, n)^2 \left(\spannorm{h^\star}+1 \right)\one - \alpha(\delta, n)^2 \sqrt{\frac{44\left(\spannorm{h^\star}+1 \right)  }{n}} \one \\
        & \geq \rho^\star - 30\alpha(\delta, n)^2 \sqrt{\frac{\spannorm{h^\star}+1   }{n}} \one
    \end{align*}
    where in the first inequality we used that $\infnorm{\Vhat^\star_{\gstar} -\Vt_{\gstar}} \leq \frac{1}{n}$ from~\eqref{eq:vt_error_1} and triangle inequality, in the second inequality we used triangle inequality, in the third we used~\eqref{eq:n_based_opt_obj_lb_1} and~\eqref{eq:n_based_opt_obj_lb_2}, in the fourth we used that $\min_{s}V_{\gstar}^{\star}(s) \geq \frac{1}{1-\gstar}\rho^\star -  \spannorm{h^\star}$ from \cite[Lemma 2]{wei_model-free_2020} and that $n \geq 16$, in the fifth we again used that $n \geq 16$, and in the sixth we used~\eqref{eq:gstar_bds} to upper bound $(1-\gstar)$.
\end{proof}

Now we combine these intermediate results to prove Theorem \ref{thm:n_based_alg}.
\begin{proof}[Proof of Theorem \ref{thm:n_based_alg}]
    Under the event in Lemma \ref{lem:concentration_event_n_based} which holds with probability at least $1 - \delta$, we have by Lemma \ref{lem:gain_lower_bound_validity} that $\rho^{\pit_{\gammahat}} \geq \Obj(\gammahat)\one$. Since we trivially have $\rho^{\pit_{\gammahat}} \geq 0\one$, this implies $\rho^{\pit_{\gammahat}} \geq \max \{\Obj(\gammahat), 0\}\one = \rhohat$. Now to lower bound $\rhohat$ we consider two cases. First, if $n \geq 16(\spannorm{h^\star}+1)$, then we have 
    \begin{align*}
    \rhohat &= \max \{\Obj(\gammahat), 0\}\one  \geq \Obj(\gammahat) \geq \Obj(\gstar)\one  \geq \rho^\star - 30 \alpha(\delta, n)^2 \sqrt{\frac{\spannorm{h^\star}+1}{n}}\one
    \end{align*}
    where the second inequality step is by the definition of $\gammahat$, and third inequality is from Lemma \ref{lem:n_based_opt_obj_lb}. Next, if $n < 16(\spannorm{h^\star}+1)$, then
    \begin{align*}
    \rhohat &= \max \{\Obj(\gammahat), 0\}\one  \geq 0 \one  \geq \rho^\star - 30 \alpha(\delta, n)^2 \sqrt{\frac{\spannorm{h^\star}+1}{n}}\one
    \end{align*}
    where the final inequality is because $\rho^\star \leq \one$ so the condition on $n$ ensures the RHS is $< \zero$ (note $\alpha(\delta, n) \geq 1$). Finally, we let $C_3 = 30$.
\end{proof}

\subsection{Proof of Lemma \ref{lem:span_constrained_planning_subroutine}}
\label{sec:span_constrained_planning_subroutine_proof}
\begin{proof}[Proof of Lemma \ref{lem:span_constrained_planning_subroutine}]
    First we check that $T$ ensures that $\gamma^T \leq \frac{(1-\gamma)^2 \varepsilon}{3}$. We have
    \begin{align*}
        T \geq \frac{\log (\frac{3 }{(1-\gamma)^2 \varepsilon})}{1-\gamma} \geq \frac{\log (\frac{3 }{(1-\gamma)^2 \varepsilon})}{\log(1/\gamma)}
    \end{align*}
    using the fact that $\frac{1}{1-\gamma} \geq \frac{1}{\log(1/\gamma)}$ for $\gamma \in (0,1)$. Thus
    \begin{align}
        \gamma^T \leq \gamma^{\frac{\log (\frac{3 }{(1-\gamma)^2 \varepsilon})}{\log(1/\gamma)}} = \left(e^{-1} \right)^{\log (\frac{3 }{(1-\gamma)^2 \varepsilon})} = \frac{(1-\gamma)^2 \varepsilon}{3}. \label{eq:gamma_T_bd}
    \end{align}

    By \cite[Lemma 16]{fruit_efficient_2018} $\Clip_{\spb}$ is non-expansive with respect to $\infnorm{\cdot}$, so since $\T_\gamma$ is $\gamma$-contractive, we have that $\L = \Clip_{\spb} \circ \T_\gamma$ is also $\gamma$-contractive. It is a standard fact that this implies $\L$ has a unique fixed point, which we name $V^\star_{\gamma, \spb}$.

    Thus
    \begin{align*}
        \infnorm{V^T - V^\star_{\gamma, \spb}} &= \infnorm{\L(V^{T-1}) - \L(V^\star_{\gamma, \spb}) } \\
        & \leq \gamma \infnorm{V^{T-1}-V^\star_{\gamma, \spb}} \leq \cdots \leq \gamma^T \infnorm{V^0 - V^\star_{\gamma, \spb}} \leq \gamma^T \frac{1}{1-\gamma},
    \end{align*}
    which is $\leq \frac{(1-\gamma)\varepsilon}{3} \leq \varepsilon$ using~\eqref{eq:gamma_T_bd}.
    Similarly, we have
    \begin{align*}
        \infnorm{V^T - \L(V^T)} &= \infnorm{\L(V^{T-1}) - \L(V^T)} \leq \gamma \infnorm{V^{T-1} - V^{T}} = \gamma \infnorm{V^{T-1} - \L(V^{T-1})} \\
        & \leq \cdots \leq \gamma^T \infnorm{V^0 - \L(V^0)} \leq \gamma^T \frac{1}{1-\gamma} \leq \frac{(1-\gamma) \varepsilon}{3}
    \end{align*}
    again using~\eqref{eq:gamma_T_bd}.
    By definition of $\rt$, we immediately have that $\rt \leq r$.
    Also by construction of $\rt$ we have that 
    \begin{align*}
        \rt_{\pihat} + \gamma P_{\pihat} V^T = \min \left\{r_{\pihat} + \gamma P_{\pihat} V^T, \spb\one + \min_{s'} V^T(s')\one \right\} ,
    \end{align*}
    where the $\min$ is elementwise.
    Thus
    \begin{align}
        &\infnorm{\L(V^T) - \rt_{\pihat} - \gamma P_{\pihat} V^T} \nonumber\\
        &= \infnorm{\min \left\{r_{\pihat} + \gamma P_{\pihat} V^T, \spb \one+ \min_{s'} (r_{\pihat} + \gamma P_{\pihat} V^T )(s')\one \right\} - \min \left\{r_{\pihat} + \gamma P_{\pihat} V^T, \spb \one + \min_{s'} V^T(s')\one \right\}} \nonumber\\
        & \leq \infnorm{\spb \one+ \min_{s'} (r_{\pihat} + \gamma P_{\pihat} V^T )(s')\one - \left( \spb \one + \min_{s'} V^T(s')\one \right)} \nonumber\\
        &= \left|\min_{s'} (r_{\pihat} + \gamma P_{\pihat} V^T )(s') -  \min_{s'} V^T(s')\right| \nonumber\\
        &= \left|\min_{s'} \Clip_{\spb}(r_{\pihat} +\gamma P_{\pihat} V^T )(s') -  \min_{s'} V^T(s')\right| \nonumber\\
        &= \left|\min_{s'} \Clip_{\spb}(\T_\gamma (V^T) )(s') -  \min_{s'} V^T(s')\right| \nonumber\\
        &= \left|\min_{s'} \L (V^T) (s') -  \min_{s'} V^T(s')\right| \nonumber\\
        & \leq \infnorm{\L (V^T) - V^T} , \label{eq:rt_defn_property_bd}
    \end{align}
    where we use the elementary fact that $|\min\{a,b\} - \min \{c,d\}| \leq \max \{ |a-c|, |b-d| \}$ for $ a,b,c,d \in \R$, as well as the fact that $\Clip_{\spb}$ does not change the minimum entry of its input.

    Now we can calculate that
    \begin{align*}
        \infnorm{V^{\pihat}_{\gamma, \rt} - V^T} &\leq \infnorm{V^{\pihat}_{\gamma, \rt} - \L(V^T)} + \infnorm{\L(V^T) - V^T} \\
        &\leq  \infnorm{V^{\pihat}_{\gamma, \rt} - \rt_{\pihat} - \gamma P_{\pihat} V^T} + \infnorm{\L(V^T) - \rt_{\pihat} - \gamma P_{\pihat} V^T} + \infnorm{\L(V^T) - V^T} \\
        &= \infnorm{\rt_{\pihat} - \gamma P_{\pihat} V^{\pihat}_{\gamma, \rt} - \rt_{\pihat} - \gamma P_{\pihat} V^T} + \infnorm{\L(V^T) - \rt_{\pihat} - \gamma P_{\pihat} V^T} + \infnorm{\L(V^T) - V^T} \\
        & \overset{\text{(I)}}{\leq} \infnorm{\rt_{\pihat} - \gamma P_{\pihat} V^{\pihat}_{\gamma, \rt} - \rt_{\pihat} - \gamma P_{\pihat} V^T}  + 2\infnorm{\L(V^T) - V^T}\\
        & \overset{\text{(II)}}{\leq}\gamma \infnorm{V^{\pihat}_{\gamma, \rt} - V^T}  + 2\infnorm{\L(V^T) - V^T},
    \end{align*}
    where in (I) we used~\eqref{eq:rt_defn_property_bd} and in (II) we used that $\infinfnorm{P_{\pihat}} \leq 1$. Thus by rearranging we have that
    \begin{align*}
        \infnorm{V^{\pihat}_{\gamma, \rt} - V^T} \leq \frac{2\infnorm{\L(V^T) - V^T}}{1-\gamma} \leq \frac{2}{3}\varepsilon
    \end{align*}
    using~\eqref{eq:gamma_T_bd}.
    Then we can bound
    \begin{align*}
        \infnorm{V^{\pihat}_{\gamma, \rt} - V^\star_{\gamma, \spb}} \leq \infnorm{V^{\pihat}_{\gamma, \rt} - V^T} + \infnorm{V^T - V^\star_{\gamma, \spb}} \leq \frac{2}{3} \varepsilon + (1-\gamma)\frac{\varepsilon}{3} \leq \varepsilon
    \end{align*}
    as desired. Also, $\spannorm{V^\star_{\gamma, \spb}} \leq \spb$ since it is equal to $\Clip_{\spb}(\T_\gamma(V^\star_{\gamma, \spb}))$ and the output of $\Clip_{\spb}$ clearly has span bounded by $\spb$. Therefore
    \begin{align*}
        \spannorm{V^{\pihat}_{\gamma, \rt}} \leq \spannorm{V^\star_{\gamma, \spb}} + \spannorm{V^{\pihat}_{\gamma, \rt} - V^\star_{\gamma, \spb}} \leq \spannorm{V^\star_{\gamma, \spb}} + 2\infnorm{V^{\pihat}_{\gamma, \rt} - V^\star_{\gamma, \spb}} \leq \spb + 2\varepsilon.
    \end{align*}

    Finally, letting $\T_{\gamma, r'}^{\pi'}$ be the Bellman consistency/evaluation operator for the policy $\pi'$ in the DMDP $(P, r', \gamma)$, we have $\T_{\gamma, r'}^{\pi'}(V_{\gamma, r'}^{\pi'}) = V_{\gamma, r'}^{\pi'}$. Also, since $\spannorm{V_{\gamma, r'}^{\pi'}} \leq \spb$, we have $\Clip_{\spb} (\T_{\gamma, r'}^{\pi'}(V_{\gamma, r'}^{\pi'}) )= V_{\gamma, r'}^{\pi'}$. Letting $\L^{\pi'} = \Clip_{\spb} \circ \T_{\gamma, r'}^{\pi'}$, we have that $\L^{\pi'}$ is a $\gamma$ contraction with respect to $\infnorm{\cdot}$ (since as discussed above, $\Clip_{\spb}$ is $\infnorm{\cdot}$-nonexpansive, and because $\T_{\gamma, r'}^{\pi'}$ is well-known to be $\gamma$-contractive with respect to $\infnorm{\cdot}$). Thus $\L^{\pi'}$ has a unique fixed point, which must be $V_{\gamma, r'}^{\pi'}$ (since we have already verified $V_{\gamma, r'}^{\pi'}$ is a fixed point), and furthermore Picard iteration will converge to this fixed point. The operator $\Clip_{\spb}$ is also monotonic (in the sense that for any $x,y \in \R^{\S}$, $x \leq y \implies \Clip_{\spb}(x) \leq \Clip_{\spb}(y)$) as shown in \cite[Lemma 16]{fruit_efficient_2018}. It is well-known that $\T_\gamma$ and $\T_{\gamma, r'}^{\pi'}$ are also monotonic, which immediately implies that the compositions $\L$ and $\L^{\pi'}$ are monotonic. It is immediate that for any $x \in \R^{\S}$ we have $\T_{\gamma, r'}^{\pi'}(x) \leq \T_\gamma(x)$, since for any $x \in \R^{\S}$ we have
    \begin{align*}
        \T_{\gamma, r'}^{\pi'}(x) = M^{\pi}(r' + \gamma P x) \leq M^{\pi}(r + \gamma P x) \leq M(r + \gamma P x) = \T_\gamma(x)
    \end{align*}
    crucially using the fact that $r' \leq r$ and also monotonicity of $M^{\pi}$.
    This thus implies that $\L^{\pi'}(x) = \Clip_{\spb}(\T_{\gamma, r'}^{\pi'}(x)) \leq \Clip_{\spb}(\T_\gamma(x)) \leq \L(x)$ for any $x \in \R^{\S}$ (using monotonicity of $\Clip_{\spb}$). Therefore $\L^{\pi'}(\zero) \leq \L(\zero)$, and if it holds for some integer $i \geq 1$ that $(\L^{\pi'})^{(i)}(\zero) \leq \L^{(i)}(\zero)$ then we can use the above-discussed properties to obtain that
    \begin{align*}
        (\L^{\pi'})^{(i+1)}(\zero) = \L^{\pi'}\left( (\L^{\pi'})^{(i)}(\zero)\right) \leq \L\left( (\L^{\pi'})^{(i)}(\zero)\right) \leq \L\left( \L^{(i)}(\zero)\right) = \L^{(i+1)}(\zero),
    \end{align*}
    so we have by induction that $(\L^{\pi'})^{(i)}(\zero) \leq  \L^{(i)}(\zero)$ holds for all $i$, and thus
    \begin{align*}
        V_{\gamma, r'}^{\pi'} = \lim_{i \to \infty} (\L^{\pi'})^{(i)}(\zero) \leq \lim_{i \to \infty} \L^{(i)}(\zero) = V_{\gamma, \spb}^{\star}.
    \end{align*}
    Thus we have
    \begin{align*}
        V_\gamma^{\pihat} \geq V^{\pihat}_{\gamma, \rt} \geq V^\star_{\gamma, \spb} - \varepsilon\one \geq V_{\gamma, r'}^{\pi'} - \varepsilon\one
    \end{align*}
    as desired, where the first inequality is because $r \geq \rt$, the second inequality uses that $\infnorm{V^{\pihat}_{\gamma, \rt} - V^\star_{\gamma, \spb}} \leq \varepsilon$, and the final inequality uses the above argument.
\end{proof}

\subsection{Proof of Theorem \ref{thm:span_regularization_performance}}
\label{sec:span_regularization_proof}

First we develop bounds on the error between value functions within $P$ and $\Phat$.
\begin{lem}
    \label{lem:fixed_policy_error_bound}
    Fix a policy $\pi$, $\gamma \in (0,1)$, and $\delta, n > 0$. Then with probability at least $1 - \delta$, we have that
    \begin{align*}
        \infnorm{V_\gamma^{\pi} - \Vhat_\gamma^{\pi}} & \leq \frac{24 \log_2 \log_2 (\frac{1}{1-\gamma} + 4)}{1-\gamma} \sqrt{\frac{  \spannorm{V_\gamma^{\pi}} + 1 }{n} 16 \log \left( \frac{12 SA n}{(1-\gamma)^2 \delta}\right) }.
    \end{align*}
\end{lem}
\begin{proof}
    We can reuse the proof of \cite[Theorem 9]{zurek_plug-approach_2024}. Specifically, \cite[Equation 31]{zurek_plug-approach_2024} is stated for an optimal policy in the DMDP $(P, r, \gamma)$, but completely identical arguments actually hold for any fixed policy $\pi$. Therefore \cite[Equation 31]{zurek_plug-approach_2024} yields that with probability at least $1 - \delta$ it holds that
    \begin{align*}
        \infnorm{V_\gamma^{\pi} - \Vhat_\gamma^{\pi}} & \leq \frac{24 \log_2 \log_2 (\frac{1}{1-\gamma} + 4)}{1-\gamma} \sqrt{\frac{  \spannorm{V_\gamma^{\pi}} + 1 }{n} 16 \log \left( \frac{12 SA n}{(1-\gamma)^2 \delta}\right) }.
    \end{align*}
    This completes the proof.
\end{proof}

We would like to have a similar bound involving policies output by the span-constrained planning procedure applied to $\Phat$, but since such policies are probabilistically dependent on $\Phat$, much more effort is needed. In the absence of span-constrained planning, \cite{agarwal_model-based_2020} introduce the absorbing MDP construction to overcome such issues, and our approach is strongly inspired by theirs, although ours requires new modifications which are specific to the span-constrained planning procedure Algorithm \ref{alg:span_constrained_planning}.

First we must define the key quantities involved in our construction.
Fix $s \in \S$ and $u \in [0,1]$. We define the MDP transition kernel $\Phat^{(s)}$ as
\begin{align*}
    \Phat^{(s)}(s' \mid s'', a) = \begin{cases}
        \Phat(s' \mid s'', a) & \text{$s'' \neq s$} \\
        1 & \text{$s'' = s' = s$} \\
        0 & \text{$s'' = s$ and $s' \neq s$} 
    \end{cases}.
\end{align*}
In words, $\Phat^{(s)}$ is identical to $\Phat$ except state $s$ is made to be absorbing. We also define the reward function $r^{(s,u)}$ as
\begin{align*}
    r^{(s,u)}(s', a) = \begin{cases}
        r(s',a) & s' \neq s \\
        u & s' = s
    \end{cases}.
\end{align*}
In words, $r^{(s,u)}$ is identical to $r$ except state $s$ gives reward $u$ for all actions.
For any $\gamma \in (0,1)$ and any $\spb > 0$, let $\That_\gamma^{(s,u)} : \R^{\S} \to \R^{\S}$ be the $\gamma$-discounted Bellman operator for the DMDP $(\Phat^{(s)}, r^{(s,u)}, \gamma)$, that is,
\begin{align*}
    \That_\gamma^{(s,u)}(x) = M(r^{(s,u)} + \gamma \Phat^{(s)}x)
\end{align*}
for all $x \in \R^{\S}$. Let $\Vhat_{\gamma, \spb}^{(s,u)}$ be the unique fixed point of $\Clip_{\spb} \circ \That_\gamma^{(s,u)}$ (this fixed point exists and is unique because $\Clip_{\spb} \circ \That_\gamma^{(s,u)}$ is a $\gamma$-contraction due to Lemma \ref{lem:span_constrained_planning_subroutine}).

Now we can summarize the key properties of this construction in the following lemma.

\begin{lem}
\label{lem:LOO_properties}
Fixing $\gamma \in (0,1)$ and $\spb > 0$, for any $u, u'\in [0,1]$ we have
    \begin{align*}
        \infnorm{\Vhat_{\gamma, \spb}^{(s,u)} - \Vhat_{\gamma, \spb}^{(s,u')}} \leq \frac{|u-u'|}{1-\gamma}.
    \end{align*}
    Also, letting $u^\star(s) = \That_\gamma(\Vhat^\star_{\gamma, \spb})(s) - \gamma \Vhat^\star_{\gamma, \spb}(s)$, we have that $u^\star(s) \in [0,1]$ and
    \begin{align*}
        \Vhat_{\gamma, \spb}^{(s,u^\star(s))} = \Vhat_{\gamma, \spb}^\star,
    \end{align*}
    where $\Vhat_{\gamma, \spb}^\star$ is the unique fixed point of $\Clip_{\spb} \circ \That_\gamma$.

    Consequently, there exists a finite set $U$ with $|U| = \left\lceil \frac{1}{2(1-\gamma)\varepsilon} \right \rceil$ such that almost surely, for any $s \in \S$ there exists $u \in U$ such that $\infnorm{\Vhat^\star_{\gamma, \spb} - \Vhat^{(s,u)}_{\gamma, \spb}} \leq \varepsilon.$

    Lastly, for any $u \in [0,1], s \in \S$, and $a \in \A$, $\Vhat_{\gamma, \spb}^{(s,u)}$ is independent of the samples $S^1_{s,a}, \dots, S^n_{s,a}$ used to construct $\Phat(\cdot \mid s, a)$.
\end{lem}
\begin{proof}
    For the first inequality, we can calculate that
    \begin{align*}
        \infnorm{\Vhat_{\gamma, \spb}^{(s,u)} - \Vhat_{\gamma, \spb}^{(s,u')}} &= \infnorm{\Clip_{\spb}\left( \That_\gamma^{(s,u)} \left( \Vhat_{\gamma, \spb}^{(s,u)}\right)\right) - \Clip_{\spb}\left( \That_\gamma^{(s,u')} \left( \Vhat_{\gamma, \spb}^{(s,u')}\right)\right)} \\
        &\leq \infnorm{ \That_\gamma^{(s,u)} \left( \Vhat_{\gamma, \spb}^{(s,u)}\right) -  \That_\gamma^{(s,u')} \left( \Vhat_{\gamma, \spb}^{(s,u')}\right)} \\
        &= \infnorm{ M \left(r^{(s,u)} + \gamma \Phat^{(s)} \Vhat_{\gamma, \spb}^{(s,u)}\right) -  M \left(r^{(s,u')} + \gamma \Phat^{(s)} \Vhat_{\gamma, \spb}^{(s,u')}\right)} \\
        & \leq \infnorm{r^{(s,u)} -  r^{(s,u')} + \gamma \Phat^{(s)}\left( \Vhat_{\gamma, \spb}^{(s,u)}  - \Vhat_{\gamma, \spb}^{(s,u')}\right)} \\
        & \leq \infnorm{r^{(s,u)} -  r^{(s,u')}} + \gamma \infnorm{\Vhat_{\gamma, \spb}^{(s,u)}  - \Vhat_{\gamma, \spb}^{(s,u')}} \\
        &= |u - u'| + \gamma \infnorm{\Vhat_{\gamma, \spb}^{(s,u)}  - \Vhat_{\gamma, \spb}^{(s,u')}}.
    \end{align*}
    Rearranging, we obtain that
    \begin{align*}
        \infnorm{\Vhat_{\gamma, \spb}^{(s,u)} - \Vhat_{\gamma, \spb}^{(s,u')}} \leq \frac{|u-u'|}{1-\gamma}
    \end{align*}
    as desired.

    Now we show the second statement in the lemma. First we show that $\Vhat_{\gamma, \spb}^{(s,u^\star(s))} = \Vhat_{\gamma, \spb}^\star$.
    To show this, it suffices to show that
    \begin{align}
        \That_\gamma\left(\Vhat_{\gamma, \spb}^\star\right) = \That_\gamma^{(s,u^\star(s))}\left(\Vhat_{\gamma, \spb}^\star\right), \label{eq:fixed_point_equality_on_vstar}
    \end{align}
    since this implies that
    \begin{align*}
        \Vhat_{\gamma, \spb}^\star = \Clip_{\spb} \left(\That_\gamma\left(\Vhat_{\gamma, \spb}^\star\right)\right) = \Clip_{\spb} \left(\That^{(s,u^\star(s))}_\gamma\left(\Vhat_{\gamma, \spb}^\star\right)\right) 
    \end{align*}
    (using the fact that $\Vhat_{\gamma, \spb}^\star$ is a fixed point of $\Clip_{\spb} \circ \That_\gamma$ in the first equality), meaning that $\Vhat_{\gamma, \spb}^\star$ is a fixed point of $\Clip_{\spb} \circ \That^{(s,u^\star(s))}_\gamma$, and since the unique fixed point of $\Clip_{\spb} \circ \That^{(s,u^\star(s))}_\gamma$ is $\Vhat_{\gamma, \spb}^{(s,u^\star(s))}$, this would imply that $\Vhat_{\gamma, \spb}^{(s,u^\star(s))} = \Vhat_{\gamma, \spb}^\star$. Now focusing on~\eqref{eq:fixed_point_equality_on_vstar}, we note that by construction of $\That_\gamma^{s, u^\star(s)}$ we already have 
    $\That_\gamma\left(\Vhat_{\gamma, \spb}^\star\right)(s') = \That_\gamma^{(s,u^\star(s))}\left(\Vhat_{\gamma, \spb}^\star\right)(s')$
    for all $s' \neq s$, so it remains to show the equality for the $s$th coordinates. We have the desired equality
    \begin{align*}
        \That_\gamma^{(s,u^\star(s))}\left(\Vhat_{\gamma, \spb}^\star\right)(s) = u^\star(s) + \gamma \Vhat_{\gamma, \spb}^\star(s) =  \That_\gamma\left(\Vhat_{\gamma, \spb}^\star\right)
    \end{align*}
    using the definition of $\That_\gamma^{(s,u^\star(s))}$ and then the definition of $u^\star(s)$, so~\eqref{eq:fixed_point_equality_on_vstar} holds and thus $\Vhat_{\gamma, \spb}^{(s,u^\star(s))} = \Vhat_{\gamma, \spb}^\star$.

    Next we check that $u^\star(s) \in [0,1]$. We consider two cases, either $\That_\gamma(\Vhat^\star_{\gamma, \spb})(s) = \Vhat^\star_{\gamma, \spb}(s)$ (the clipping operator does not affect entry $s$) or $\That_\gamma(\Vhat^\star_{\gamma, \spb})(s) > \Vhat^\star_{\gamma, \spb}(s)$. In the first case we immediately have
    \begin{align*}
        u^\star(s) = \That_\gamma(\Vhat^\star_{\gamma, \spb})(s) - \gamma \Vhat^\star_{\gamma, \spb}(s) = (1-\gamma)\Vhat^\star_{\gamma, \spb}(s)
    \end{align*}
    which is $\in [0,1]$ since $\Vhat^\star_{\gamma, \spb}(s) \in [0, \frac{1}{1-\gamma}]$. (This fact can be seen by noting that $\That_\gamma(x) \geq (\Clip_{\spb} \circ \That_\gamma)(x)$ for any $x$ so by a monotonicity argument we have that $\Vhat^\star_{\gamma, \spb} \leq \Vhat^\star_\gamma \leq \frac{1}{1-\gamma}\one$.) In the second case, since the clipping affects entry $s$, this means that $\Vhat^\star_{\gamma, \spb}(s)$ must be the largest (not necessarily the uniquely largest) entry of $\Vhat^\star_{\gamma, \spb}$. Therefore we have by the definition of $\That_\gamma(\Vhat^\star_{\gamma, \spb})$ that for some $a^\star \in \A$,
    \begin{align*}
        \That_\gamma(\Vhat^\star_{\gamma, \spb})(s) = r(s,a^\star) + \gamma \Phat_{s,a^\star} \Vhat^\star_{\gamma, \spb} \leq r(s,a^\star) + \gamma \Vhat^\star_{\gamma, \spb}(s)
    \end{align*}
    since $\Phat_{s,a^\star} \Vhat^\star_{\gamma, \spb}$ is an average of entries of $\Vhat^\star_{\gamma, \spb}$ which must be less than the largest entry $\Vhat^\star_{\gamma, \spb}(s)$. After rearranging we have that
    \begin{align*}
        u^\star(s) = \That_\gamma(\Vhat^\star_{\gamma, \spb})(s) - \gamma \Vhat^\star_{\gamma, \spb}(s) \leq r(s,a^\star) \leq 1,
    \end{align*}
    and to lower bound $u^\star(s)$ we can simply use that in this case we have assumed $\That_\gamma(\Vhat^\star_{\gamma, \spb})(s) > \Vhat^\star_{\gamma, \spb}(s)$, so
    \begin{align*}
        u^\star(s) = \That_\gamma(\Vhat^\star_{\gamma, \spb})(s) - \gamma \Vhat^\star_{\gamma, \spb}(s) > (1-\gamma)\Vhat^\star_{\gamma, \spb}(s) \geq 0.
    \end{align*}

    Now we show the penultimate statement. Letting $U$ be a set of $\left\lceil \frac{1}{2(1-\gamma)\varepsilon} \right \rceil$ equally spaced points in $[0,1]$, for any $z \in [0,1]$ there exists $u \in U$ such that $|x - u| \leq (1-\gamma)\varepsilon$. Therefore for any $s$, letting $U(s)$ be the closest element of $U$ to $u^\star(s)$, we have that
    \begin{align*}
        \infnorm{\Vhat^\star_{\gamma, \spb} - \Vhat^{(s,U(s))}_{\gamma, \spb}} = \infnorm{\Vhat^{(s,u^\star(s))}_{\gamma, \spb} - \Vhat^{(s,U(s))}_{\gamma, \spb}} \leq \frac{|u^\star(s)- U(s)|}{1-\gamma} \leq \varepsilon
    \end{align*}
    as desired.

    Finally, for any $s \in \S, a \in \A, u \in [0,1]$, the independence of $\Vhat_{\gamma, \spb}^{(s,u)}$ from the samples $S^1_{s,a}, \dots, S^n_{s,a}$ is immediate from the construction, since $\Vhat_{\gamma, \spb}^{(s,u)}$ uses the transition kernel $\Phat^{(s)}$ which is independent of $S^1_{s,a}, \dots, S^n_{s,a}$.
\end{proof}

Now we can use this construction to prove the desired error bound. First we state the desired bound:
\begin{lem}
    \label{lem:span_const_plan_policy_error_bound}
    Fix $\gamma \in (0,1)$ and $\delta, \spb, n > 0$. Let $\pihat, \Vt, \rt$ be the output of the Span-Constrained Planning Algorithm \ref{alg:span_constrained_planning} with inputs $(\Phat, r, \gamma), \spb, \frac{1}{n}$. Then with probability at least $1 - \delta$, we have that
    \begin{align*}
        \infnorm{V_{\gamma, \rt}^{\pihat} - \Vhat_{\gamma, \rt}^{\pihat}} & \leq \frac{24 \log_2 \log_2 (\frac{1}{1-\gamma} + 4)}{1-\gamma} \sqrt{\frac{  \spannorm{\Vhat_{\gamma, \rt}^{\pihat}} + 1 }{n} 16 \log \left( \frac{12 SA n}{(1-\gamma)^2 \delta}\right) }.
    \end{align*}
\end{lem}

At a high level, to prove Lemma \ref{lem:span_const_plan_policy_error_bound}, we take advantage of the formal similarity between our absorbing MDP construction, designed for span-constrained MDPs, and the original construction of \cite{agarwal_model-based_2020}. Specifically, we are able to reuse certain proof steps from \cite{zurek_plug-approach_2024} which utilize the absorbing MDP construction of \cite{agarwal_model-based_2020}, since although their statements involve different objects, their proofs utilize certain properties of said objects which all hold in an identical manner for the objects we consider.
First, by examining the specific properties used within the proof of \cite[Lemma 20]{zurek_plug-approach_2024}, the following result is actually shown:
\begin{lem}
\label{lem:LOO_DMDP_bernstein_bound_abstract}
Suppose there exists a random variable $\Vhat^\star$, a set $U$, as well as random variables $(\Vhat^\star_{s,u})_{s \in \S, u \in U}$ such that
\begin{enumerate}
    \item For any $s \in \S$, $u \in U$, and $a \in \A$, $\Vhat^\star_{s,u}$ is independent of the samples $S^1_{s,a}, \dots, S^n_{s,a}$ used to construct $\Phat(\cdot \mid s, a)$.
    \item Almost surely, for any $s \in \S$, there exists $u(s) \in U$ such that $\infnorm{\Vhat^\star - \Vhat^\star_{s,u(s)}} \leq \frac{1}{n}$.
    \item $|U| \leq \left \lceil \frac{n}{2(1-\gamma)} \right \rceil$.
\end{enumerate}
    Also suppose $n \geq 4$. With probability at least $1-\delta$, the following holds: for all policies $\pi$ and all $X \in \R^{\S}$ which satisfy $\infnorm{X - \Vhatstar} \leq \frac{1}{n}$ and that $\spannorm{X} \leq \frac{1}{1-\gamma}$, letting $\Vc = X - \left(\min_s X(s)\right)\one$, for all $k = 0, \dots, \left\lceil\log_2 \log_2 \left( \spannorm{X} + 4\right) \right\rceil$, we have
    \begin{align}
        \left| \left(\Phat_{\pi} - P_{\pi}\right) \left( \Vc \right)^{\circ 2^k} \right|
    &\leq \sqrt{\frac{\beta   \Var_{\Phat_{\pi}} \left[ \left(\Vc \right)^{\circ 2^k}  \right]    }{n }} + \frac{ \beta \cdot 2^k }{n}  \left( \infnorm{\Vc} + 1\right)^{2^k}   \one \label{eq:LOO_DMDP_bernstein_bound_abstract_conclusion}
    \end{align}
where $\beta = 16 \log \left(12 \frac{SAn}{(1-\gamma)^2 \delta} \right)$.
\end{lem}

Here $x^{\circ p}$ denotes the elementwise $p$th power of any vector $x$.
With Lemma \ref{lem:LOO_DMDP_bernstein_bound_abstract}, as well as another key result from \cite{zurek_plug-approach_2024} which utilizes a conclusion in the form of Lemma \ref{lem:LOO_DMDP_bernstein_bound_abstract} to prove an error bound between value functions, we can now prove Lemma \ref{lem:span_const_plan_policy_error_bound}.
\begin{proof}[Proof of Lemma \ref{lem:span_const_plan_policy_error_bound}]
    By Lemma \ref{lem:LOO_properties}, it is immediate that the assumptions of Lemma \ref{lem:LOO_DMDP_bernstein_bound_abstract} are satisfied for $\Vhat^\star \gets \Vhat^\star_{\gamma, \spb}$ and $\Vhat^\star_{s,u} \gets \Vhat^{(s,u)}_{\gamma, \spb}$. We would like to obtain the conclusion~\eqref{eq:LOO_DMDP_bernstein_bound_abstract_conclusion} for $X \gets \Vhat^{\pihat}_{\gamma, \spb}$, which we now argue satisfies the assumptions. First, by Lemma \ref{lem:span_constrained_planning_subroutine}, we have that (almost surely) $\infnorm{\Vhat_{\gamma, \rt}^{\pihat} - \Vhat_{\gamma, \spb}^\star} \leq \frac{1}{n}$. Second, we have $\zero \leq \Vhat_{\gamma, \rt}^{\pihat} \leq \frac{1}{1-\gamma} \one$, which implies $\spannorm{\Vhat_{\gamma, \rt}^{\pihat}} \leq \frac{1}{1-\gamma}$. Therefore, assuming that $n \geq 4$, by Lemma \ref{lem:LOO_DMDP_bernstein_bound_abstract} we have with probability at least $1 -\delta$ that, letting $\Vo = \Vhat_{\gamma, \rt}^{\pihat} - \left(\min_s \Vhat_{\gamma, \rt}^{\pihat}(s)\right)\one$ and $\ell = \left\lceil\log_2 \log_2 \left( \spannorm{\Vhat_{\gamma, \rt}^{\pihat}} + 4\right) \right\rceil$, for all $k = 0 , \dots, \ell$, we have
    \begin{align}
        \left| \left(\Phat_{\pi} - P_{\pi}\right) \left( \Vc \right)^{\circ 2^k} \right|
    &\leq \sqrt{\frac{\beta   \Var_{\Phat_{\pi}} \left[ \left(\Vc \right)^{\circ 2^k}  \right]    }{n }} + \frac{ \beta \cdot 2^k }{n}  \left( \infnorm{\Vc} + 1\right)^{2^k}   \one \label{eq:LOO_DMDP_bernstein_bound_concrete_conclusion}
    \end{align}
for $\beta = 16 \log \left(12 \frac{SAn}{(1-\gamma)^2 \delta} \right)$.

Now we can immediately apply~\eqref{eq:LOO_DMDP_bernstein_bound_concrete_conclusion} with \cite[Lemma 16]{zurek_plug-approach_2024} to obtain that
\begin{align}
    \infnorm{V_{\gamma, \rt}^{\pihat} - \Vhat_{\gamma, \rt}^{\pihat}} & \leq \frac{4(\ell + 1)\beta}{(1-\gamma)n}\spannorm{\Vhat_{\gamma, \rt}^{\pihat}} + \frac{2 (\ell + 1)}{1-\gamma} \sqrt{\frac{2 \beta (\spannorm{\Vhat_{\gamma, \rt}^{\pihat}} + 1)}{n}}. \label{eq:LOO_DMDP_bernstein_bound_concrete_conclusion_2}
\end{align}
Now all that remains is to simplify~\eqref{eq:LOO_DMDP_bernstein_bound_concrete_conclusion_2} which can be done in an identical manner as to \cite[Proof of Theorem 9]{zurek_plug-approach_2024} to obtain that
\begin{align*}
    \infnorm{V_{\gamma, \rt}^{\pihat} - \Vhat_{\gamma, \rt}^{\pihat}} & \leq \frac{24 \log_2 \log_2 (\frac{1}{1-\gamma} + 4)}{1-\gamma} \sqrt{\frac{  \spannorm{\Vhat_{\gamma, \rt}^{\pihat}} + 1 }{n} 16 \log \left( \frac{12 SA n}{(1-\gamma)^2 \delta}\right) }.
\end{align*}
\end{proof}

Now we can apply these error bounds to the analysis of Algorithm \ref{alg:span_regularization}. 
We recall the definitions of some objects which appear in Algorithm \ref{alg:span_regularization} for convenience, which will be in effect for the remainder of this subsection. We define the empirical transition kernel $\Phat(s' \mid s, a) = \frac{1}{n}\sum_{i=1}^n \ind\{S^i_{s,a} = s'\}$, for all $s' \in \S$, using the $n$ samples drawn from all state-action pairs within Algorithm \ref{alg:span_regularization}. For all integers $i \in \{2 , \dots, \lceil \log_2 n \rceil$ we define $\spb_i = 2^i$, and we define $\gamma_i$ so that $\frac{1}{1-\gamma_i} = \max \big\{\frac{\sqrt{n \spb_i}}{\alpha(\delta, n) 2\sqrt{2}}, 1\big\}$. For each $i \in \{2 , \dots, \lceil \log_2 n \rceil$ we also define $\pit_i$, $\Vt_i$, and $\rt_i$ as the outputs from the Span-Constrained Planning Algorithm~\ref{alg:span_constrained_planning} with input DMDP $(\Phat, r, \gamma_i)$, input span constraint bound $\spb_i$, and input target error $\frac{1}{n}$.
For the remainder of this proof we fix a policy $\pi$ such that $\rho^\pi$ is a constant vector. ($\pi$ will later be chosen to optimize a certain deterministic quantity.)

\begin{lem}
\label{lem:concentration_event_span_const_plan}
    Define the function $\alpha(\tilde{\delta}, \tilde{n}) = 24 \sqrt{16 \log \left( \frac{24 SA \tilde{n}^5}{\tilde{\delta}}\right)} \log_2\left( \log_2 (\tilde{n} + 4) \right)$. Fix $\delta > 0$. Then with probability at least $1 - \delta$, we have for all $i = 2, \dots, \lceil \log_2 n \rceil$ that
\begin{align}
    \infnorm{V_{\gamma_i}^{\pi} - \Vhat_{\gamma_i}^{\pi}} &\leq \frac{\alpha(\delta, n)}{1-\gamma_i} \sqrt{\frac{ \spannorm{V_{\gamma_i}^{\pi}} + 1 }{n}} \label{eq:span_reg_fixed_pol_eval_bd}
\end{align}
    and also the Span-Constrained Planning Algorithm \ref{alg:span_constrained_planning} used within subroutine on line \ref{alg:span_constrained_planning_step} outputs a policy $\pit_{i}$, approximate value function $\Vt_i$, and reward function $\rt_i$ such that
\begin{align}
        \infnorm{\Vhat_{\gamma_i, \rt_i}^{\pit_{i}} - \Vhat^\star_{\gamma_i, \spb_i}} &\leq \frac{1}{n} \label{eq:pit_span_reg_near_opt}\\
        \infnorm{\Vt_{i} - \Vhat^\star_{\gamma_i, \spb_i}} &\leq \frac{1}{n} \label{eq:span_reg_vt_error}\\
        \infnorm{\Vhat_{\gamma_i, \rt_i}^{\pit_{i}} - V_{\gamma_i, \rt_i}^{\pit_{i}}} & \leq \frac{\alpha(\delta, n)}{1-\gamma_i} \sqrt{\frac{ \spannorm{\Vhat_{\gamma_i, \rt_i}^{\pit_{i}}} + 1 }{n}} \label{eq:span_reg_emp_pol_eval_bd}.
\end{align}
\end{lem}
\begin{proof}
    First, we note that the properties~\eqref{eq:pit_span_reg_near_opt} and~\eqref{eq:span_reg_vt_error} immediately follow from Lemma \ref{lem:span_constrained_planning_subroutine}. By a union bound and Lemmas \ref{lem:fixed_policy_error_bound} and \ref{lem:span_const_plan_policy_error_bound}, we have that with probability at least $1 - 2(\lceil \log_2 n \rceil-1)\delta'$, for all $i = 2, \dots, \lceil \log_2 n \rceil$, it holds that
    \begin{align*}
        \infnorm{V_{\gamma_i}^{\pi} - \Vhat_{\gamma_i}^{\pi}} & \leq \frac{24 \log_2 \log_2 (\frac{1}{1-\gamma_i} + 4)}{1-\gamma_i} \sqrt{\frac{  \spannorm{V_{\gamma_i}^{\pi}} + 1 }{n} 16 \log \left( \frac{12 SA n}{(1-\gamma_i)^2 \delta'}\right) }.
    \end{align*}
    and
    \begin{align*}
        \infnorm{V_{\gamma_i, \rt_i}^{\pit_i} - \Vhat_{\gamma_i, \rt_i}^{\pit_i}} & \leq \frac{24 \log_2 \log_2 (\frac{1}{1-\gamma_i} + 4)}{1-\gamma_i} \sqrt{\frac{  \spannorm{\Vhat_{\gamma_i, \rt_i}^{\pit_i}} + 1 }{n} 16 \log \left( \frac{12 SA n}{(1-\gamma_i)^2 \delta'}\right) }
    \end{align*}
    Since
    \begin{align*}
        \frac{1}{1-\gamma_i} \leq \frac{1}{1-\gamma_{\lceil \log_2 n \rceil}} \leq \max\left\{ 1, \frac{\sqrt{n M_{\lceil \log_2 n \rceil}}}{\alpha(\delta, n) 2 \sqrt{2}}\right\} \leq \max\left\{ 1, \frac{\sqrt{n 2n }}{\alpha(\delta, n) 2 \sqrt{2}}\right\} \leq n
    \end{align*}
    since $\alpha(\delta, n) \geq 1$ and $n \geq 1$, and also noting that $\lceil \log_2 n \rceil - 1 \leq  \log_2 n  \leq n$, by taking $\delta' = \frac{\delta}{2(\lceil \log_2 n \rceil-1)}$ we obtain that with probability at least $1 - \delta$ both
    \begin{align*}
        \infnorm{V_{\gamma_i}^{\pi} - \Vhat_{\gamma_i}^{\pi}} & \leq \frac{24 \log_2 \log_2 (n + 4)}{1-\gamma_i} \sqrt{\frac{  \spannorm{V_{\gamma_i}^{\pi}} + 1 }{n} 16 \log \left( \frac{24 SA n^3 (\lceil \log_2 n \rceil-1)}{ \delta}\right) }\\
        &\leq \frac{\alpha(\delta, n)}{1-\gamma_i} \sqrt{\frac{ \spannorm{V_{\gamma_i}^{\pi}} + 1 }{n}}
    \end{align*}
    and similarly
    \begin{align*}
        \infnorm{\Vhat_{\gamma_i, \rt_i}^{\pit_{i}} - V_{\gamma_i, \rt_i}^{\pit_{i}}} & \leq \frac{\alpha(\delta, n)}{1-\gamma_i} \sqrt{\frac{ \spannorm{\Vhat_{\gamma_i, \rt_i}^{\pit_{i}}} + 1 }{n}}
    \end{align*}
    for all $i = 2, \dots, \lceil \log_2 n \rceil$.
\end{proof}

\begin{lem}
\label{lem:span_reg_gain_lower_bound_validity}
Under the event described in Lemma \ref{lem:concentration_event_span_const_plan}, we have
    \begin{align}
        \rho^{\pit_i} & \geq \Objsp(i)\one = (1-\gamma_i)\min_{s} \Vt_i(s) \one  - 2\frac{1-\gamma_i}{n} \one - \alpha(\delta, n) \sqrt{\frac{ \spb_i + \frac{2}{n} + 1 }{n}} \one
    \end{align}
    for all $i = 2, \dots, \lceil \log_2 n \rceil$.    
\end{lem}
\begin{proof}
    The proof is very similar to the first part of the proof of Lemma \ref{lem:gain_lower_bound_validity}.
    Fix $i \in \{2, \dots, \lceil \log_2 n \rceil\}$. 
    First, by using~\eqref{eq:pit_span_reg_near_opt} and the fact that $\spannorm{\Vhat^\star_{\gamma_i, \spb_i}} \leq \spb_i$, we have
    \begin{align}
        \spannorm{\Vhat_{\gamma_i, \rt_i}^{\pit_{i}}} \leq \spannorm{\Vhat^\star_{\gamma_i, \spb_i}} + \spannorm{\Vhat_{\gamma_i, \rt_i}^{\pit_{i}} - \Vhat^\star_{\gamma_i, \spb_i}} 
        \leq \spannorm{\Vhat^\star_{\gamma_i, \spb_i}} + 2\infnorm{\Vhat_{\gamma_i, \rt_i}^{\pit_{i}} - \Vhat^\star_{\gamma_i, \spb_i}} \leq \spb_i + \frac{2}{n}. \label{eq:span_reg_gain_lower_bound_validity_sp_bd}
    \end{align}
    We note that by triangle inequality and~\eqref{eq:pit_span_reg_near_opt} and~\eqref{eq:span_reg_vt_error} we have $\infnorm{\Vt_i - \Vhat^{\pit_i}_{\gamma_i, \rt_i}} \leq \frac{2}{n}$.
    Using this bound on $\infnorm{\Vt_i - \Vhat^{\pit_i}_{\gamma_i, \rt_i}}$,~\eqref{eq:span_reg_emp_pol_eval_bd}, and~\eqref{eq:span_reg_gain_lower_bound_validity_sp_bd}, we have
    \begin{align}
        \rho^{\pit_i} &\geq (1-\gamma_i)\min_{s} V_{\gamma_i}^{\pit_i}(s) \one \nonumber\\
        &\geq (1-\gamma_i)\min_{s} V_{\gamma_i, \rt_i}^{\pit_i}(s) \one \nonumber\\
        &\geq (1-\gamma_i)\min_{s} \Vhat_{\gamma_i, \rt_i}^{\pit_i}(s) \one - (1-\gamma_i)\infnorm{V_{\gamma_i, \rt_i}^{\pit_i} - \Vhat_{\gamma_i, \rt_i}^{\pit_i}} \one \nonumber\\
        &\geq (1-\gamma_i)\min_{s} \Vt_i(s) \one -(1-\gamma_i)\infnorm{\Vt_i - \Vhat^{\pit_i}_{\gamma_i, \rt_i}} \one - (1-\gamma_i)\infnorm{V_{\gamma_i, \rt_i}^{\pit_i} - \Vhat_{\gamma_i, \rt_i}^{\pit_i}} \one \nonumber\\
        &\geq (1-\gamma_i)\min_{s} \Vt_i(s) \one  - 2\frac{1-\gamma_i}{n} \one - \alpha(\delta, n) \sqrt{\frac{ \spannorm{\Vhat_{\gamma_i, \rt_i}^{\pit_{i}}} + 1 }{n}} \one \nonumber\\
        &\geq (1-\gamma_i)\min_{s} \Vt_i(s) \one  - 2\frac{1-\gamma_i}{n} \one - \alpha(\delta, n) \sqrt{\frac{ \spb_i + \frac{2}{n} + 1 }{n}} \one\nonumber \\
        &= \Objsp(i)\one . \nonumber
    \end{align}
\end{proof}

Before continuing we need to establish some relationships between DMDP and AMDP quantities for the policy $\pi$. This lemma is similar to \cite[Lemma 2]{wei_model-free_2020} but for a general policy.
\begin{lem}
    \label{lem:fix_pol_DMDP_AMDP_relns}
    Suppose $\rho^\pi$ is constant. Then
    \begin{enumerate}
        \item $\infnorm{\frac{1}{1-\gamma}\rho^\pi - V_\gamma^\pi} \leq \spannorm{h^\pi}$
        \item $\spannorm{V^\pi_\gamma} \leq 2 \spannorm{h^\pi}$.
    \end{enumerate}
\end{lem}
\begin{proof}
    For the first statement, we have that
    \begin{align*}
        V_\gamma^\pi = (I - \gamma P_\pi)^{-1} r_\pi = (I - \gamma P_\pi)^{-1} (\rho^\pi + h^\pi - P_\pi h^\pi) = (I - \gamma P_\pi)^{-1} \rho^\pi + (I - \gamma P_\pi)^{-1}(I - P_\pi) h^\pi.
    \end{align*}
    Since $P_\pi \rho^\pi = \rho^\pi$, by the Neumann series the first term is equal to $\frac{1}{1-\gamma} \rho^\pi$. By a standard calculation (e.g. \cite[Lemma 20]{zurek_span-based_2025}) the second term satisfies $\infnorm{(I - \gamma P_\pi)^{-1}(I - P_\pi) h^\pi } \leq \spannorm{h^\pi}$. Therefore we have that  $\infnorm{\frac{1}{1-\gamma}\rho^\pi - V_\gamma^\pi} \leq \infnorm{(I - \gamma P_\pi)^{-1}(I - P_\pi) h^\pi } \leq \spannorm{h^\pi}$.

    For the second statement, since $\rho^\pi$ is constant, we have that
    \begin{align*}
        \spannorm{V^\pi_\gamma} = \spannorm{V^\pi_\gamma - \frac{1}{1-\gamma}\rho^\pi} \leq 2 \infnorm{V^\pi_\gamma - \frac{1}{1-\gamma}\rho^\pi} \leq 2 \spannorm{h^\pi}.
    \end{align*}
\end{proof}

\begin{lem}
\label{lem:span_reg_obj_lower_bound}
    Suppose that for some integer $i \geq 2$ we have $\spannorm{h^\pi} + 1 \leq \spb_i / 4$, and also that $n \geq 2\alpha(\delta, n)^2$. Then under the event in Lemma \ref{lem:concentration_event_span_const_plan}, we have
    \begin{align*}
        \Objsp(i)\one \geq \rho^\pi -  \alpha(\delta, n) (2 + \sqrt{2})\sqrt{\frac{ \spb_i }{n}} \one.
    \end{align*}
\end{lem}
\begin{proof}
    The fact that $n \geq 2\alpha(\delta, n)^2$ implies that we always have
    \begin{align*}
        \frac{1}{1-\gamma_i} = \max \left\{\frac{\sqrt{n \spb_i}}{\alpha(\delta, n) 2\sqrt{2}}, 1\right\} = \frac{\sqrt{n \spb_i}}{\alpha(\delta, n) 2\sqrt{2}}
    \end{align*}
    for all $i \geq 2$, since $M_i = 2^i \geq 4$.
    Now, by triangle inequality, $\spannorm{\cdot} \leq 2 \infnorm{\cdot}$,~\eqref{eq:span_reg_fixed_pol_eval_bd}, Lemma \ref{lem:fix_pol_DMDP_AMDP_relns}, and the above expression for $\frac{1}{1-\gamma_i}$, we calculate that 
    \begin{align*}
        \spannorm{\Vhat_{\gamma_i}^\pi} & \leq \spannorm{V_{\gamma_i}^\pi} +\spannorm{\Vhat_{\gamma_i}^\pi - V_{\gamma_i}^\pi} \\
        & \leq \spannorm{V_{\gamma_i}^\pi} + 2\infnorm{\Vhat_{\gamma_i}^\pi - V_{\gamma_i}^\pi} \\
        & \leq \spannorm{V_{\gamma_i}^\pi} + 2 \frac{\alpha(\delta, n)}{1-\gamma_i} \sqrt{\frac{ \spannorm{V_{\gamma_i}^{\pi}} + 1 }{n}} \\
        & \leq 2\spannorm{h^\pi} + 2 \frac{\alpha(\delta, n)}{1-\gamma_i} \sqrt{\frac{ 2\spannorm{h^\pi} + 1 }{n}} \\
        & =2\spannorm{h^\pi} + \frac{\sqrt{ n \spb_i}}{\sqrt{2}} \sqrt{\frac{ 2\spannorm{h^\pi} + 1 }{n}} \\
        & \leq \frac{\spb_i}{2} + \frac{\sqrt{ \spb_i}}{\sqrt{2}} \sqrt{\spb_i/2} \\
        &= \spb_i.
    \end{align*}
    Consequently by Lemma \ref{lem:span_constrained_planning_subroutine} we have that $\Vhat^\star_{\gamma_i, \spb_i} \geq \Vhat_{\gamma_i}^\pi$, and thus
    \begin{align*}
        \Vt_i \geq \Vhat^\star_{\gamma_i, \spb_i} - \frac{1}{n}\one \geq \Vhat_{\gamma_i}^\pi - \frac{1}{n}\one.
    \end{align*}
    Now we lower-bound $\Vhat_{\gamma_i}^\pi$. We have that
    \begin{align}
        \Vhat_{\gamma_i}^\pi & \geq V_{\gamma_i}^\pi - \infnorm{\Vhat_{\gamma_i}^\pi - V_{\gamma_i}^\pi} \one \nonumber\\
        & \geq \frac{1}{1-\gamma_i}\rho^\pi - \infnorm{\frac{1}{1-\gamma_i}\rho^\pi - V_{\gamma_i}^\pi} \one - \infnorm{\Vhat_{\gamma_i}^\pi - V_{\gamma_i}^\pi} \one \nonumber\\
        & \geq \frac{1}{1-\gamma_i}\rho^\pi - \spannorm{h^\pi} \one - \frac{\alpha(\delta, n)}{1-\gamma_i} \sqrt{\frac{ \spannorm{V_{\gamma_i}^{\pi}} + 1 }{n}} \one \label{eq:span_reg_obj_lower_bound_1}
    \end{align}
    using Lemma \ref{lem:fix_pol_DMDP_AMDP_relns} and~\eqref{eq:span_reg_fixed_pol_eval_bd} again.
    Thus
    \begin{align*}
        \Vt_i &\geq \Vhat_{\gamma_i}^\pi - \frac{1}{n}\one \\
        &\geq \frac{1}{1-\gamma_i}\rho^\pi - \spannorm{h^\pi} \one - \frac{\alpha(\delta, n)}{1-\gamma_i} \sqrt{\frac{ \spannorm{V_{\gamma_i}^{\pi}} + 1 }{n}} \one - \frac{1}{n}\one \\
        & \geq \frac{1}{1-\gamma_i}\rho^\pi - \frac{\spb_i}{4} \one - \frac{\alpha(\delta, n)}{1-\gamma_i} \sqrt{\frac{ 2\spannorm{h^\pi} + 1 }{n}} \\
        & \geq \frac{1}{1-\gamma_i}\rho^\pi - \frac{\spb_i}{4} \one - \frac{\alpha(\delta, n)}{1-\gamma_i} \sqrt{\frac{ \spb_i/2 }{n}}
    \end{align*}
    where in the first inequality we combine~\eqref{eq:span_reg_vt_error} with Lemma \ref{lem:span_constrained_planning_subroutine}, which states that $\Vhat^\star_{\gamma_i, \spb_i} \geq \Vhat_{\gamma_i}^\pi$ since (as we calculated above) we have $\spannorm{\Vhat_{\gamma_i}^\pi} \leq \spb_i$. In the second inequality we use~\eqref{eq:span_reg_obj_lower_bound_1}, in the third we use Lemma \ref{lem:fix_pol_DMDP_AMDP_relns}, and in the final inequality we use the assumption $\spannorm{h^\pi} + 1 \leq \spb_i / 4$.
    Therefore, we have
    \begin{align*}
        \Objsp(i) &= (1-\gamma_i)\min_{s} \Vt_i(s)   - 2\frac{1-\gamma_i}{n}  - \alpha(\delta, n) \sqrt{\frac{ \spb_i + \frac{2}{n} + 1 }{n}} \\
        & \geq \rho^\pi(s_0) - (1-\gamma_i)\frac{\spb_i}{4} \one - \alpha(\delta, n) \sqrt{\frac{ \spb_i/2 }{n}} - 2\frac{1-\gamma_i}{n}  - \alpha(\delta, n) \sqrt{\frac{ \spb_i + \frac{2}{n} + 1 }{n}} \\
        &= \rho^\pi(s_0) - \frac{\alpha(\delta, n)2\sqrt{2}}{\sqrt{n \spb_i}}\frac{\spb_i}{4} \one - \alpha(\delta, n) \sqrt{\frac{ \spb_i/2 }{n}} - 2\frac{1-\gamma_i}{n}  - \alpha(\delta, n) \sqrt{\frac{ \spb_i + \frac{2}{n} + 1 }{n}}  \\
        &\geq  \rho^\pi(s_0) -  \alpha(\delta, n) \sqrt{2}\sqrt{\frac{ \spb_i }{n}} - \frac{1}{\sqrt{2}} \frac{1}{n} \frac{1}{2\sqrt{n}}  - \alpha(\delta, n) \sqrt{\frac{ \spb_i + \frac{2}{n} + 1 }{n}} \\
        &\geq  \rho^\pi(s_0) -  \alpha(\delta, n) (2 + \sqrt{2})\sqrt{\frac{ \spb_i }{n}},
    \end{align*}
    where $s_0$ is an arbitrary state (since $\rho^\pi$ is constant), and we used that $\alpha(\delta, n) \geq 1$ and $\spb_i \geq 4$.
\end{proof}

\begin{proof}[Proof of Theorem \ref{thm:span_regularization_performance}]
We assume that the event in Lemma \ref{lem:concentration_event_span_const_plan} holds, which occurs with probability at least $1 - \delta$. Lemma \ref{lem:span_reg_gain_lower_bound_validity} immediately implies that
\begin{align*}
    \rho^{\pihat} = \rho^{\pit_{\widehat{i}}} & \geq \Objsp(\widehat{i})\one.
\end{align*}
We also trivially have that $\rho^{\pihat} \geq \zero$, which implies that 
\begin{align*}
    \rho^{\pihat} \geq \max\{\Objsp(\widehat{i}), 0 \} \one = \rhohat.
\end{align*}

Now we lower-bound $\rhohat$. First note that if $n < 2\alpha(\delta, n)^2$, then
\begin{align*}
        10 \alpha(\delta, n)  \sqrt{\frac{\spannorm{h^\pi}+1}{n}} \geq 10/\sqrt{2} > 1,
    \end{align*}
    so the desired conclusion holds trivially since we must have $\rhohat \geq \zero$, and $\rho^\pi \leq \one$. Thus we can now consider the situation $n \geq 2\alpha(\delta, n)^2$.
    We note that the smallest $i = 2$ causes $\spb_2 /8 = \frac{1}{2} < 1$.
    Therefore by the construction of the $\spb_i$'s, either there exists $i$ such that
    \begin{align}
        \spb_i / 8 \leq \spannorm{h^\pi} + 1 \leq \spb_i / 4 \label{eq:span_h_bounds}
    \end{align}
    or we have that $\spannorm{h^\pi} + 1 > \spb_i / 4$ for all $i$. Since the largest $i = \lceil \log_2 n \rceil$ causes $\spb_i / 4 > n/4$, in this second case we thus have that
    \begin{align*}
        10 \alpha(\delta, n)  \sqrt{\frac{\spannorm{h^\pi}+1}{n}} \geq 10 \alpha(\delta, n) \frac{1}{2} > 1,
    \end{align*}
    so again the desired conclusion holds trivially since we must have $\rhohat \geq \zero$.
    In the first case that~\eqref{eq:span_h_bounds} holds for some $i$, since also $n \geq 2\alpha(\delta, n)^2$, by Lemma \ref{lem:span_reg_obj_lower_bound} we have that
    \begin{align*}
        \rhohat \geq \Objsp(\widehat{i})\one \geq \Objsp(i)\one &\geq \rho^\pi -  \alpha(\delta, n) (2 + \sqrt{2})\sqrt{\frac{ \spb_i }{n}} \one \\
        & \geq \rho^\pi -  (2 + \sqrt{2}) \sqrt{8} \alpha(\delta, n) \sqrt{\frac{\spannorm{h^\pi}+1}{n}} \\
        & \geq \rho^\pi - 10 \alpha(\delta, n) \sqrt{\frac{\spannorm{h^\pi}+1}{n}}.
    \end{align*}
    Since we have proven this for arbitrary fixed $\pi$, we can apply this result to a policy $\pi \in \sup_{\pi : \rho^\pi \text{ constant}}\rho^\pi(s_0) - 10 \alpha(\delta, n) \sqrt{\frac{\spannorm{h^\pi}+1}{n}}$ (for an arbitrary state $s_0$) to obtain the desired conclusion. Finally, we can set $C_4 = 10$.
\end{proof}

\section{Examples}
\label{sec:examples}

In this section we provide the two examples mentioned in Subsection \ref{sec:span_regularization} of situations where the guarantee of Theorem \ref{thm:span_regularization_performance} could be much better than the minimax rate. In both examples each solid line represents a single action, and an expression ``$R = \dots$'' denotes the reward for said action. If the line splits into multiple dashed arrows then this indicates that the next-state transition from following this action is stochastic, and the numbers next to each dashed line indicate the transition probabilities. Otherwise if the line does not split then it indicates a deterministic next-state transition.

First we provide the simpler of the two examples, where $\tspannorm{h^\star}$ is arbitrarily larger than $\inf_{\pi: \rho^\pi = \rho^\star} \tspannorm{h^\pi}$.
\begin{figure}[H]
\centering
\begin{tikzpicture}[ -> , >=stealth, shorten >=2pt , line width=0.5pt, node distance =2cm, scale=0.7]

\node [circle, draw] (one) at (-2 , 0) {1};
\node [circle, draw] (two) at (2 , 0) {2};
\node [circle, draw, fill, inner sep=0.03cm] (dot1) at (-0.6 , 1.5) {};

\path (one) edge[-] [bend left] node [above] {$R=1$~~~~~~~~~~~~~~~~} (dot1) ;
\path (dot1) edge[dashed] [bend left] node [right] {$1-\frac{1}{T}$} (one);
\path (dot1) edge[dashed] [bend left] node [above] {$\frac{1}{T}$} (two);
\path (two) edge [loop right, looseness=15] node [right] {$R =\frac{1}{2}$}  (two) ;
\path (one) edge [bend right] node [below] {$R = \frac{1}{2}$} (two);

\end{tikzpicture}
\caption{An MDP where $\tspannorm{h^{\star}}$ can be arbitrarily larger than $\inf_{\pi: \rho^\pi = \rho^\star} \tspannorm{h^\pi}$.}
\label{fig:example_1}
\end{figure}

\begin{thm}
\label{thm:better_span_param_example}
    Consider the MDP displayed in Figure \ref{fig:example_1}. For any $T \geq 1$, we have $\tspannorm{h^\star} = \frac{T}{2}$ and $\inf_{\pi: \rho^\pi = \rho^\star} \tspannorm{h^\pi} = 0$.
\end{thm}

Next, we provide an instance where $\tspannorm{h^\star}$ can be arbitrarily large but a policy $\pi$ with an arbitrarily low level of suboptimality $\rho^\star - \rho^\pi$ can have $\tspannorm{h^\pi} = O(1)$.

\begin{figure}[H]
\centering
\begin{tikzpicture}[ -> , >=stealth, shorten >=2pt , line width=0.5pt, node distance =2cm, scale=1]

\node [circle, draw] (one) at (-2 , 2) {1};
\node [circle, draw] (two) at (-2 , -2) {2};
\node [circle, draw] (three) at (2 , 2) {3};
\node [circle, draw] (four) at (2 , -2) {4};
\node [circle, draw, fill, inner sep=0.03cm] (dot3) at (1 , 0.8) {};
\node [circle, draw, fill, inner sep=0.03cm] (dot4) at (3 , -0.8) {};

\path (one) edge [bend left] node [right] {$R=\frac{1}{2}$} (two);
\path (two) edge [bend left] node [left] {$R=\frac{1}{2}$} (one);
\path (one) edge node [above] {$R = 0$} (three);
\path (four) edge node [below] {$R = 0$} (two);
\path (three) edge[-] [bend right] node [left] {$R = \frac{1}{2}$} (dot3);
\path (four) edge[-] [bend right] node [right] {$R = \frac{1}{2} + \varepsilon$} (dot4);
\path (dot3) edge[dashed] [bend right] node [below] {~~~~~~~~~$1 - 1/T$} (three);
\path (dot3) edge[dashed] [bend right] node [left] {$1/T$} (four);
\path (dot4) edge[dashed] [bend right] node [above] {$1 - 1/T$~~~~~~~~~} (four);
\path (dot4) edge[dashed] [bend right] node [right] {$1/T$} (three);


\end{tikzpicture}
\caption{An MDP where $\tspannorm{h^{\star}}$ can be arbitrarily larger than $\tspannorm{h^\pi}$ for some near-optimal policy $\pi$ satisfying $\rho^\pi = \rho^\star - \frac{\varepsilon}{2} \one$.}
\label{fig:example_2}
\end{figure}

\begin{thm}
\label{thm:near_opt_policy_example}
    Consider the MDP displayed in Figure \ref{fig:example_2}. For any $T \geq 1$ and $\varepsilon > 0$, we have that $\tspannorm{h^\star} = \frac{\varepsilon T}{2} + \varepsilon + \frac{1}{2}$, but there exists some policy $\pi$ with constant gain such that $\rho^\pi = \rho^\star - \frac{\varepsilon}{2} \one$ and $\tspannorm{h^\pi} = \frac{1}{2}$.
\end{thm}

\subsection{Proofs of Theorems \ref{thm:better_span_param_example} and \ref{thm:near_opt_policy_example}}
\begin{proof}[Proof of Theorem \ref{thm:better_span_param_example}]
    It is easy to see that state $1$ is transient under all policies and state $2$ is absorbing, so all policies are gain-optimal and have gain $\frac{1}{2} \one$. Only state $1$ has multiple possible actions, so it suffices to consider the two policies $\pi_{\text{up}}$, which takes the ``up'' action which has nonzero probability of returning to state $1$, and $\pi_{\text{down}}$, which leads to an immediate transition to state $2$. It is trivial to see that $h^{\pi_{\text{down}}} = \zero$, so we have that $\tspannorm{h^{\pi_{\text{down}}}}=0$. To compute $h^{\pi_{\text{up}}}$, we must have $h^{\pi_{\text{up}}}(2) = 0$, so we can then calculate that
    \begin{align*}
        &h^{\pi_{\text{up}}}(1) + \frac{1}{2} = 1 + (1-\frac{1}{T})h^{\pi_{\text{up}}}(1) + \frac{1}{T} h^{\pi_{\text{up}}}(2) \\
        \implies & h^{\pi_{\text{up}}}(1) = \frac{T}{2} + h^{\pi_{\text{up}}}(2) = \frac{T}{2}
    \end{align*}
    and thus $\tspannorm{h^{\pi_{\text{up}}}}=T/2$. Since these are the only two stationary deterministic policies, one of them must be Blackwell-optimal, and since they have equal gain and elementwise $h^{\pi_{\text{up}}} \geq h^{\pi_{\text{down}}}$ (with a strict inequality in state $1$), we have that $\pi_{\text{up}} = \pistar$ and $\tspannorm{h^\star} = \tspannorm{h^{\pi_{\text{up}}}}=T/2$. Since $\pi_{\text{down}}$ is gain-optimal it is immediate from $h^{\pi_{\text{down}}} = \zero$ that $\inf_{\pi: \rho^\pi = \rho^\star} \tspannorm{h^\pi} = 0$.
\end{proof}

\begin{proof}[Proof of Theorem \ref{thm:near_opt_policy_example}]
    There are two states, $1$ and $4$, where multiple actions are possible. We name the actions in state $1$ the ``down'' action (which leads to $2$) and the ``right'' action (which leads to $3$), and we name the actions in state $4$ the ``up'' action (which has positive probability of leading to $3$) and the ``left'' action (which leads to $2$). A deterministic stationary policy can be specified by its two choices between the actions available at states $1$ and $4$. We thus use $\pi_{DL}$ to indicate the policy which takes the down action in state $1$ and the left action in state $4$, and likewise for other choices of $\{D, R\} \times \{U, L\}$.

    It is easy to see that since $\varepsilon > 0$ the unique gain-optimal policy is $\pi_{RU}$ which has $\rho^{\pi_{RU}} = \rho^\star = \frac{1+\varepsilon}{2}\one$. Thus this policy must also be Blackwell-optimal. We now compute $h^\star = h^{\pi_{RU}}$. We have
    \begin{align*}
        &h^\star(3) + \frac{1+\varepsilon}{2} = \frac{1}{2} + (1-\frac{1}{T})h^\star(3) + \frac{1}{T}h^\star(4) \\
        \implies & h^\star(3) = h^\star(4) - \frac{\varepsilon T}{2}
    \end{align*}
    and since the stationary distribution of $\pi_{RU}$ has equal mass on states $3$ and $4$ we must have $h^\star(3) + h^\star(4) = 0$, which implies that
    $h^\star(3) = - \frac{\varepsilon T}{4}$ and $h^\star(4) =  \frac{\varepsilon T}{4}$. It is then easy to check that $h^\star(1) = h^\star(3) - \frac{1+\varepsilon}{2} = \frac{-\varepsilon T - 2\varepsilon - 2}{4}$ and $h^\star(2) = h^\star(1) - \frac{\varepsilon}{2} = \frac{-\varepsilon T - 4\varepsilon - 2}{4}$, so we have that $\tspannorm{h^\star} =   \frac{\varepsilon T}{4} - \frac{-\varepsilon T - 4\varepsilon - 2}{4} = \frac{\varepsilon T}{2} + \varepsilon + \frac{1}{2}$.

    Next we consider the policy $\pi_{DL}$. It is easy to see that $\rho^{\pi_{DL}}  = \frac{1}{2}\one$, which is constant and which satisfies $\rho^\star - \rho^{\pi_{DL}} = \frac{\varepsilon}{2}\one$. Now we compute $\tspannorm{h^{\pi_{DL}}}$. It is immediate to see that $h^{\pi_{DL}}(1) = h^{\pi_{DL}}(2) = 0$. Then we can calculate that $h^{\pi_{DL}}(4) = h^{\pi_{DL}}(2) - \frac{1}{2} = -\frac{1}{2}$, and finally that
    \begin{align*}
        & h^{\pi_{DL}}(3) + \frac{1}{2} = \frac{1}{2} + (1-\frac{1}{T})h^{\pi_{DL}}(3) + \frac{1}{T} h^{\pi_{DL}}(4) \\
        \implies & h^{\pi_{DL}}(3) = h^{\pi_{DL}}(4) = -\frac{1}{2}.
    \end{align*}
    Therefore we have that $\tspannorm{h^{\pi_{DL}}} = 0 - -\frac{1}{2} = \frac{1}{2}$.
\end{proof}

\end{document}